\documentclass[accepted]{uai2025} 
                        
\usepackage[greek,english]{babel}

\usepackage{natbib} 
\usepackage{mathtools} 
\usepackage{booktabs} 
\usepackage{tikz} 

\usepackage{amsmath,amsfonts,bm}

\def\eqref#1{(\ref{#1})}

\def\1{\bm{1}}

\def\vn{{\bm{n}}}

\def\vs{{\bm{s}}}

\def\vx{{\bm{x}}}

\def\evx{{x}}

\def\mA{{\bm{A}}}
\def\mB{{\bm{B}}}

\def\mI{{\bm{I}}}

\def\mS{{\bm{S}}}

\def\mU{{\bm{U}}}

\def\mW{{\bm{W}}}
\def\mX{{\bm{X}}}
\def\mY{{\bm{Y}}}

\DeclareMathAlphabet{\mathsfit}{\encodingdefault}{\sfdefault}{m}{sl}
\SetMathAlphabet{\mathsfit}{bold}{\encodingdefault}{\sfdefault}{bx}{n}
\newcommand{\tens}[1]{\bm{\mathsfit{#1}}}

\def\tS{{\tens{S}}}

\def\tW{{\tens{W}}}
\def\tX{{\tens{X}}}

\def\gG{{\mathcal{G}}}

\def\sR{{\mathbb{R}}}

\def\sV{{\mathbb{V}}}

\newcommand{\R}{\mathbb{R}}

\newcommand{\normlone}{L^1}
\newcommand{\normltwo}{L^2}

\DeclareMathOperator*{\argmax}{arg\,max}
\DeclareMathOperator*{\argmin}{arg\,min}

\usepackage{uai2025-template/panos}

\title{\mobius: Estimating Structural Vector Autoregression\\Assuming Sparse Input}

\author[1]{Panagiotis Misiakos}
\author[1]{Markus Püschel}
\affil[1]{%
    Computer Science Dept.\\
    ETH Zurich\\
    Zürich, Switzerland
}  
  
\begin{document}
\maketitle

\begin{abstract}
We introduce \mobius, a novel method for estimating a structural vector autoregression (SVAR) from time-series data under sparse input assumption. Unlike prior approaches using Gaussian noise, we model the input as independent Laplacian variables, enforcing sparsity and yielding a maximum likelihood estimator (MLE) based on least absolute error regression. 
We provide theoretical consistency guarantees for the MLE under mild assumptions. \mobius is efficient: it can leverage GPU acceleration to scale to thousands of nodes. On synthetic data with Laplacian or Bernoulli-uniform inputs, \mobius outperforms state-of-the-art methods in accuracy and runtime. When applied to S\&P 500 data, it clusters stocks by sectors and identifies significant structural shocks linked to major price movements, demonstrating the viability of our sparse input assumption.
\end{abstract}

\section{Introduction}
\label{sec:intro}
Time series arise in numerous applications where multi-dimensional observations are recorded at regular intervals, such as meteorology~\citep{yang2022heatUS}, finance~\citep{kleinberg2013finance, varlingam2023linkagesfinance}, and brain imaging~\citep{smith2011FMRI}. A fundamental challenge in analyzing time series is causal discovery, which seeks to uncover causal dependencies over time~\citep{assaad2022survey, hasan2023causalsurvey}.
If causal effects occur faster than the data’s temporal resolution, they appear instantaneous and can be modeled with a linear structural equation model (SEM)\citep{elementsCausalInference}. When the resolution is higher, they appear as lagged effects, typically captured by vector autoregression (VAR)\citep{kilian2013SVAR}. Regardless of the model, recovering true causal relationships requires additional assumptions, such as the absence of latent confounders, identifiability conditions, or access to interventions~\citep{dyalikedags}, which rarely hold in real-world settings. 
For instance, in financial markets, it is nearly impossible to observe all hidden confounders or directly intervene in stock prices.
Instead of identifying true causal effects, we focus on learning instantaneous and lagged dependencies through a structural vector autoregression (SVAR), which unifies a linear SEM and a VAR~\citep{hyvarinen2010varlingam}.

\paragraph{Structural vector autoregression} 
Originally introduced by~\citet{sims1980comparison}, SVAR has been widely applied in econometrics~\citep{lutkepohl2005new, kilian2013SVAR} and serves as a foundation for causal discovery in time-series data~\citep{pamfil2020dynotears}. SVAR models linear dependencies between variables, distinguishing between instantaneous effects (occurring within the same time step) and lagged effects (propagating over time). The model naturally associates time-series data with a directed acyclic graph (DAG), which encodes how each time step is generated from previous ones. These relationships collectively form the window graph, a DAG that uniquely determines the SVAR parameters. SVAR further assumes stationarity, meaning that the dependencies remain constant over time. 

\paragraph{Challenges and Limitations} Even when abstracting away the need for true causal effects, DAG learning from time-series data remains computationally challenging due to the complexity of temporal dependencies and the high dimensionality of real-world datasets. Theoretically, it generalizes DAG learning from static data, which is already an NP-hard problem~\citep{chickering2004nphard}.
Several methods have been proposed to estimate the weighted window graph from time-series data, including approaches specifically tailored for SVAR~\citep{hyvarinen2010varlingam}. However, many existing methods suffer from critical limitations. Some approaches, such as Granger causality-based methods, learn the summary graph that fails to incorporate time lags~\citep{bussmann2021NAVAR}, while others do not account for instantaneous dependencies~\citep{khanna2019eSRU}.
Most methods face computational challenges when applied to large DAGs, making them impractical for graphs with thousands of nodes~\citep{cheng2024cuts+}.
Structural shocks, i.e., the input variables of an SVAR~\citep{lanne2017MLE_estimation_SVAR} at each node, are often interpreted merely as noise variables in prior work~\citep{hyvarinen2010varlingam, pamfil2020dynotears}, limiting their interpretability and potential insights into the underlying causal mechanisms.
To address these challenges, we introduce a novel, efficient method that enforces sparsity in the input of the SVAR.

\paragraph{\mobius: Sparse Input SVAR}
\citet{hyvarinen2010varlingam} model SVAR under a non-Gaussian noise assumption for the inputs. We extend this by enforcing sparsity in the input, following \citet{misiakos2024fewrootcauses}, and model it as independent Laplacian random variables. The Laplace distribution naturally promotes sparsity~\citep{jing2015sparsematrixfactLaplace} due to its sharp peak at zero and heavy tails. Intuitively, this means a few significant independent events drive the observed data through the SVAR structure. This contrasts with prior work, which typically assumes zero-mean Gaussian input, either explicitly~\citep{lachapelle2019granDAG} or implicitly via mean square error-based optimization objectives~\citep{pamfil2020dynotears, sun2023ntsnotears, tank2021neuralGranger}. By incorporating this Laplacian input model, we derive a maximum likelihood estimator (MLE) based on least absolute error regression, leading to \mobius, a new method for efficient SVAR estimation from time-series data. This framework provides both theoretical and empirical advantages.

\paragraph{Contributions}
Our main contributions are:
\begin{itemize}
    \item We model sparse SVAR input as independent zero-mean Laplacian variables, yielding an MLE formulation for estimating SVAR parameters.
    \item We prove the consistency of this MLE under mild assumptions on the window graph weights.
    \item We introduce \mobius, a regularized MLE framework enabling fast, and accurate SVAR estimation from time-series data.
    \item In synthetic experiments with sparse SVAR input, generated via Laplacian or Bernoulli-uniform distribution as in~\citep{misiakos2024fewrootcauses}, \mobius can learn an associated DAG with up to several thousands of nodes and outperforms various state-of-the-art methods.
    \item On real-world financial data from the S\&P 500 index, we show that the sparse input assumption allows to cluster stocks by sector and identify structural shocks reflecting significant changes in the stock prices.
\end{itemize}

\section{SVAR with Sparse Input}
\label{sec:svar}

We introduce notation, the needed background on SVARs, the motivation for sparsity assumption in the input and its statistical modeling using the Laplace distribution.

\paragraph{Time-series data} A multi-dimensional data vector $\vx_t$, measured at time point $t \in {0,1,\dots,T-1} = [T]$, is written as $\vx_t = (x_{t,1}, x_{t,2}, \dots, x_{t,d}) \in \R^{1\times d}$. A time series consists of a sequence of such data vectors $\vx_0, \dots, \vx_{T-1}$ recorded at consecutive time points. We assume these vectors are stacked as rows in a matrix, representing the entire time series, denoted as $\mX \in \R^{T \times d}$.
When multiple realizations of $\mX$ are available, they are collected as slices of a tensor $\tX \in \R^{N \times T \times d}$. These can obtained by dividing a long time series into smaller segments of length $T$.

\paragraph{Example: stock market} We consider an example of time-series data from the stock market. We collect daily stock values $\vx_t$ for a particular stock index (e.g., S\&P $500$) for, say, 20 years. A time series for one year is denoted with the matrix $\mX$ and $20$ years yield the data tensor $\tX$. 

\paragraph{Model Demonstration}
We impose a graph-based model on the generation of time-series data and first illustrate it with a simple example. Suppose that the vector $\vx_t$ at time $t$ is generated from the previous time step’s data $\vx_{t-1}$ according to the equation:
\begin{equation}
\vx_t = \vx_{t-1}\mB + \vs_t,
\label{eq:VAR_1}
\end{equation}
where $\vs_t$ represents the input variables, commonly referred to as structural shocks~\citep{kilian2013SVAR}, though they have also been described as root causes\citep{misiakos2024fewrootcauses}. Given $\vs_t$, the data $\vx_t$ is fully determined by the matrix $\mB$ through~\eqref{eq:VAR_1}.
The model in~\eqref{eq:VAR_1} is an instance of vector autoregression (VAR)~\citep{kilian2013SVAR}. The $(i,j)$ entry of the matrix $\mB \in \mathbb{R}^{d \times d}$ quantifies the influence of $x_{t-1,i}$ on $x_{t,j}$. This corresponds to the adjacency matrix of a directed graph $\gG = \left(\sV, \mB\right)$, where $\sV$ is the set of nodes enumerated as $\sV = {1,2,...,d}$. The primary objective is to learn $\mB$ from time-series data $\{\vx_t\}_{t\in[T]}$.
The model in~\eqref{eq:VAR_1} is stationary, meaning that $\mB$ remains constant across all time steps. Additionally, it has a time lag of one, as each observation $\vx_t$ depends only on the previous time step $\vx_{t-1}$ and the newly introduced inputs $\vs_t$ at time $t$.

\paragraph{Example} In the stock market example, the stocks ${1, 2, \dots, 500}$ in the S\&P 500 market index would represent the nodes of a graph and $\mB$ would encode the influences between these stocks. The model then would imply that the value $x_{t,i}$ of stock $i$ on day $t$ is determined by the stock values $\vx_{t-1}$ from day $t-1$, combined with a structural shock $s_{t,i}$ representing an event occurring on day $t$.

\paragraph{Structural vector autoregression} An SVAR~\citep{lutkepohl2005new,pamfil2020dynotears} expands the VAR in~\eqref{eq:VAR_1} to the general form with time lag $k$. Namely, we assume there exist adjacency matrices $\mB_{0},\mB_{1},...,\mB_{k}\in\R^{d\times d}$ and  $\vs_t\in\R^{1\times d}$, 
such that $\vx_t = \bm{0}$ for $t<0$ and for $t \in [T]$
\footnote{We provide a stability condition for~\eqref{eq:svar_lag_k} in App.~\ref{appendix:subsec:stability}.} :
\begin{equation}
    \vx_t = \vx_{t}\mB_{0} + \vx_{t - 1}\mB_{1}+ ... + \vx_{t - k}\mB_{k} + \vs_t.
    \label{eq:svar_lag_k}
\end{equation}
The $(i,j)$ entry of $\mB_\tau$ represents the influence of $i$ to $j$ after $\tau$ time steps (i.e., a lag of $\tau$) and $\vs_t$ are the structural shocks.
$\mB_{0}$ represents the \textit{instantaneous} dependencies, while the $\mB_{1},...,\mB_{k}$ represent the \textit{lagged} dependencies. 
The SVAR is \textit{stationary}, since the $\mB_\tau$ do not depend on $t$. Following \citet{pamfil2020dynotears} we assume that $\mB_{0}$ corresponds to a DAG, ensuring that the recurrence~\eqref{eq:svar_lag_k} is solvable for $\vx_t$.

The instantaneous $\mB_{0}$ and lagged dependencies $\mB_{1},...,\mB_{k}$ are collected as block-rows in a matrix $\mW \in\R^{d(k+1)\times d}$ which forms the so-called window graph
depicted with an example in Fig.~\ref{fig:causes-data-window}. 
Note that the window graph is a DAG since the edges go only forward in time.
The problem we aim to solve is to infer the window graph $\mW$ from time-series data under the assumption that there are few significant structural shocks. To achieve this, our approach imposes a sparsity assumption on the input $\vs_t$.

\paragraph{Example}
In the previous stock market example, the matrix $\mB_0$ represents instantaneous influences within the same day, while the other matrices $\mB_\tau$ capture influences across different days. Since stock markets typically react almost instantaneously to new information, one would expect most dependencies to be reflected in $\mB_0$.

\begin{figure}[t]
\vspace{-5pt}
    \centering
    \includegraphics[width=0.95\linewidth]{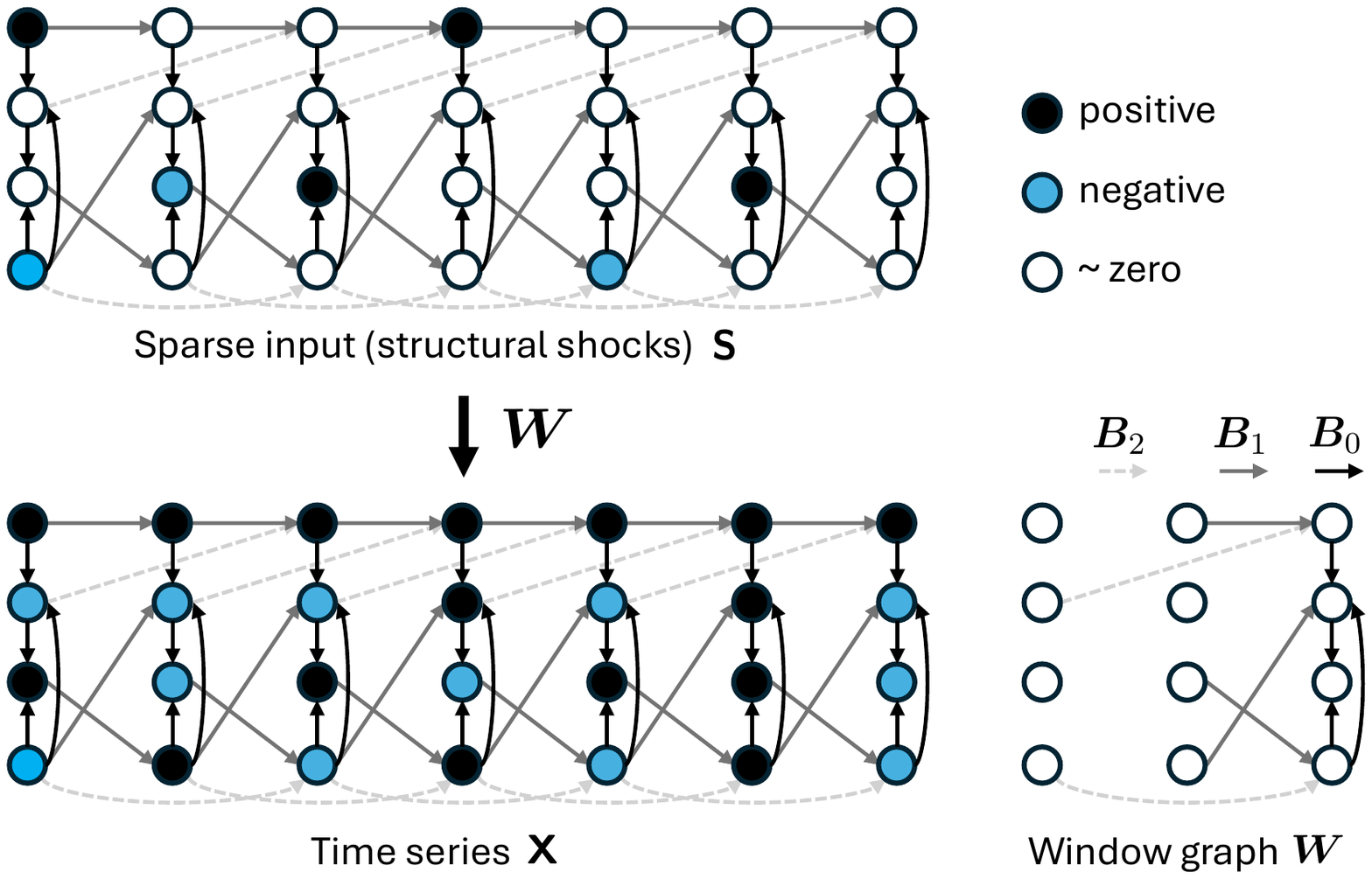}
    \vspace{-15pt}
    \caption{Visualizing an SVAR~\eqref{eq:SVAR} with sparse input $\tS$. Out of $28$ structural shocks in $\tS$ only seven are significant (positive or negative) and the rest are approximately zero. The window graph $\mW$, composed of $\mB_0, \mB_1, \mB_2$, generates the observed dense time series $\tX$ (bottom) via~\eqref{eq:SVAR}.}
    \label{fig:causes-data-window}
    \vspace{-5pt}
\end{figure}

\paragraph{Sparse input} We denote with $\vx_{t,\text{past}} = \left(\vx_t,...,\vx_{t-k}\right),$ $t\in[T]$, the data at previous time steps of $\vx_t$ with lag up to a chosen fixed $k$. Analogously, $\mX_{\text{past}}\in \sR^{T\times d(k+1)}$ contains as rows the vectors $\vx_{t,\text{past}},\,t\in[T]$ and $\tX_{\text{past}}\in\sR^{N\times T\times d}$ contains $N$ realizations of $\mX_{\text{past}}$.
With this notation, the SVAR~\eqref{eq:svar_lag_k} can be written in the following matrix format:
\begin{equation}
    \mX = \mX_{\text{past}}\mW + \mS \Leftrightarrow \tX = \tX_{\text{past}}\mW + \tS.
    \label{eq:SVAR}
\end{equation}
Intuitively, the non-zero values in $\tS$ represent unobserved events that propagate through space (according to $\mB_0$) and also through time $t$ (according to $\mB_1,...,\mB_k$) to generate $\mX$ via~\eqref{eq:SVAR}. 
In Fig.~\ref{fig:causes-data-window} we illustrate the data generation process~\eqref{eq:SVAR}. 
In the upper part, the significant structural shocks $\tS$ are denoted in color, whereas white nodes correspond to (approximately) zero values (noise). 

\paragraph{Example} In our stock market example, the structural shocks $\vs_t$ would represent significant events (big news) that trigger changes in the prices of the stocks at day $t$. Examples include unexpected quarterly results, administrative changes in the company, capital investment, lancing a new product, etc. 
It is intuitive that significant events happen rarely and affect few stocks every day, and thus $\tS$ is sparse. Later, we confirm the sparse input assumption in experiments with real-world financial time series.

\paragraph{Laplace Distribution}
In practical applications, input sparsity can only be approximately satisfied. Therefore, we consider a distribution for $\tS$ that encourages sparsity formation. A natural choice is the $\text{Laplace}(0, \beta)$ distribution, which is characterized by a sharp peak at $0$ and heavy tails~\citep{jing2015sparsematrixfactLaplace}.
\citet{tibshirani1996LASSOregression} introduced the classical LASSO regression by adopting the Laplace prior, leading to the well-known $\normlone$ regularizer that promotes sparsity. The Laplace prior has also been used in Bayesian linear regression\citep{castillo2015bayesiansparseregression}, compressive sensing\citep{babacan2009bayesianComprSensing}, sparse matrix factorization\citep{jing2015sparsematrixfactLaplace}, and sparse principal component analysis (PCA)\citep{guan2009sparsePCA}.
Based on this, we impose the following assumption on $\mS$ and derive its probability density function $f_S$:
\begin{equation}
    \mS_{t,j}\sim\text{Laplace}(0,\beta)\Leftrightarrow f_S(\mS_{t,j}|\beta) = \frac{1}{2\beta}e^{-\frac{\left|\mS_{t,j}\right|}{\beta}}.
    \label{eq:laplace_model}
\end{equation}
We denote the unknown ground truth $\beta$ parameter as $\beta^*$.

\section{Learning the SVAR}
\label{sec:method}

In this section, we establish the identifiability of our setting, derive the Laplacian MLE, prove its consistency, and formulate the proposed optimization framework, \mobius.

\paragraph{Identifiability} 
A fundamental question in causal discovery is whether the graph structure is identifiable from the data~\citep{park2020conditional}. 
Let $\mW^*$ be the ground-truth window graph, and $f_X(\tX | \mW, \beta)$ denote the probability density function of the data, parameterized by $(\mW, \beta)$.
Identifiability means that if $f_X(\tX|\mW, \beta) = f_X(\tX|\mW^*, \beta^*)$, then necessarily $\mW = \mW^*$. 
This ensures that the window graph $\mW^*$ is uniquely determined by the data distribution. 
Theorem~\ref{th:identifiability} establishes the identifiability of $\mW$ and the parameter $\beta$, which is a necessary condition for our consistency result.

\begin{theorem}
    Consider the time-series model~\eqref{eq:SVAR} with $\mS$ following a multivariate Laplace distribution~\eqref{eq:laplace_model} with $\beta^* > \frac{1}{NTd}$.
    Then the adjacency matrices $\mB_{0},\mB_{1},...,\mB_{k}\in\sR^{d\times d}$ and $\beta$ are identifiable from the time-series data $\tX$. 
    \label{th:identifiability}
\end{theorem}

\begin{proof}[Proof sketch]
    We unroll $\mW$ over time into a DAG and rewrite~\eqref{eq:SVAR} as a linear SEM, as explained in~\citep{misiakos2024icassp}. The identifiability then follows from LiNGAM~\citep{shimizu2006lingam}, since $\tS$ follows a Laplacian distribution. The window graph is identified by extracting $\mB_0, \mB_1, \dots, \mB_k$ from the unrolled DAG. The parameter $\beta$ is identified using the monotonicity of the probability density function. A full proof is provided in App.~\ref{appendix:subsec:identifiability}.
\end{proof}

\paragraph{Laplacian MLE} 
The MLE is a fundamental statistical method for estimating model parameters by maximizing the likelihood function $f_X(\tX|\mW, \beta)$ given the observed data $\tX$. Under the Laplacian noise model~\eqref{eq:laplace_model}, the probability density function of $\mX$ is given by (see App.~\ref{appendix:subsec:MLE_computation} for details):
\begin{equation}
    f_X(\tX|\mW,\beta) =
    \frac{\left|\text{det}\left(\mI - \mB_0\right)\right|^{NT}}{(2\beta)^{NdT}}
    e^{-\normii{\tX - \tX_{\text{past}}\mW}/\beta}.
\end{equation}
The MLE seeks to find the optimal parameters by maximizing the likelihood function. Equivalently, we maximize the log-likelihood function $\mathcal{L}\left(\tX|\mW,\beta\right)= \log f_X\left(\tX|\mW,\beta\right)$:
\begin{align}
    \mathcal{L}\left(\tX|\mW,\beta\right) &= NT\log \left|\text{det}\left(\mI - \mB_0\right)\right|  
    - NTd \log(2\beta) \notag  \\&-\frac{1}{\beta} \normii{\tX - \tX_{\text{past}}\mW}.
    \label{eq:loglikelihoodMLE}
\end{align}
Thus, the MLE estimate $\widehat{\mW}$ is given by:
\begin{equation}
    \widehat{\mW} = \argmax_{\mW\in\ml{W}} \mathcal{L}\left(\tX|\mW,\beta\right).
    \label{eq:MLE}  
\end{equation}
A desirable property of the MLE is that $\widehat{\mW} = \mW^*$. Under the assumption of identifiability, this property holds for the population log-likelihood~\citep{newey1994MLEconsistency}, defined as:
\begin{equation}
    \logpop{\mW,\beta} = \expvp{\mW^*,\beta^*}{\loglike{\tX}{\mW,\beta}},
\end{equation}
where $\logpop{\mW,\beta}$ represents the expected value of the log-likelihood function $\mathcal{L}\left(\tX|\mW,\beta\right)$ computed under the ground truth probability density $f_X(\tX|\mW^*, \beta^*)$.  
Intuitively, it corresponds to the log-likelihood if we had access to infinitely many samples.  
The following lemma formalizes this property, with a proof provided in App.~\ref{appendix:subsec:MLE_consistency_background}.
\begin{lemma}
    Assume that the ground truth parameters  $(\mW^*,\beta^*)$ are identifiable from the data distribution $f_X\left(\mX|\mW^*,\beta^*\right)$.  
    Then, the population likelihood $\logpop{\mW,\beta}$ has a unique maximum at $(\mW^*,\beta^*)$.
    \label{lemma:uniqueMLEmaximizer}
\end{lemma}
Lemma~\ref{lemma:uniqueMLEmaximizer} implies that with infinite data, the log-likelihood has a unique global maximizer at the ground truth $\mW^*$.  
However, since we only have a finite dataset, we require a stronger result for the empirical log-likelihood $\mathcal{L}\left(\tX|\mW,\beta\right)$.

\paragraph{Consistency of MLE} We prove the consistency of the MLE, which states that as the amount of data increases, $\widehat{\mW}$ converges in probability to $\mW^*$. Formally, we show the following result.
\begin{theorem}
    The maximum log-likelihood estimator~\eqref{eq:MLE} satisfies the conditions of Theorem 2.5 of \citet{newey1994MLEconsistency} and thus is consistent under the following assumptions:
    \begin{itemize}
        \item The space of window graphs is $\ml{W}\subseteq [-1,1]^{d(k+1)\times d}$ and $\mB_0$ acyclic.
        \item $\beta\in [a,b]$ is bounded, with a lower bound $a > 1/NTd$.
    \end{itemize}
    \label{th:consistency}
\end{theorem}

\begin{proof}[Proof sketch]     
    The proof requires a compact search space for $\mW$, which is satisfied by the given bounds on $\mW$ and $\beta$. Additionally, the set of acyclic matrices is closed, as it can be expressed as the pre-image $h^{-1}(\{0\})$, where $h$ is a continuous function characterizing acyclicity~\citep{zheng2018notears}. Identifiability of $\mW$ and $\beta$ is ensured by Theorem~\ref{th:identifiability}. Finally, the log-likelihood is continuous, and it can be shown that $\sup_{\mW \in \ml{W}} \left|\loglike{\mX}{\mW}\right|$ has finite expectation. Under these requirements, Theorem 2.5 of \citet{newey1994MLEconsistency} then utilizes the uniform law of large numbers to show that $\widehat{\mW}$ converges in probability to $\mW^*$. A full proof is provided in App.~\ref{appendix:subsec:MLE_consistency_proof}.
\end{proof}

\paragraph{Our method \mobius} 
Theorem~\ref{th:consistency} implies that $\widehat{\mW}$ in~\eqref{eq:MLE} converges in probability to $\mW^*$ as $N \to \infty$. Since the parameter $\beta$ is fixed but unknown, we estimate it by maximizing the log-likelihood function~\eqref{eq:loglikelihoodMLE}. Following \citet{ng2020GOLEM}, we compute an estimate $\widehat{\beta}$ by solving:
\begin{equation}
    \frac{\partial \mathcal{L}}{\partial \beta} = 0 \Leftrightarrow  \widehat{\beta} = \frac{1}{NTd}\normii{\tX - \tX_{\text{past}}\mW}.
\end{equation}
This estimate is consistent in expectation. Indeed, it is true that $\expv{\normii{\tX - \tX_{\text{past}}\mW}} = \expv{\normii{\tS}} = NTd\beta^*$.  
Thus, the log-likelihood maximization problem for approximating $\mW$ reduces to (see App.~\ref{appendix:subsec:optimization_derivation} for details):
\begin{align}
        \widehat{\mW} &= \argmax_{\mW\in\ml{W}} \mathcal{L}\left(\tX|\mW,\widehat{\beta}\right)\label{eq:reduced_MLE}\\
    &= \argmin_{\mW\in\ml{W}} \log\normii{\tX - \tX_{\text{past}}\mW} -\frac{1}{d}\log \left|\text{det}\left(\mI - \mB_0\right)\right|. \notag
\end{align}
However, directly minimizing~\eqref{eq:reduced_MLE} over the space of DAGs is computationally inefficient. This would require enforcing a hard DAG constraint to restrict $\mW \in \ml{W}$, as in~\citep{zheng2018notears}, where it is implemented via the augmented Lagrangian method. Such an approach demands careful fine-tuning and can lead to numerical instabilities, as demonstrated by~\citet{ng2020GOLEM}. To overcome these challenges, following~\citet{ng2020GOLEM}, we relax the hard acyclicity constraint and introduce a soft regularizer. This approach maintains strong performance while improving efficiency, as demonstrated in our experiments. The final optimization problem for \mobius is formulated as:
\begin{align}
    \widetilde{\mW} &=\argmin_{\mW\in\sR^{d(k+1) \times d}}  \log\normii{\tX - \tX_{\text{past}}\mW} \label{eq:cont_opt} \\
    &-\frac{1}{d}\log \left|\text{det}\left(\mI - \mB_0\right)\right| 
    + \lambda_1 \|\mW\|_1  + \lambda_2\cdot h\left(\mB_0\right).\notag
\end{align}
The first term in~\eqref{eq:cont_opt} promotes sparsity in the structural shocks, while the remaining terms encourage sparsity in the window graph $\mW$ and enforce acyclicity in $\mB_0$, respectively. The acyclicity regularizer $h\left(\mB_0\right) = e^{\mA \odot \mA} - d$, introduced by~\citet{zheng2018notears}, ensures that $\mB_0$ satisfies the DAG constraint. Notably,~\eqref{eq:cont_opt} is well-suited for GPU acceleration using tensor operations, making it highly efficient in practice.
In our implementation, we represent $\mW$ as the parameter matrix of a (PyTorch) linear layer with $(k+1)d$ inputs and $d$ outputs. The precomputed $\tX_{\text{past}}$ serves as input, and the linear layer’s output is subtracted from the observed data $\tX$. The objective in~\eqref{eq:cont_opt} is then computed and optimized using the Adam optimizer~\citep{kingma2014adam}. More implementation details can be found in App.~\ref{appendix:sec:spinsvar_implementation}.

Since the proposed objective function is non-convex, it may have multiple local optima, and there is no guarantee of convergence to the global maximum. However, in practice, our method performs well and often even recovers the edges of $\mW^*$ without error. This phenomenon, also observed in GOLEM~\citep{ng2020GOLEM}, motivates further theoretical investigation. 

Once $\widehat{\mW}$ is obtained via~\eqref{eq:cont_opt}, we approximate the input $\widehat{\tS}$:
\begin{equation}
    \widehat{\tS} = {\tX} - {\tX}_{\text{past}}\widehat{\mW}.
    \label{eq:root_causes_estimation}
\end{equation}
In recovering $\tS$ from $\widehat{\tS}$, we are particularly interested in identifying significant structural shocks. To this end, we apply thresholding to filter out insignificant values in $\widehat{\tS}$. In our experiments, this threshold is selected based on the synthetic data generation process.

\section{Related Work}
\label{sec:related_work}
\paragraph{Time-series causal discovery} 
Our work falls within the category of continuous optimization methods but differs in its assumption of sparsity in the input of the SVAR. Closely related approaches include functional causal model-based methods such as \varlingam~\citep{hyvarinen2010varlingam}, which estimates an SVAR, as well as TiMINO~\citep{peters2013timino} and NBCB~\citep{assaad2021NBCB}, which recover only the summary graph that disregards time delays~\citep{causal2023temporaloverview}. In contrast, our method learns the full window graph.
Other continuous optimization methods include DYNOTEARS~\citep{pamfil2020dynotears}, NTS-NOTEARS~\citep{sun2023ntsnotears} for non-linear data, and iDYNO~\citep{gao2022idyno} for interventional data. These methods optimize the mean square error loss and do not impose sparsity on the SVAR input. In our experiments, we compare against these methods, as well as others that learn the window graph from observational time-series data, selecting both methodologically relevant approaches and representative alternatives.

Different from our approach, constraint-based methods infer edges using conditional independence tests. Examples include PCMCI~\citep{runge2019PCMCI}, tsFCI~\citep{entner2010tsFCI}, PCMCI+\citep{runge2020pcmci+}, LPCMCI\citep{gerhardus2020lpcmci}, PC-GCE~\citep{assaad2022PC-GCE}, and SVAR-FCI~\citep{malinsky2018SVAR-FCI}. Methods based on Granger causality typically recover only the summary graph. Notable examples include neural Granger causality~\citep{tank2021neuralGranger}, eSRU~\citep{khanna2019eSRU}, GVAR~\citep{marcinkevivcs2020GVAR}, and convergent cross mapping~\citep{sugihara2012CCM}. Another line of work leveraging neural networks includes TCDF~\citep{nauta2019TCDF}, SCGL~\citep{xu2019SCGL}, neural graphical modeling~\citep{bellot2021neuralgraphmodelling}, and amortized learning~\citep{lowe2022amortized}.

\paragraph{Maximum Likelihood Estimator}
By modeling sparsity with a Laplacian distribution, we derive an MLE objective based on least absolute error loss, unlike prior causal discovery methods~\citep{ng2020GOLEM, pamfil2020dynotears, nauta2019TCDF}, which use mean-square loss suited for Gaussian noise. \citet{peters2014identifiability} provide consistency guarantees of the MLE for a linear SEM with equivariant Gaussian errors and GranDAG~\citep{lachapelle2019granDAG} applies it to nonlinear additive noise models. 
However, these methods neither support time-series data nor enforce input sparsity. 
For SVAR estimation, \citet{hyvarinen2010varlingam} propose a generic MLE for non-Gaussian noise but do not integrate it explicitly in the methodology. 
Other MLE approaches for SVAR~\citep{lanne2017MLE_estimation_SVAR, fiorentini2023pseudoMLE_SVAR, maekawa2023pseudo_log_likelihood_nonGauss} remain generic and are not specific for Laplacian or sparse inputs.

\paragraph{Least Absolute Error and Sparsity}
The least absolute error (LAE) loss arises as an MLE when assuming that the SVAR input follows a Laplacian distribution~\cite{chai2019GeneralGaussianDistRegression, li2004fastMLE_LAD}, enforcing sparsity in the model. LAE has been widely used as a regression objective across various fields, including dynamical systems~\citep{jiang2023RLAD_dynamical_systems, he2024LAD_sparse_dynamical_systems}, due to its robustness against outliers compared to mean square error (MSE) loss~\citep{pollard1991asymptoticsLAD, bassett1978asymptotic_theoryLAD, kumar2015regression_model_LAD, narula1999minimumAbsErrorRegression}.
Despite this, the only method that employs LAE regression to enforce sparsity in the input of a linear SEM is SparseRC, proposed by~\citet{misiakos2024fewrootcauses}.
\citet{misiakos2024icassp} extended SparseRC to time-series graph learning by unrolling the window graph into a DAG, requiring the estimation of $(dT)^2$ parameters—rendering it computationally infeasible for our experiments.
Our method advances over SparseRC by formulating a Laplacian MLE to enforce sparse input, providing both consistency guarantees and improved computational efficiency in practice.

\section{Experiments}
\label{sec:experiments}

We compare \mobius to prior state-of-the-art work on learning the window graph $\mW$ from time-series data. Our experiments in this section cover synthetic and real data. Additional experiments are in Appendix~\ref{app:sec:more_experiments}.

\paragraph{Baselines} We compare against functional causal model methods \varlingam, \dlingam~\citep{hyvarinen2010varlingam}, and the GPU-accelerated \clingam~\citep{akinwande2024acceleratedlingam}, continuous optimization methods DYNOTEARS~\citep{pamfil2020dynotears} and SparseRC~\citep{misiakos2024fewrootcauses}, non-linear approaches NTS-NOTEARS~\citep{sun2023ntsnotears} and TCDF~\citep{nauta2019TCDF} and constraint-based methods tsFCI~\citep{entner2010tsFCI} and PCMCI~\citep{runge2019PCMCI}.
Among these, LiNGAM-based methods assume non-Gaussian SVAR input, which yields the most competitive performance but at the cost of higher computational complexity. SparseRC enforces input sparsity but times out; thus, we modify its setup to a smaller unrolled DAG (details in App.~\ref{appendix:exp:sparserc}). The other baselines do not enforce input sparsity. We compare the optimization objective and computational complexity of the baselines and \mobius in App.~\ref{appendix:subsec:comparison_baselines}. For the implementations we use public repositories (App.~\ref{appendix:exp:code_resources}), with hyperparameters tuned via grid search (App.~\ref{appendix:exp:hyperparameter}).


\paragraph{Metrics} We evaluate the unweighted approximation of $\mW$ using the structural Hamming distance (SHD), which counts the edge removals, insertions, and reversals needed to match the ground truth. The structural intervention distance (SID)~\citep{peters201SID} is omitted as it times-out for DAGs with thousands of nodes.  
Additional results in App.~\ref{appendix:exp:additional_metrics} include area under ROC curve (AUROC), F1 score, and normalized MSE (NMSE) for the weighted approximation of $\mW$. We also assess the detection of significant input values $\tS$ using SHD and NMSE for $\widehat{\tS}$.  
For all metrics, we report the mean and standard deviation (shown as shade) in Fig.~\ref{fig:synthetic_plots} over five experiment repetitions. In the real-world stock market dataset, where the ground truth is unknown, evaluation is purely empirical.

\begin{figure*}[t]
    \centering
    \vspace{-3pt}
    \begin{subfigure}{0.20\linewidth}
        \centering
        \includegraphics[width=\linewidth]{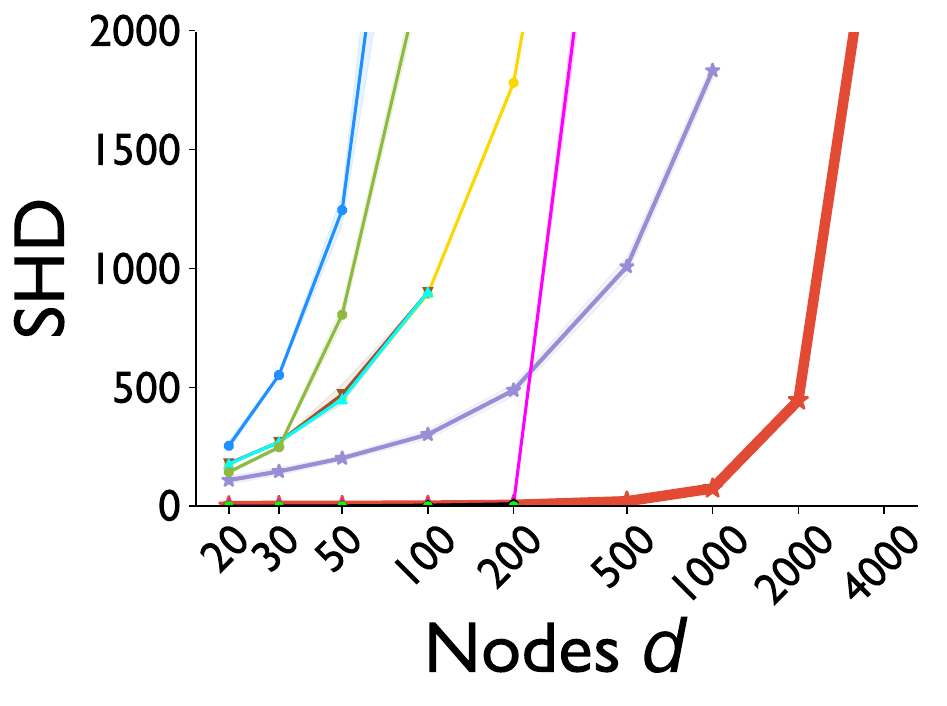}
        
        \includegraphics[width=\linewidth]{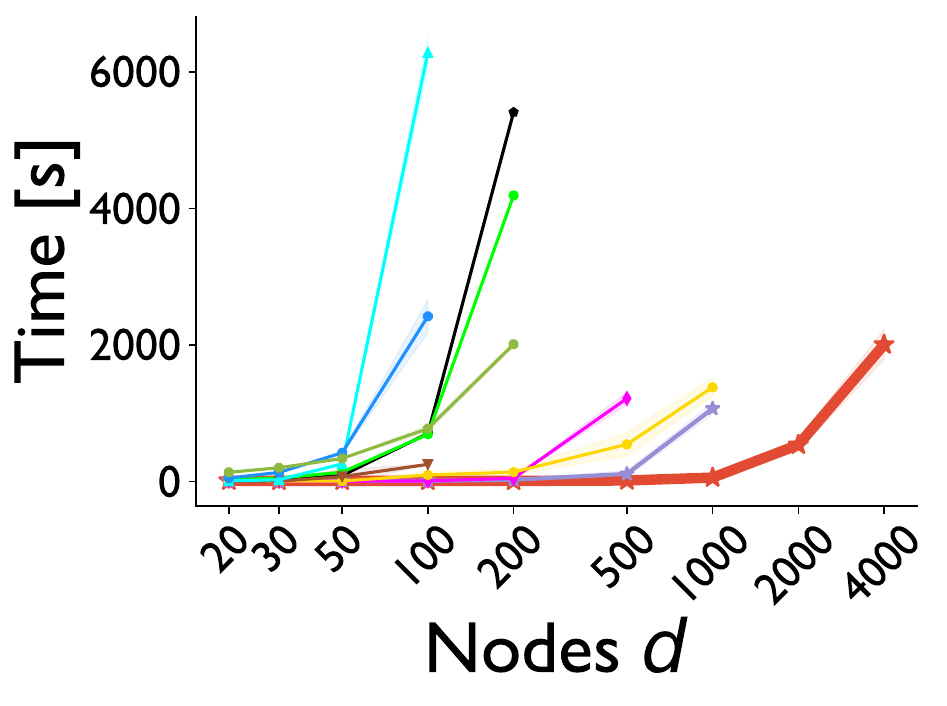}
        \caption{$N=10$, Laplace}
        \label{fig:synthetic_plots:samples10_laplace}
    \end{subfigure}
    \hfill
    \begin{subfigure}{0.20\linewidth}
        \centering
        \includegraphics[width=\linewidth]{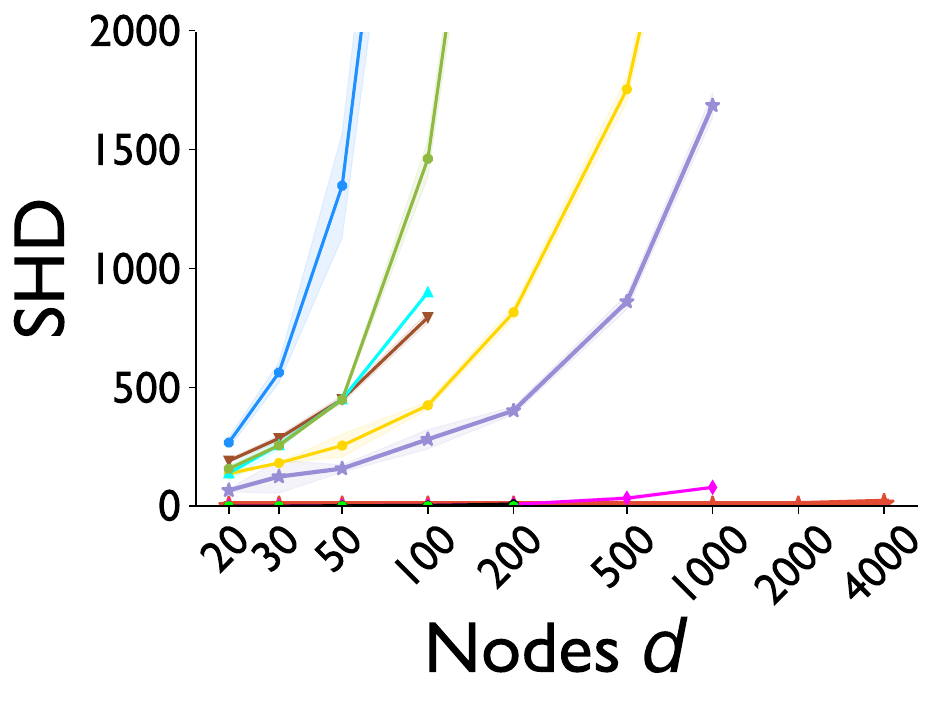}
        
        \includegraphics[width=\linewidth]{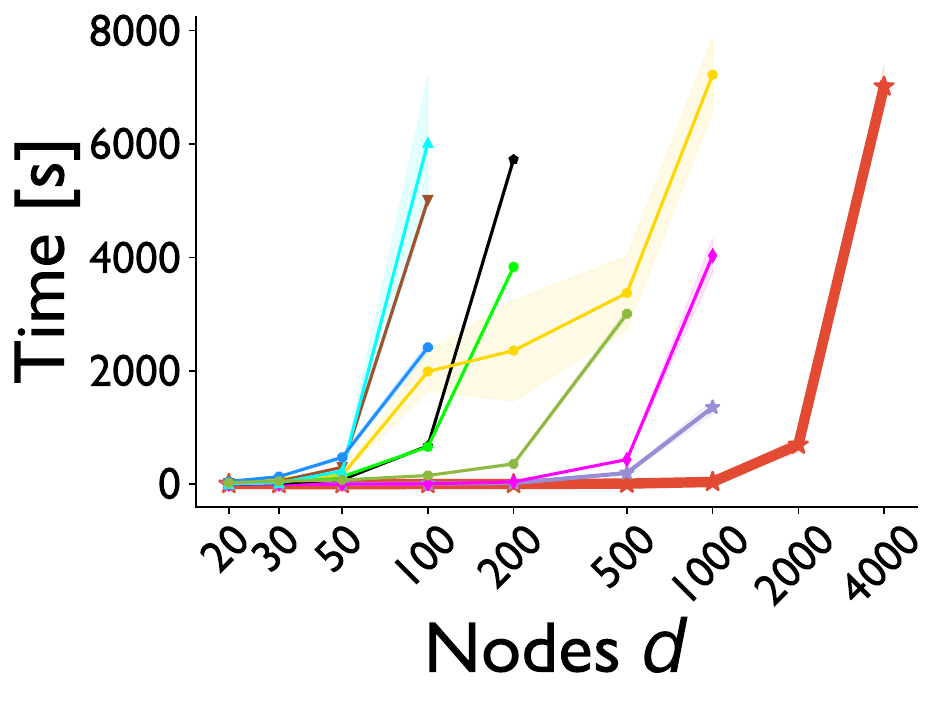}
        \caption{$N=10$, Bernoulli}
        \label{fig:synthetic_plots:samples10_bernoulli}
    \end{subfigure}
    \hfill
    \begin{subfigure}{0.20\linewidth}
        \centering
        \includegraphics[width=\linewidth]{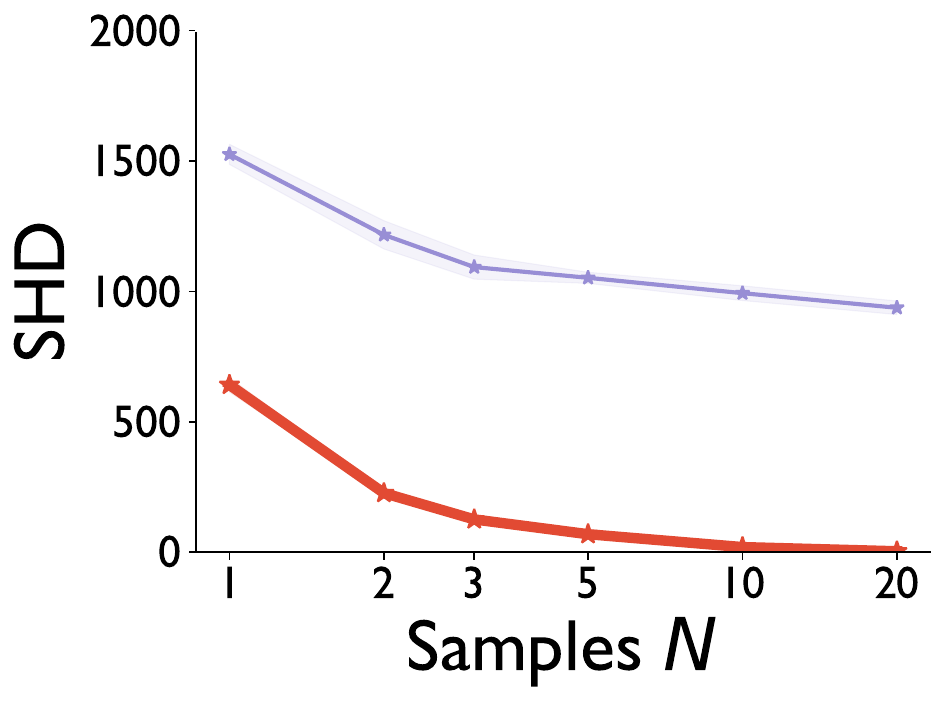}
        
        \includegraphics[width=\linewidth]{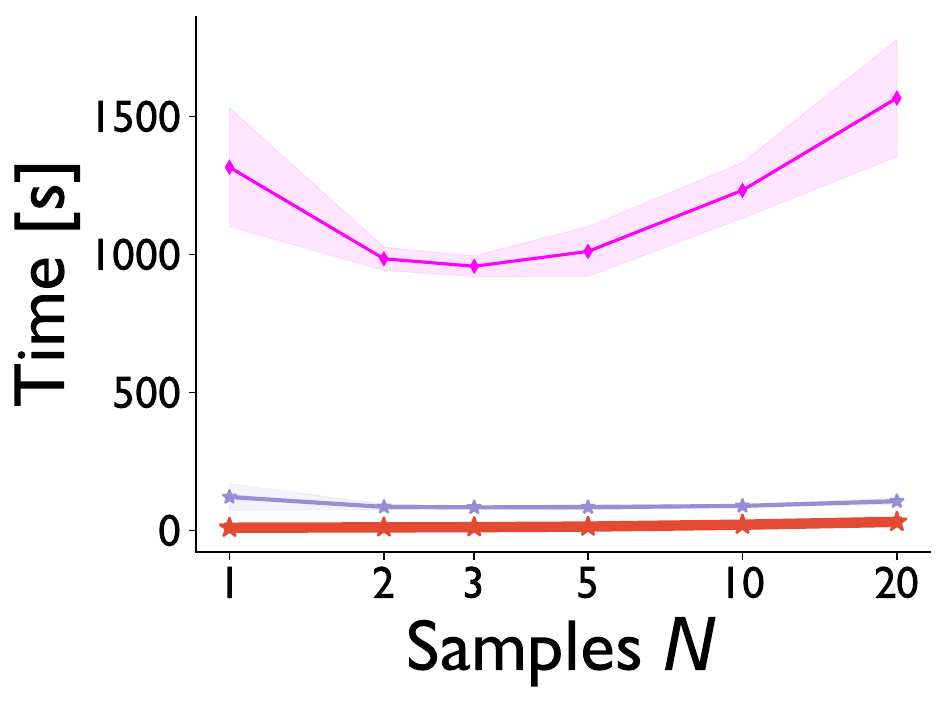}
        \caption{$d=500$, Laplace}
        \label{fig:synthetic_plots:nodes500_laplace}
    \end{subfigure}
    \hfill
    \begin{subfigure}{0.20\linewidth}
        \centering
        \includegraphics[width=\linewidth]{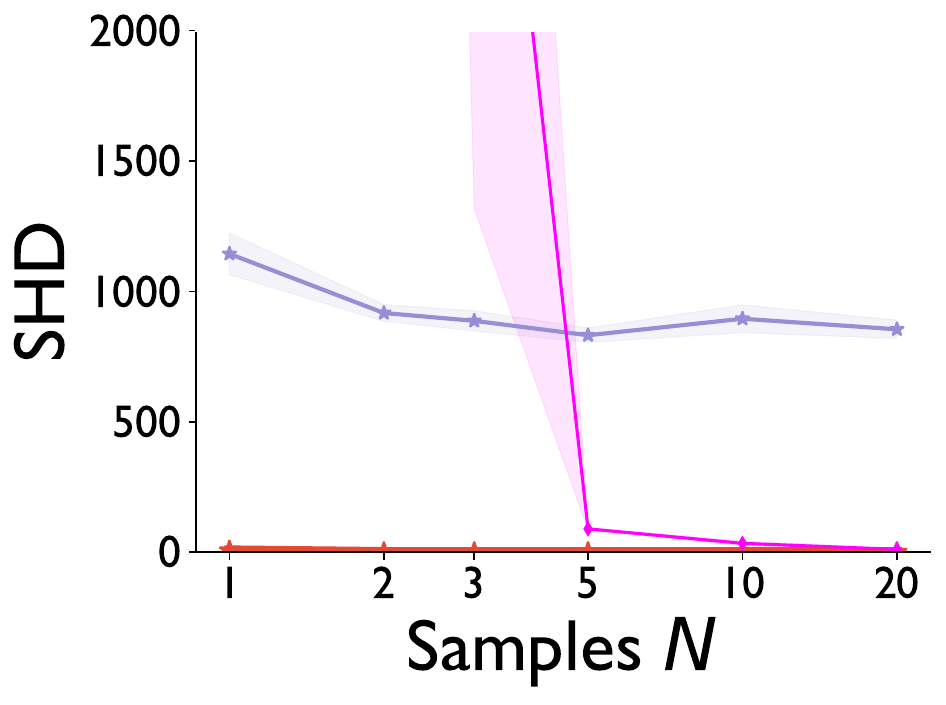}
        
        \includegraphics[width=\linewidth]{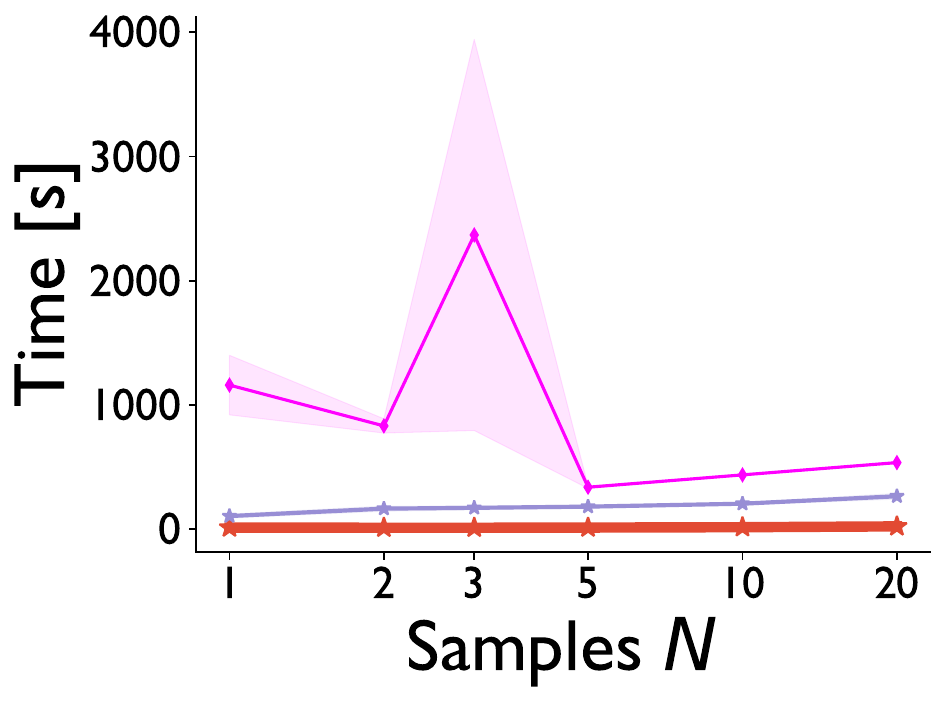}
        \caption{$d=500$, Bernoulli}
        \label{fig:synthetic_plots:nodes500_bernoulli}
    \end{subfigure}
    \hfill
    \begin{subfigure}{0.18\linewidth}
        \hspace{30pt}
        \includegraphics[trim={7.4cm 0 0 0}, clip, width=\linewidth]{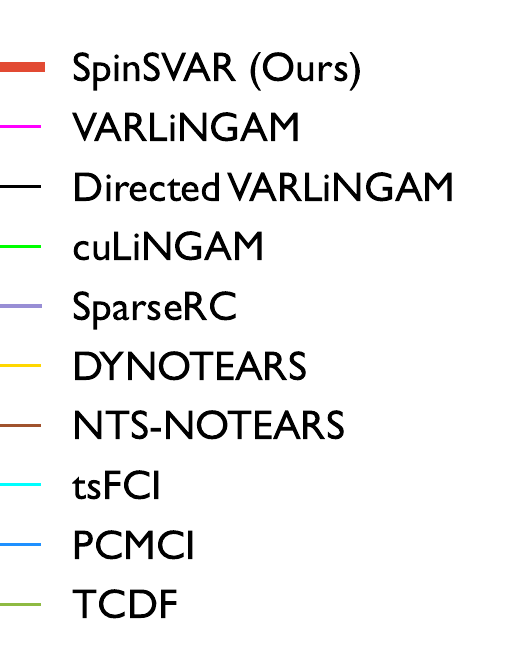}
        \vspace{30pt}
    \end{subfigure}
        \vspace{-15pt}
    \caption{Synthetic experiments. First row SHD (lower is better), second row runtime. (a), (b) consider $N = 10$ samples of time-series with $T = 1000$ and varying number $d$ of nodes for both input distributions. (c), (d) consider $d = 500$ nodes and varying number of samples $N$ of time-series of length $T = 1000$. Any non-reported point implies a time-out (execution time $> 10.000\text{s}\approx 2\text{:}45$h).}
    \label{fig:synthetic_plots}
    \vspace{-10pt}
\end{figure*}

\subsection{Synthetic experiments} 
\label{subsec:synthetic}

\paragraph{Data Generation} 
We generate data using the SVAR model~\eqref{eq:SVAR}, following settings similar to~\citep{pamfil2020dynotears} for the SVAR window graph $\mW$ and to~\citep{misiakos2024fewrootcauses} for the sparse SVAR input $\tS$. First, we set the number of nodes $d$, the length $T$ of the time series, the number of realizations $N$, and the maximum lag $k$ of the SVAR~\eqref{eq:svar_lag_k}. For the window graph $\mW$, we generate directed random Erdös-Renyi graphs for $\mB_0, \mB_1, \dots, \mB_k$, where $\mB_0$ is a DAG with an average degree of $5$, and $\mB_1, \dots, \mB_k$ have an average degree of $2$. We consider a default time lag of $k=2$ and include an additional version with $k=5$ in App.~\ref{appendix:subsec:large_time_lag}. The edges of $\mW$ are assigned uniform random weights from $[-0.5, -0.1] \cup [0.1, 0.5]$. The upper bound of $0.5$ ensures that~\eqref{eq:SVAR} is stable, and the generated data $\tX$ remain bounded in most cases (we discard $\tX$ if its entries become excessively 
large; see App.~\ref{appendix:exp:stability} for details). 

To impose sparsity in $\tS$, we consider two scenarios, using a threshold of $0.1$ to distinguish significant values from approximately zero values in $\tS$. First, we use the Laplacian distribution~\eqref{eq:laplace_model} with $\beta = \frac{1}{3}$, where in expectation only $5\%$ of values are significant (magnitudes greater than $0.1$; see App.~\ref{appendix:subsec:Laplace_properties}). Second, we use a Bernoulli distribution to control the percentage of significant entries in $\tS$~\citep{kalisch2007highDimDAGs, misiakos2024fewrootcauses}: each entry is non-zero with probability $p = 5\%$ (assigned uniform weights from $[-1, -0.1] \cup [0.1, 1]$) or zero otherwise. To create approximate sparsity, we add zero-mean Gaussian noise with a standard deviation of $0.01$ to $\tS$. We refer to this distribution as Bernoulli-uniform or simply Bernoulli.

\paragraph{Results} 
Fig.~\ref{fig:synthetic_plots} presents the results of our synthetic experiments for both sparsity scenarios of $\tS$ (Laplace and Bernoulli). Figs.~\ref{fig:synthetic_plots:samples10_laplace},~\ref{fig:synthetic_plots:samples10_bernoulli} correspond to a fixed number of samples, $N=10$, with the number of nodes $d$ ranging from 20 (180 edges) to 4000 (36,000 edges). In Figs.~\ref{fig:synthetic_plots:nodes500_laplace},~\ref{fig:synthetic_plots:nodes500_bernoulli}, we fix $d=500$ and vary the samples $N$ from 1 to 20. In all cases, the time-series length is $T=1000$. Baselines that are omitted either perform worse or time out.

In Figs.~\ref{fig:synthetic_plots:samples10_laplace},\ref{fig:synthetic_plots:samples10_bernoulli}, \mobius achieves the best performance, recovering $\mW$ nearly perfectly for Bernoulli inputs and handling up to $2000$ nodes for Laplacian inputs while maintaining the best runtime. Its computational complexity is superior to SparseRC, \varlingam, and its variants, and comparable to DYNOTEARS (see App.~\ref{appendix:subsec:comparison_baselines} for details). This efficiency stems from leveraging the sparse input assumption, enabling faster convergence with fewer iterations.
Baseline methods such as PCMCI, tsFCI, TCDF, NTS-NOTEARS, and DYNOTEARS perform poorly even on small graphs. The latter three rely on MSE loss, which is better suited for Gaussian inputs. SparseRC, which enforces sparsity via an LAE loss, exhibits slight improvements but struggles with larger graphs, timing out beyond $1000$ nodes.
The strongest baseline methods are \varlingam and its variants. \varlingam scales better but performs slightly worse, running up to $1000$ nodes before timing out. \dlingam and \clingam yield strong results but time out beyond $200$ nodes. For large graphs with Bernoulli input (Fig.~\ref{fig:synthetic_plots:samples10_bernoulli}), \varlingam remains competitive but is approximately $100$ times slower for $d=1000$.

For varying $N$ (Figs.~\ref{fig:synthetic_plots:nodes500_laplace},~\ref{fig:synthetic_plots:nodes500_bernoulli}), \mobius consistently excels in the Bernoulli case and improves as $N$ increases in the Laplace case. SparseRC performs poorly in both setups. \varlingam struggles with Laplacian input and requires more samples than \mobius in the Bernoulli setup.
Additional results for varying $d$ at fixed $N=1$ 
in App.~\ref{appendix:exp:additional_metrics} confirm these trends.  

\paragraph{Larger graphs} In Table~\ref{tab:sample_complexity} we evaluate \varlingam and \mobius on graphs with up to $d=4000$ nodes, varying the number of samples. These were the only methods that maintained reasonable performance without timing out at $d=1000$. \varlingam struggles with increasing graph sizes, timing out beyond $d=1000$, and requiring significantly more samples for reasonable SHD. In contrast, \mobius achieves strong results with fewer samples, particularly in the Bernoulli case. For Laplacian input, it requires slightly more samples to match that performance. Remarkably, \mobius can nearly perfectly recover a window graph with $3 \times 4000$ nodes (including time lags) and $16 \times 1000$ time points in $6759$s for Bernoulli input. 

\paragraph{Time lag \boldmath$k$} In App.~\ref{appendix:subsec:more_time_lags}, we present additional experiments on the sensitivity of the time lag $k$, showing that \mobius performance remains unaffected as long as it parametrizes a large enough time lag. In real-world datasets, where the true value of $k$ is unknown, we choose a large enough $k$ such that $\mB_k$ is approximately zero, making it highly unlikely that meaningful dependencies exist at even higher lags.

\begin{table}[t]
    \centering
    \caption{SHD report for large DAGs ($T = 1000$).}
    \vspace{-5pt}
    \resizebox{\linewidth}{!}{%
    \begin{tabular}{@{}llllll@{}}
    \toprule
      \mobius \hfill $N=$ & $1$ & $2$ & $4$ & $8$ & $16$\\
    \midrule
        $d=1000$, $\tS\sim$ Laplace & $8.3k$ & $1k$ & $371$ & $112$ & $\boldsymbol{27}$ \\
    $d=1000$, $\tS\sim$ Bernoulli & $2$ & $\boldsymbol{0}$ & $\boldsymbol{0}$ & $\boldsymbol{0}$ & $\boldsymbol{0}$ \\
    $d=2000$, $\tS\sim$ Laplace & $18k$ & $17k$ & $2.1k$ & $645$ & $183$ \\
    $d=2000$, $\tS\sim$ Bernoulli  & $12$ & $\boldsymbol{0}$ & $\boldsymbol{0}$ & $\boldsymbol{0}$ & $\boldsymbol{0}$  \\
    $d=4000$, $\tS\sim$ Laplace & $36k$ & $36k$ & $33k$ & $4.5k$ & $1.2k$ \\
    $d=4000$, $\tS\sim$ Bernoulli & $164$ & $27$ & $15$ & $\boldsymbol{7}$ & $\boldsymbol{9}$  \\
    \midrule
    \varlingam  \hfill $N=$& $1$ & $2$ & $4$ & $8$ & $16$\\
    \midrule
    $d=1000$, $\tS\sim$ Laplace  & $-$ & $-$ & $-$ & $-$ & $-$ \\
    $d=1000$, $\tS\sim$ Bernoulli  & $-$ & $-$ & $-$ & $115$ & $29$ \\
    \bottomrule
    \end{tabular}
    }
    \vspace{-10pt}
    \label{tab:sample_complexity}
\end{table}

\subsection{Application: S\&P 500 stock data}

\paragraph{Dataset} 
We consider stock values from the Standard and Poor's (S\&P) 500 market index. We gather data from March 1st, $2019$, to March 1st, $2024$, focusing only on stocks present in the index throughout this period, leaving $d=410$ stocks as nodes. We collect daily closing values for each stock, resulting in $1259$ time points per stock. The data values are computed as normalized log-returns~\citep{pamfil2020dynotears}, defined for stock $i$ at day $t$ as $x_{t,i} = \log(y_{t+1,i}/y_{t,i})$, where $y_{t,i}$ is the closing value. We partition the time series into shorter intervals of $50$ days length to obtain time-series data $\tX$ of shape $25 \times 50 \times 410$. Using these data, we learn a window graph $\widehat{\mW}$ that captures temporary relations between stocks and the underlying input $\widehat{\tS}$ that generates the data.

\paragraph{Learning Stock Relations} 
We execute all baselines with hyperparameters set according to a simulated experiment shown in App.~\ref{appendix:exp:simulated}. Fig.~\ref{fig:stocks_lag0} shows the \mobius estimate for $\widehat{\mB}_0$, representing instantaneous relations between stocks. A similar figure is discovered by SparseRC, but other baselines did not yield reasonable results with our chosen hyperparameters or those from the published papers (see App.~\ref{appendix:exp:real}). Below, we analyze this result and argue that the sparse input assumption yields interpretable results for financial data.

For better visualization, we focus on the $45$ highest-weighted stocks in the S\&P 500 index. In the execution of \mobius, we set a maximum time lag of $k=2$, but the method discovered that only $\mB_0$ was significant. This aligns with the efficient market hypothesis~\citep{fama1970efficient}, which states that stock prices fully reflect all available information, making past data redundant. Fig.~\ref{fig:stocks_lag0} can be interpreted well: the edges of $\widehat{\mB}_0$ roughly cluster stocks according to their economic sectors. A few outliers arise due to major IT companies being spread across multiple sectors. For example: (i) MSFT influences GOOG and AMZN, (ii) META, AAPL, and MSFT influence AMZN, and (iii) AMZN influences GOOG and MSFT. Notably, the weights of $\widehat{\mB}_0$ are positive, indicating that these stocks positively influence each other: when one increases or decreases, the others do so as well.

\begin{figure}[t]
    \begin{subfigure}{\linewidth}
        \centering
        \includegraphics[width=1.1\linewidth]{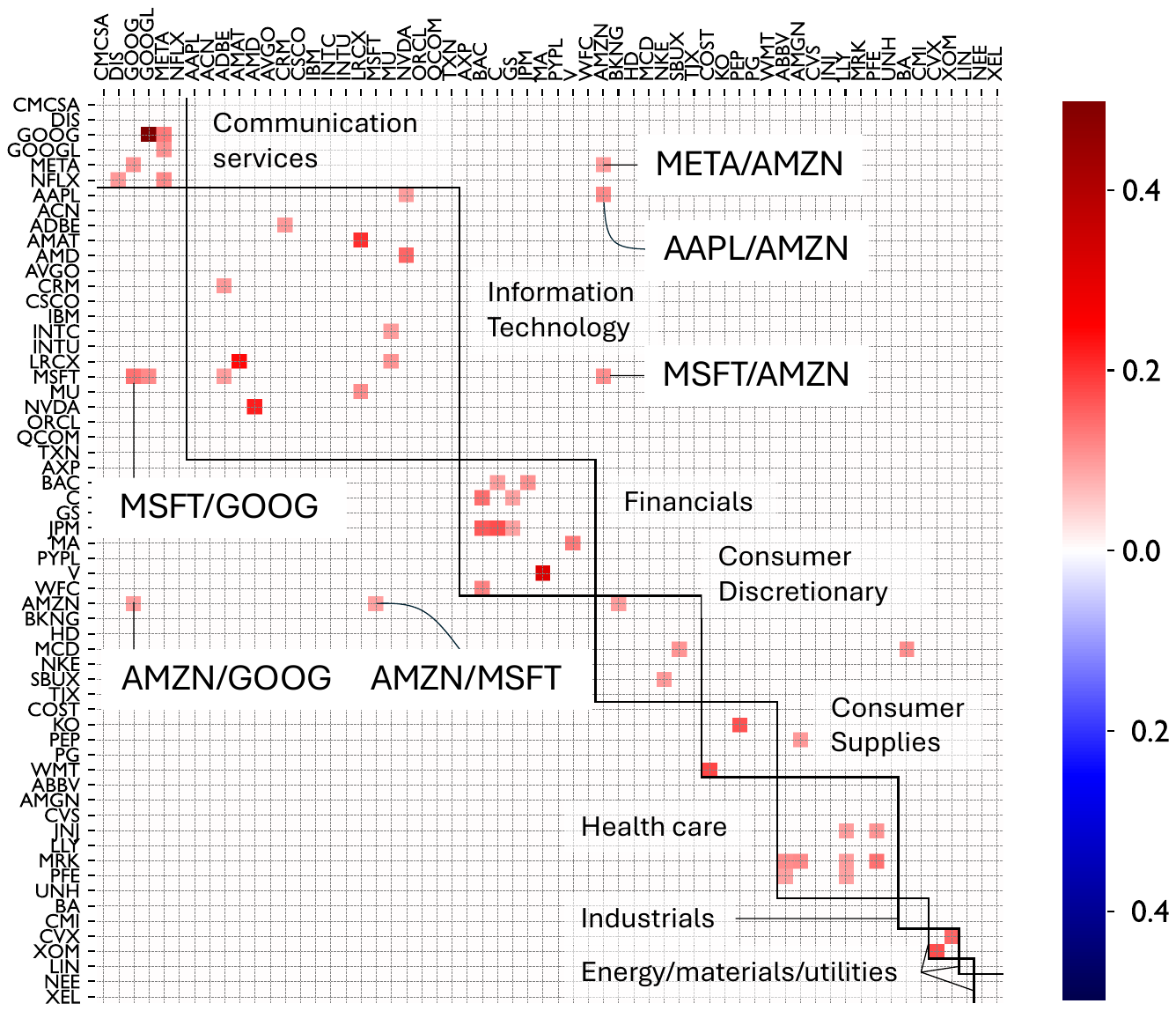}
        \vspace{-23pt}
        \caption{\mobius estimate for $\widehat{\mB}_0$}
        \label{fig:stocks_lag0}
    \end{subfigure}

    \begin{subfigure}{\linewidth}
        \centering
        \includegraphics[width=\linewidth]{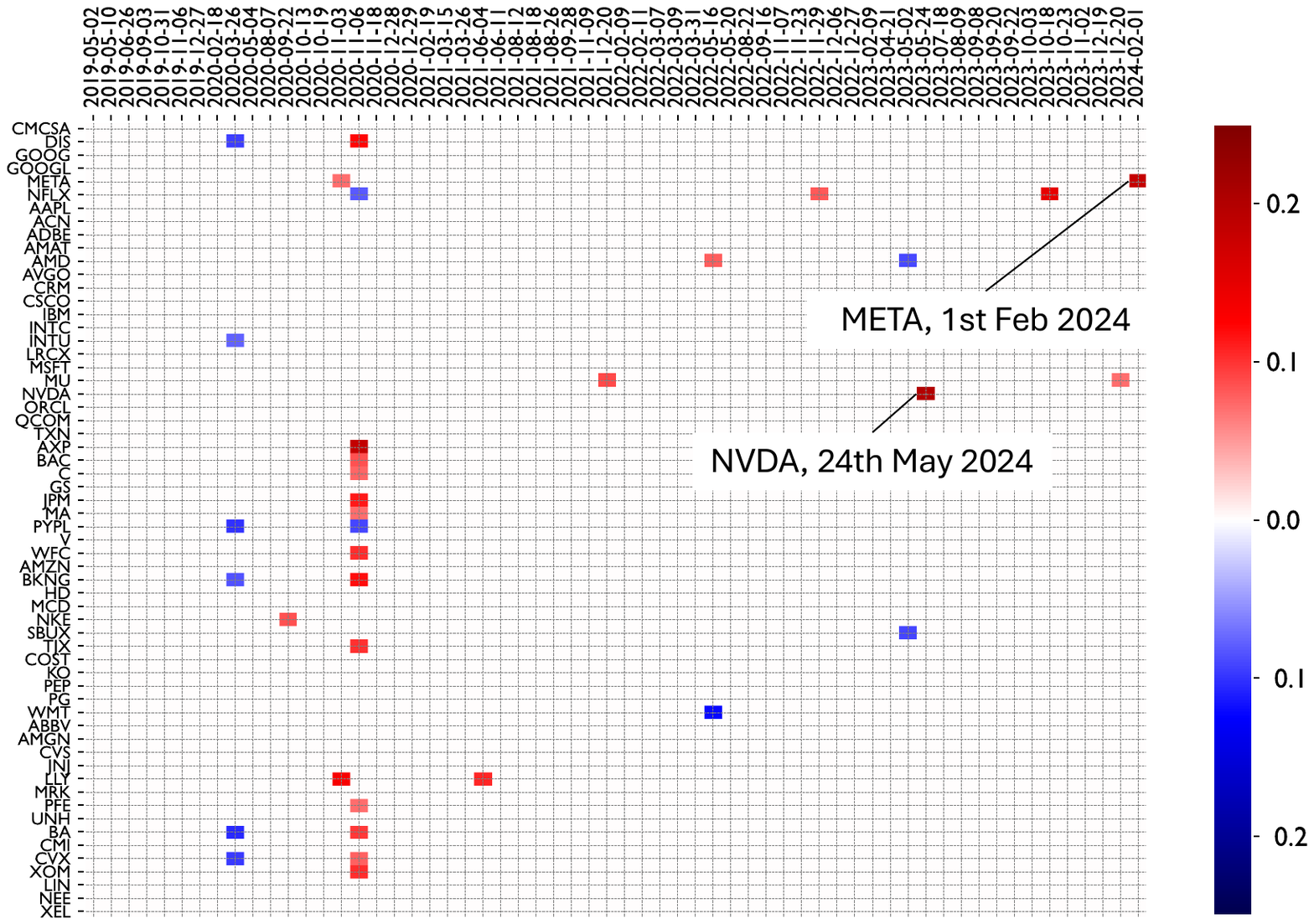}
        \vspace{-22pt}
        \caption{\mobius estimate for $\widehat{\mS}$}
        \label{fig:stocks_rootcauses}
    \end{subfigure}
    \caption{Real experiment on the S\&P 500 stock market index. (a) Instantaneous relations $\widehat{\mB}_0$ between the $45$ highest weighted stocks within S\&P 500, grouped by sectors (squares), and (b) the discovered structural shocks $\widehat{\mS}$ for $60$ days. In (a) the direction of influence is from row to column.}
    \vspace{-10pt}
    \label{fig:real_stocks}
\end{figure}

\paragraph{Learning the input} 
From the window graph approximation $\widehat{\mW}$, we can estimate the input $\widehat{\tS}$ using~\eqref{eq:root_causes_estimation}.  
Fig.~\ref{fig:stocks_rootcauses} presents this estimation for the same 45 stocks across 60 randomly chosen dates.  
As expected, significant input values (structural shocks) correspond to substantial changes in stock prices.  
To investigate this further, we evaluated all input values based on their alignment with stock price changes.  
We say that the input $s_{t,i}$ \textit{aligns} with the change in data if  $s_{t,i} \left( x_{t+1,i} - (1 + s_{t,i}/2)x_{t,i} \right) > 0$.
For example, if $s_{t,i} = 0.1$ aligns with the data change,  
then $x_{t+1,i}$ is at least $1.05$ times $x_{t,i}$.  
Considering the most significant $\approx 1\%$ of the $NTd = 512\text{,}500$ input entries of $\widehat{\tS}$ results in a threshold of $0.07$ and amounts to $4\text{,}656$ significant structural shocks, out of which $99.5\%$ align with stock value changes.  
Thus, whenever a structural shock occurs at day $t$,  
the stock price at day $t + 1$\footnote{The structural shock effect happens on the next day as the data we consider are the log returns of stock prices.}  
will increase if the value is positive (red) or decrease if the value is negative.  

\paragraph{News and dividends} 
We conjecture that structural shocks primarily capture significant unexpected events. For instance, META had a positive structural shock of $+0.18$ on February 1, 2024 (Fig.~\ref{fig:stocks_rootcauses}) and the same day it announced that it would pay dividends for the first time~\citep{paul2024facebook}. Similarly, NVDA experienced a $+0.20$ structural shock on May 24, 2023, coinciding with an upward sales forecast revision due to rising AI demand \citep{mehta2023nvidia}.
In contrast, significant, but expected, stock value changes like dividends deducted on the ex-dividend date are unlikely to generate structural shocks. Our dataset contains 3,796 dividend payments, yet only 36 coincided with a negative structural shock, supporting this conjecture.





\section{Limitations} 
\mobius inherits limitations of structure learning based on SVAR, which implies a linear and stationary model. The directed edges found are not necessarily true causal relations; establishing those would require further assumptions. We implicitly assume no undersampling: the measurement frequency is at least as high as the causal effects frequency. This may affect the stock market experiment where we used daily measurements, but stock market effects happen within split seconds. 
In addition, we assume that there are no missing values in the data and the measurements on each node are taken with the same frequency.
Also, while we can learn DAGs with up to thousands of nodes, very large graphs beyond that are still out of reach.
Finally, our work is designed specifically for sparse SVAR input. 
In App.~\ref{appendix:exp:simulated},\ref{appendix:exp:dream} we include experiments on a simulated financial dataset and the Dream3 gene expression dataset. While our method performs competitively it is not best, potentially because the sparse input assumption (or even linearity) is violated.




\section{Conclusion}\label{sec:conclusion}

We proposed \mobius, a novel method for estimating SVARs from time-series data under the assumption of sparse input. By modeling the input as independent Laplacian variables, \mobius is formulated as a maximum likelihood estimator based on least absolute error regression. Our method is supported by theoretical consistency guarantees and demonstrates superior performance over the state-of-the-art in experiments with synthetic and real-world financial datasets. The results highlight the utility of the sparse input assumption in uncovering interpretable structures and identifying significant events in real-world time-series data. This work opens avenues for future research in leveraging sparse input SVARs for causal discovery in time series.

\newpage

\onecolumn

\title{\mobius: Estimating Structural Vector Autoregression\\Assuming Sparse Input\\
(Supplementary material)}
\maketitle

\appendix

\section*{Ethics}
\mobius inherits the broader impact of other DAG learning methods from time series. From an ethical viewpoint, the methodology is generic and poses no specific potential risk.

\section*{Reproducibility} 
We acknowledge the importance of reproducibility and here we explain the actions that we took towards a more effortless reproduction of our results. 

\paragraph{Code} We provide our code written in Python $3.9$ as supplementary material and will make it available on github upon acceptance. 
In the README.md file, we explain the Python environment installation, how the code can be executed, and provide a Jupyter notebook demonstrating a synthetic experiment.
More importantly, our code not only provides an implementation of our method but rather the whole experimental pipeline, showing how the data are generated and how the baselines are applied.

\paragraph{Data} The sparse input SVAR data generation can be executed using our code or reproduced according to the parameters explained in the experimental section of the main text and the details in Appendix~\ref{appendix:exp:stability}. 
For the simulated financial and the S\&P 500 data we provide in Appendix~\ref{appendix:exp:data_resources} the sources to download them.

\paragraph{Methods} We have explained in great detail in the main text the optimization problem solved by \mobius and the adapted version of SparseRC that we use for fair comparison, also explained in Appendix~\ref{appendix:exp:sparserc}. 
For the execution of all baselines, we use publicly available repositories listed in~\ref{appendix:exp:code_resources} with hyperparameters set as shown in~\ref{appendix:exp:hyperparameter}.
Competitor methods can also be executed using the provided code.

\section{Mathematical Proofs and computations}
\label{sec:appendix:proofs}

In this section we provide all the proofs of technical results used in the manuscript.

\subsection{SVAR stability}
\label{appendix:subsec:stability}

Whenever a measurement can be taken in a system, stability in the measured data holds by definition. 
For example, temperature measurements or stock price markets are never unbounded. 
To ensure that the same happens for synthetic data, one needs to guarantee the stability of the data generation process. 
A few prior works mention stability~\citep{gong2015subsampledNGEM, khanna2019eSRU, bellot2021neuralgraphmodelling, malinsky2018SVAR-FCI}, and here we want to acknowledge its importance.

Equation~\eqref{eq:SVAR} can be viewed as a discrete-time multi-input multi-output (MIMO) system~\citep{skogestad2005multivariablecontrol}, in which the input is the structural shocks $\mS$ and the output is the time-series data $\mX$.
As the time-series length $T$ in~\eqref{eq:svar_lag_k} increases, the values of $\mX$ can get arbitrarily large. 
We desire to find a range of weights for the matrices $\mB_0,\mB_1,...,\mB_k$ that guarantees that our time-series data are bounded. 
In particular, we require a condition for the bounded-input bounded-output (BIBO) stability of this system. 
This has been already considered by~\citet{lutkepohl2005new} (linear case, for non-linear refer to~\citep{saikkonen2001stability}). 
The proposed condition requires the roots of the reverse characteristic polynomial to have a modulus less than $1$. 
Here, we prove a practical and intuitive condition for stability as a derivation of the~\citep{lutkepohl2005new} result. 

\paragraph{Transitive closure} To begin, we introduce the definition of the weighted transitive closure of the unrolled DAG~\eqref{eq:unrolledDAG}.
\begin{equation}
    \widetilde{\mX} =\widetilde{\mX}\mA + \widetilde{\mS} \Leftrightarrow \widetilde{\mX} = \widetilde{\mS}\left(\mI - \mA\right)^{-1} = \widetilde{\mS}\left(\mI + \transclos{\mA}\right),
    \label{eq:transitive_closure}
\end{equation}
On the right hand~\eqref{eq:transitive_closure} $\transclos{\mA} = \mA + ... + \mA^{dT - 1}$ is the weighted transitive closure~\citep{bastiJournalpaper} of the unrolled DAG $\mA$. 

\paragraph{Stability of model~\eqref{eq:SVAR}} 
We will now prove Theorem~\ref{appendix:th:stabilitySEM} that we are interested in. 
This provides a sufficient condition under which the model~\eqref{eq:SVAR} is BIBO stable. 
BIBO stability here means that if the input $\mS$ is bounded, then so are the output measurements $\mX$.

\begin{theorem} The model~\eqref{eq:SVAR} is BIBO stable if for some (sub-multiplicative) matrix norm $\|\cdot\|$:
\begin{equation*}
     \norm{\mW}< 1
\end{equation*}
\label{appendix:th:stabilitySEM}
\end{theorem}
\begin{proof}
    If $\norm{\mW} = λ< 1$ then from the structure of $\mA$ also $\norm{\mA} = \norm{\mW} = λ < 1$.
    Therefore:
    \begin{equation*}
        \norm{ \mI + \transclos{\mA}} = \norm{\mI + \mA + ... + \mA^{dT- 1}} \leq \sum_{t=0}^{dT-1}\norm{\mA}^t\leq \sum_{t=0}^{dT- 1} λ^t \leq \sum_{t=0}^{\infty} λ^t = \frac{1}{1-λ}=M
    \end{equation*}
    Thus 
    \begin{align*}
        \lim_{T\to \infty}\norm{\mX} &= \lim_{T\to \infty}\norm{\left(\mI + \transclos{\mA}\right)\mS}\\
        &\leq  \lim_{T\to\infty}\norm{\mI + \transclos{\mA}}\norm{\mS } \\
        &\leq M \norm{\mS }
    \end{align*}
    This implies that $\norm{\mX}$ is bounded for all $T$ and the model~\eqref{eq:SVAR} is BIBO stable. 
\end{proof}

\paragraph{Example} Consider the induced $L^{\infty}-$norm as $\normi{\mA} = \max_{j} \sum_{i=1}^d |a_{ij}|$.
The induced $L^{\infty}-$norm is sub-multiplicative and thus Theorem~\ref{appendix:th:stabilitySEM} can be utilized. 
In fact it can be proved that any induced vector norm is sub-multiplicative (Theorem 5.6.2 in~\citep{horn2012matrixanalysis}). 
Then, condition $\normi{\mW} < 1$ translates to all outcoming weights (rows of the window graph matrix) having the sum of absolute values less than $1$. 

For the sake of completeness, we provide a proof of the submultiplicativity property of the $L^{\infty}-$norm in Lemma~\ref{appendix:lemma:submultiplicative}. 

\begin{lemma}
    The induced $L^{\infty}-$norm is submultiplicative.
    \label{appendix:lemma:submultiplicative}
\end{lemma}
\begin{proof}
    Consider any two square matrices $\mA,\mB \in \R^{d\times d}$. We need to show that $\norm{\mA\mB}\leq \norm{\mA}\norm{\mB}$. Indeed, 
    \begin{align*}
        \norm{\mA\mB} &= \max_{i} \sum_{j=1}^d\left|\sum_{k=1}^da_{ik}b_{kj}\right| \\
        & \leq \max_{i} \sum_{j=1}^d\sum_{k=1}^d\left|a_{ik}b_{kj}\right| \\
        & = \max_{i} \sum_{k=1}^d\sum_{j=1}^d\left|a_{ik}\right|\left|b_{kj}\right| \\
    \end{align*}
    \begin{align*}
        & = \max_{i} \sum_{k=1}^d\left|a_{ik}\right|\left(\sum_{j=1}^d\left|b_{kj}\right| \right)\\
        & \leq \max_{i} \sum_{k=1}^d\left|a_{ik}\right|\left(\max_{k}\sum_{j=1}^d\left|b_{kj}\right| \right)\\
        & = \max_{i} \sum_{k=1}^d\left|a_{ik}\right|\norm{\mB}\\
        & \leq \norm{\mA}\norm{\mB}\\
    \end{align*}
\end{proof}

\paragraph{Example} The $L^{\infty}-$norm is particularly interesting for our scenario as the condition of Theorem~\ref{appendix:th:stabilitySEM} provides an intuitive interpretation for the weights.
Consider our stock market example. 
Then the condition in~\ref{appendix:th:stabilitySEM} means that for every stock that affects a set of other stocks, each with some factor $<1$, the total sum should be less than $1$. 
Of course, this is only a sufficient condition for the data to be bounded, but we believe that it is meaningful to consider that the influences between stocks are of this form in reality.
To understand better why the condition in~\ref{appendix:th:stabilitySEM} provides bounded data, we can think about it in the following way.
When the $L^{\infty}-$norm is bounded, the total effect of a stock is divided into individual fractions that affect other stocks and doesn't get iteratively increased (which could be the case with sum $L^{\infty}-$norm $>1$). Bounding the sum of outcoming weights to $1$ has also been considered in~\citep{bastiJournalpaper, misiakos2024fewrootcauses} in the scenario of pollution propagation in a river network. 

We further include another submultiplicative property, that we later use on our proofs.

\begin{lemma}
    The $L^1-$norm, defined as sum of absolut values of the entries of a matrix is submultiplicative.
    \label{appendix:lemma:submultiplicativeL1}
\end{lemma}
\begin{proof}
    Consider any two square matrices $\mA,\mB \in \R^{d\times d}$. We want to show that $\normii{\mA\mB}\leq \normii{\mA}\normii{\mB}$. Indeed, 
    \begin{align*}
        \normii{\mA\mB} &= \sum_{i=1}^d\sum_{j=1}^d\left|\sum_{k=1}^da_{ik}b_{kj}\right| \\
        & \leq \sum_{i=1}^d \sum_{j=1}^d\sum_{k=1}^d\left|a_{ik}b_{kj}\right| \\
        & = \sum_{i=1}^d \sum_{k=1}^d\sum_{j=1}^d\left|a_{ik}\right|\left|b_{kj}\right| \\
        & = \sum_{i=1}^d \sum_{k=1}^d\sum_{j=1}^d\sum_{l=1}^d\left|a_{ik}\right|\mathbbm{1}_{k=l}\left|b_{lj}\right| \\
        &  \leq\sum_{i=1}^d \sum_{k=1}^d\sum_{j=1}^d\sum_{l=1}^d\left|a_{ik}\right|\left|b_{lj}\right| \\
        & =\left(\sum_{i=1}^d \sum_{k=1}^d\left|a_{ik}\right|\right)\left(\sum_{j=1}^d\sum_{l=1}^d\left|b_{lj}\right|\right) \\
        & \leq \normii{\mA}\normii{\mB}\\
    \end{align*}
\end{proof}

\subsection{Identifiability}
\label{appendix:subsec:identifiability}

\begin{theorem}
    Consider the time-series model~\eqref{eq:SVAR} with $\mS$ following a multivariate Laplace distribution as in~\eqref{eq:laplace_model} with $\beta\in[a,b]$ and $a> \frac{1}{NTd}$. Then the matrices $\mB_{0},\mB_{1},...,\mB_{k}\in\R^{d\times d}$ and $\beta$ are identifiable from the time-series data $\tX$. 
\end{theorem}

\begin{proof}
    As we explain later in Appendix~\ref{appendix:exp:sparserc} we can rewrite the SVAR~\eqref{eq:SVAR} as a linear SEM:
    \begin{equation}
        \tX = \tX_{\text{past}}\mW + \tS 
        \Leftrightarrow \widetilde{\mX} =\widetilde{\mX}\mA + \widetilde{\mS}.
        \label{eq:unrolled_SVAR}
    \end{equation}
    where $\widetilde{\mX},\widetilde{\mS}\in\sR^{N\times dT}$ and $\mA\in \sR^{dT\times dT}$.
    The structural shocks follow a Laplacian distribution which implies that $\widetilde{\mS}$ is non-Gaussian. 
    Moreover, based on the acyclicity assumption on $\mB_0$, the unrolled matrix $\mA$ represents a DAG and therefore~\eqref{eq:unrolled_SVAR} describes an SEM with non-Gaussian noise variables as in~\citep{shimizu2006lingam}. 
    The identifiability of $\mA$ then follows from LiNGAM~\citep{shimizu2006lingam}.
    Moreover, identifiability on $\mA$ implies identifiability for the parameters $\mB_0, \mB_1, ...,\mB_k$ of the window graph $\mW$, as desired.

    We will now establish identifiability of $\beta$ using the monotonicity of the Laplacian probability distribution. Notice, identifiability on $\mW$ means, that for any $\mW\in\ml{W}$ and any $\beta\in[a,b]$, the equation $f_X(\tX|\mW,\beta) = f_X(\tX|\mW^*,\beta^*) $ gives $\mW = \mW^*$. This in turn implies that the parameter $\beta$ is identifiable. Indeed:
    \begin{equation}
        f_X(\tX|\mW^*,\beta) = \left|\text{det}\left(\mI - \mB_0\right)\right|^{NT} \frac{1}{(2\beta)^{NTd}}e^{-\frac{\normii{\tX - \tX_{\text{past}}\mW^*}}{\beta}}
    \end{equation}
    The derivative with respect to $\beta$ is:
    \begin{equation}
        \frac{\partial f_X}{\partial \beta} = \left|\text{det}\left(\mI - \mB_0\right)\right|^{NT} \left(\frac{1}{\beta} - NTd\right)\frac{\normii{\tX - \tX_{\text{past}}\mW^*}}{2^{NTd}\beta^{NTd + 1}}e^{-\frac{\normii{\tX - \tX_{\text{past}}\mW^*}}{\beta}} < 0
    \end{equation}
    Therefore, $f_X$ is monotonically decreasing and thus bijective for $\beta > \frac{1}{NTd}$. 
    Therefore: 
    \begin{equation}
        f_X(\tX|\mW,\beta) = f_X(\tX|\mW^*,\beta^*) \xRightarrow[]{\text{LiNGAM}} \mW = \mW^* \text{ and } f_X(\tX|\mW^*,\beta) = f_X(\tX|\mW^*,\beta^*) \xRightarrow[]{\text{monotonicity}} \beta = \beta^*
    \end{equation}
    Thus, $\left(\mW,\beta\right)$ are identifiable from the data $\tX$. 
    
\end{proof}

\subsection{MLE computation}
\label{appendix:subsec:MLE_computation}
\paragraph{Estimator computation}
Here we compute the MLE assuming that each entry of the structural shocks $\mS\in \sR^{d\times T}$ follows independently a Laplace distribution $\text{Laplace}\left(0,\beta\right)$. The multivariate probability density function of $\mS$ is:
\begin{equation}
    f_C(\mS) = \prod_{\tau, j} \frac{1}{2\beta}e^{-\frac{\left|\mS_{\tau,j}\right|}{\beta}}
\end{equation}
Solving with respect to $\mX$ equation~\eqref{eq:SVAR} gives $\mX = \mS(\mI - \mA)^{-1}$ where $\mA\in \sR^{dT\times dT}$ is the unrolled DAG matrix of $\mW$ according to~\eqref{eq:unrolledDAG}. Here, we didn't change our notation, but $\mX$ and $\mS$ are supposed to represent $1\times dT$ dimensional vectors. For simplicity we will do this interchange in the following computations as it doesn't affect the probability distribution. Using this linear transformation the probability density function (pdf) of $\mX$, or likelihood of the data, becomes 
\begin{align*}
    f_X(\mX|\mW,\beta) &= \frac{f_C\left(\mX\left(\mI - \mA\right)\right)}{\left|\text{det}\left(\left(\mI - \mA\right)^{-1}\right)\right|} \\
    &=\left|\text{det}\left(\mI - \mA\right)\right|\prod_{\tau, j} \frac{1}{2\beta}e^{-\frac{\left|\mX_{\tau,j} - \mX_{\text{past}
    \tau,:}\mW_{:,j}\right|}{\beta}} \\
    &= \left|\text{det}\left(\mI - \mB_0\right)^T\right|\prod_{\tau, j} \frac{1}{2\beta}e^{-\frac{\left|\mX_{\tau,j} - \mX_{\text{past}
    \tau,:}\mW_{:,j}\right|}{\beta}}\\
    &= \left|\text{det}\left(\mI - \mB_0\right)\right|^T \frac{1}{(2\beta)^{dT}}e^{-\frac{\normii{\mX - \mX_{\text{past}}\mW}}{\beta}}
\end{align*}
Therefore, for $N$ realizations of $\mX$ in the tensor $\tX$ we have that:
\begin{align}
    f_X(\tX|\mW,\beta) 
    &= \left|\text{det}\left(\mI - \mB_0\right)\right|^{NT} \frac{1}{(2\beta)^{NdT}}e^{-\frac{\normii{\tX - \tX_{\text{past}}\tW}}{\beta}},
\end{align}
which in turn gives the log-likelihood for the data:
\begin{align}
    \mathcal{L}\left(\tX|\mW,\beta\right) &= \log f_X\left(\tX|\mW,\beta\right)  \notag\\
    &=NT\log \left|\text{det}\left(\mI - \mB_0\right)\right|  - NTd \log(2\beta)   -\frac{1}{\beta} \normii{\tX - \tX_{\text{past}}\mW}.
    \label{appendix:eq:loglikelihoodMLE}
\end{align}

In what follows, for simplicity of notation we will skip the parameter $\beta$ and will use $\mathcal{L}\left(\tX|\mW,\beta\right)$ and $\mathcal{L}\left(\tX|\mW\right)$
 interchangeably.

\subsection{MLE consistency background}
\label{appendix:subsec:MLE_consistency_background}
We proceed by analyzing the prior theorems that we will use to prove our results. 
First, denote with $\logpop{\mW}$ the population log-likelihood~\citep{lachapelle2019granDAG,newey1994MLEconsistency}, defined as:
\begin{equation}
    \logpop{\mW, \beta} = \expvp{\mW^*,\beta^*}{\loglike{\tX}{\mW,\beta}}.
\end{equation}
Note that we use $\mathcal{L}\left(\tX|\mW,\beta\right)$ and $\mathcal{L}\left(\tX|\mW\right)$
 interchangeably, as well as $\logpop{\mW,\beta}$ and $\logpop{\mW}$. 
In essence, the population log-likelihood is the expected value of the log-likelihood function computed with the probability density $f_X(\tX|\mW^*, \beta^*)$ with parameters the ground truth window graph $\mW^*$ and parameter $\beta^*$. 

\begin{lemma}
    Assume that the ground truth window graph $\mW^*$ and parameter $\beta^*$ are identifiable from the data distribution. This means, that for $\left(\mW,\beta\right)\neq\left(\mW^*,\beta^*\right)$ it is true that $f_X\left(\tX|\mW, \beta\right)\neq f_X\left(\tX|\mW^*,\beta^*\right)$. Then, the population likelihood $\logpop{\mW,\beta}$ has unique maximum at the true window graph $\mW^*$ and true $\beta^*$.
    \label{appendix:lemma:uniqueMLEmaximizer}
\end{lemma}
\begin{proof}
    We show that $\logpop{\mW^*,\beta^*} > \logpop{\mW,\beta}$ for every $\left(\mW,\beta\right)\neq\left(\mW^*,\beta^*\right)$. By simplifying our notation we have:
    \begin{align*}
        \logpop{\mW^*}  - \logpop{\mW} & = \expvp{\mW^*}{\loglike{\tX}{\mW^*} -\loglike{\tX}{\mW}}\\
        & = \expvp{\mW^*}{-\log\frac{f_X\left(\tX|\mW\right)}{f_X\left(\tX|\mW^*\right)}}\\
        & >-\log \expvp{\mW^*}{\frac{f_X\left(\tX|\mW\right)}{f_X\left(\tX|\mW^*\right)}}\\
        & =-\log \int_{\tX\in \sR^{N\times T\times d}}\frac{f_X\left(\tX|\mW\right)}{f_X\left(\tX|\mW^*\right)}f_X\left(\tX|\mW^*\right)d\tX\\
        & =-\log \int_{\tX\in \sR^{N\times T\times d}}f_X\left(\tX|\mW\right)d\tX\\
        & =-\log 1 = 0
    \end{align*}
    On the second line we used that $\frac{f_X\left(\tX|\mW\right)}{f_X\left(\tX|\mW^*\right)}$ is non-constant, so we can apply the strict Jensen inequality~\citep{newey1994MLEconsistency} $\expv{a(\mY)} >\expv{a(\mY)}$ for a convex function $a$ and non-constant random variable $\mY$.
\end{proof}

As a next result for our toolset to prove the MLE consistency, we include the uniform law of large numbers as stated by~\citet{newey1994MLEconsistency}. 

\begin{lemma}[Uniform Law of Large Numbers]
    Consider that the log-likelihood function $\loglike{\tX}{\mW},\,\mW\in\ml{W}$ satisfy the following conditions.
    \begin{itemize}
        \item The data $\tX_i$ are independent and identically distributed.
        \item $\ml{W}$ is a compact space.
        \item $\loglike{\tX_i}{\mW},\,\mW\in\ml{W}$ is continuous at each $\mW\in \ml{W}$ with probability $1$. 
        \item There exists dominating function $D(\mW)$ such that $\left|\loglike{\tX}{\mW}\right|\leq D\left(\mW\right)$ and $\expvp{\mW^*}{D(\mW)}<\infty$.
    \end{itemize}
    Then the population likelihood $\logpop{\mW}$ and the empirical average log-likelihood converges uniformly in probability to it:
    \begin{equation}
        \sup_{\mW\in\ml{W}}\left|\frac{1}{n}\sum_{i=1}^{n}\loglike{\tX_i}{\mW} - \logpop{\mW}\right| \xrightarrow[]{p}0 
    \end{equation}
    \label{lemma:LawLargeNum}
\end{lemma}

We now present Theorem~\ref{appendix:th:consistency}, which establishes the consistency of the maximum likelihood estimator (MLE). This theorem is based on a set of sufficient assumptions for ensuring MLE consistency. For completeness, we include a detailed proof of Theorem~\ref{appendix:th:consistency}, leveraging the uniform law of large numbers.

\begin{theorem}
    Consider that the average log-likelihood function $\logpopn{\mW}$ and population $\logpop{\mW}$ satisfy the following conditions for $\mW\in\ml{W}$:
    \begin{itemize}
        \item $\mW^* = \argmax_{\mW \in \ml{W}} \logpop{\mW}$ is identifiable from the data.
        \item $\ml{W}$ is a compact space.
        \item $\loglike{\tX_i}{\mW}$ is continuous at each $\mW\in\ml{W}$ with probability $1$.
        \item $\expvp{\mW^*}{\sup_{\mW\in\ml{W}}\left|\loglike{\tX}{\mW}\right|}<\infty$ .
    \end{itemize}
    Then, if the maximum of $\logpopn{\mW}=\frac{1}{n}\sum_{i=1}^{n}\loglike{\tX_i}{\mW}$ is achieved at $\widehat{\mW}_n$ then $\widehat{\mW}_n$ converges uniformly to $\mW^*$.
    \label{appendix:th:consistency}
\end{theorem}

\begin{proof}
We repeat the proof of \citet{newey1994MLEconsistency} for our scenario. 
From the identifiability assumption, Lemma~\ref{appendix:lemma:uniqueMLEmaximizer} implies that $\mW^*$ is the unique and global maximizer of $\logpop{\mW}$.
Also, if we set $D(\mW) = \sup_{\mW\in\ml{W}}\left|\loglike{\tX}{\mW}\right|$, then the conditions of Lemma~\ref{lemma:LawLargeNum} are satisfied and therefore $\logpopn{\mW}$ converges uniformly in probability to $\logpop{\mW}$.
We will leverage the compactness of the space $\ml{W}$ to show that their maxima satisfy 
\begin{equation}
    \widehat{\mW}_n\xrightarrow[]{p}\mW^*
\end{equation}

From the uniform convergence it follows that with probability approaching  $1$ for any $\epsilon$ (or $\epsilon/3$ as we use next):
\begin{equation}
    \left|\logpopn{\mW} - \logpop{\mW}\right| < \epsilon \Leftrightarrow \logpop{\mW} - \epsilon <\logpopn{\mW} < \logpop{\mW} + \epsilon ,\,\forall \mW \in \ml{W}.
    \label{eq:uniform}
\end{equation}
Since by definition $\logpopn{\mW}$ is a continuous function and $\ml{W}$ is compact it takes a maximum value at point $\widehat{\mW}_n$. Since $\logpopn{\widehat{\mW}_n} \geq \logpopn{\mW^*}$ the maximum would satisfy for any $\epsilon > 0$
\begin{equation}
    \logpopn{\widehat{\mW}_n} > \logpopn{\mW^*} - \epsilon/3.
\end{equation}
This in combination with~\eqref{eq:uniform} would imply
\begin{equation}
    \logpop{\widehat{\mW}_n} > \logpopn{\widehat{\mW}_n} - \epsilon/3 > \logpopn{\mW^*} - 2\epsilon/3 > \logpop{\mW^*} - \epsilon.
\end{equation}
In essence we have proved that $\logpop{\widehat{\mW}_n}$ can get arbitrarily close to $\logpop{\mW^*}$. This in turn gives that $\widehat{\mW}_n$ approaches $\mW^*$ with probability $1$ as $n\to \infty$. Indeed, if we consider any open interval $\ml{I}$ containing $\mW^*$, then $\ml{W}\cap \ml{I}^c$ is compact and we can compute
\begin{equation}
    M = \sup_{\mW\in \ml{W}\cap \ml{I}^c} \logpop{\mW} < \logpop{\mW^*}
\end{equation}

Note that by Lemma~\ref{lemma:LawLargeNum} $\logpop{\mW}$ is continuous, so the supremum is a finite value.
If we choose $\epsilon = \logpop{\mW^*} - M$ then: 
\begin{equation}
    \logpop{\widehat{\mW}_n} > \logpop{\mW^*} - \epsilon = M
\end{equation}

Thus $\widehat{\mW}_n\in \ml{I}$ which concludes the proof.
\end{proof}

\subsection{MLE consistency for DAGs}
\label{appendix:subsec:MLE_consistency_proof}

We will now show that the MLE computed at~\eqref{appendix:eq:loglikelihoodMLE} satisfies the requirements of Theorem~\ref{appendix:th:consistency} for consistency. Practically, this result implies that as the amount of available data $\tX$ increases, the maximizer $\widehat{\mW}$ of the log-likelihood function $\loglike{\tX}{\mW,\beta}$ converges to the maximizer $\mW^*$ of the population likelihood $\logpop{\mW,\beta}$. To begin with we introduce the following useful lemma. In essence, using the continuous characterization of acyclicity~\citep{zheng2018notears} we show that the space of bounded DAGs is also closed and thus compact. 

\begin{lemma}
    The set of acyclic matrices $\ml{A} =\left\{\mA\in[-1,1]^{d\times d}|\mA\text{ is acyclic}\right\}$ is compact.
    \label{lemma:compact}
\end{lemma}
\begin{proof}
    Note that \citet{zheng2018notears} proved that 
    \begin{equation}
        \mA\text{ is acyclic}\Leftrightarrow h\left(\mA\right) = 0,
    \end{equation}
    where $h\left(\mA\right) = e^{\mA\odot\mA}-d$ is a continuous function. We proceed by showing that $\ml{A}$ is closed and bounded.
    \begin{itemize}
        \item \textbf{Closed:} $[-1,1]^{d\times d}$ is closed and since $h\left(\mA\right)$ is continuous and $\left\{\mA\text{ is acyclic}\right\} = h^{-1}(\{0\})$ implies that $\ml{A}$ is closed~\citep{sutherland2009topologicalspaces}.
        \item \textbf{Bounded:} $\ml{A}$ is bounded because $\ml{A}\subset[-1,1]^{d\times d}$ which is bounded.
    \end{itemize}
    Therefore, since $\ml{A}\subset \sR^{d\times d}$ is closed and bounded, $\ml{A}$ is compact~\citep{sutherland2009topologicalspaces}.
\end{proof}

Now using this Lemma we are ready to prove our consistency result. 

\begin{theorem}
    The maximum log-likelihood estimator of~\eqref{appendix:eq:loglikelihoodMLE}  satisfies the conditions of Theorem~\ref{appendix:th:consistency} and thus is consistent under the following assumptions:
    \begin{itemize}
        \item The space of window graphs  $\ml{W}\subseteq [-1,1]^{d(k+1)\times d}$ is bounded and $\mB_0$ is acyclic.
        \item The Laplacian parameter $\beta\in [a,b]$ is bounded with lower bound $a > \frac{1}{dT}$\footnote{This value in our experiment is at most $\frac{1}{2\cdot 10^4}$, so this is a mild assumption.}.
    \end{itemize}
    \label{appendix:th:MLE_consistency}
\end{theorem}
\begin{proof}
    We check one-by-one the requirements of Theorem~\ref{appendix:th:consistency}.
    
    First, the identifiability of the ground truth $\mW^*$ and $\beta^*$ follows from Theorem~\ref{appendix:subsec:identifiability}.
    
    Also, $(\mW,\beta)\in\ml{W}\times [a,b] = \ml{A}\times [-1,1]^{dk \times d}\times [a,b]$ which is compact because  the space $\ml{A}$ of acyclic graphs $\mB_0$ is compact from Lemma~\ref{lemma:compact} and $[-1,1]^{dk \times d}$ and $[a,b]$ are both closed and bounded and thus compact according to~\citet{sutherland2009topologicalspaces}. 
    
    Moreover, the log-likelihood 
    \begin{equation}
        \mathcal{L}\left(\tX|\mW,\beta\right) 
    =NT\log \left|\text{det}\left(\mI - \mB_0\right)\right|  - NTd \log(2\beta)   -\frac{1}{\beta} \normii{\tX - \tX_{\text{past}}\mW} 
    \end{equation}
    is continuous at $\left(\mW,\beta\right)$.

    Finally, we need to show that $\expv{\sup_{\mW\in \ml{W}}\left|\loglike{\tX}{\mW}\right|}<\infty$. For this we compute:

    \begin{align*}
        \left|\loglike{\tX}{\mW}\right| &= \left|\log f_X\left(\tX|\mW,\beta\right)\right|  \\
         &=\left|NT\log \left|\text{det}\left(\mI - \mB_0\right)\right|  - NTd \log(2\beta)   -\frac{1}{\beta} \normii{\tX - \tX_{\text{past}}\mW}\right| \\
         & = \left|- NTd \log(2\beta) -\frac{1}{\beta} \normii{\tX - \tX_{\text{past}}\mW}\right|\\
         &\leq \left|NTd \log(2b)\right| + \left| \frac{1}{\beta}\normii{\tX - \tX_{\text{past}}\mW}\right|\\
         & \leq C_1 + \frac{1}{a} \normii{\tX - \tX_{\text{past}}\mW}\\
         & \leq C_1 + C_2 \normii{\tX}
    \end{align*}
    Here we used that $\mB_0$ is acyclic and thus $NT\log \left|\text{det}\left(\mI - \mB_0\right)\right| =0$.
    We assumed that $\beta\in[a,b]$ is bounded. Also we used that the $(τ,j)$ entry of $\mX - \mX_{\text{past}}\mW$ is $\mX_{\tau,j} - \mX_{\text{past}
    \tau,:}\mW_{:,j}$ and 
    \begin{align*}
        \left|\mX_{\tau,j} - \mX_{\text{past}\tau,:}\mW_{:,j}\right| &< \left|\mX_{\tau,j}\right| + \normii{\mX_{\text{past}\tau,:}}\Rightarrow\\
         \sum_{\tau,j}\left|\mX_{\tau,j} - \mX_{\text{past}\tau,:}\mW_{:,j}\right|&< \sum_{\tau,j} ((k+1)d + 1)\left|\mX_{\tau,j}\right| = ((k+1)d + 1)\normii{\mX},
    \end{align*} 
    which furthermore implies 
    \begin{equation}
         \normii{\tX - \tX_{\text{past}}\mW} = \sum_{i}\left|\tX_{i} - \tX_{i,\text{past}}\mW\right|< \sum_{i}((k+1)d + 1)\normii{\tX_i} = ((k+1)d + 1)\normii{\tX} = C_2\normii{\tX},
    \end{equation}
    for some constant $C_2$. Therefore:
    \begin{align*}
        \expvp{\mW^*}{\left|\loglike{\tX}{\mW}\right|} & = \int_{\tX\in\sR^{T\times d}} \left|\loglike{\tX}{\mW}\right|f_X\left(\tX|\mW^*,\beta^*\right)d \tX\\
        & < \int_{\tX\in\sR^{N\times T\times d}} \left(C_1 + C_2 \normii{\tX}\right)f_X\left(\tX|\mW^*,\beta^*\right)d \tX\\
        & = \int_{\tX\in\sR^{N\times T\times d}} 
        \left(C_1 + C_2 \normii{\tX}\right)
        \left|\text{det}\left(\mI - \mB_0^*\right) \right|^{NT} \frac{1}{(2\beta^*)^{NdT}} e^{-\frac{\normii{\tX - \tX_{\text{past}}\mW^*}}{\beta^*}}d\tX\\
        & = \int_{\tX\in\sR^{N\times T\times d}} 
        \left(C_1 + C_2 \normii{\tX}\right)
         \frac{1}{(2\beta^*)^{NdT}} e^{-\frac{\normii{\tX - \tX_{\text{past}}\mW^*}}{\beta^*}}\left|\text{det}\left(\mI - \mA^*\right) \right|d\tX\\
         & = \int_{\tS\in\sR^{N\times T\times d}} 
        \left(C_1 + C_2 \normii{\tX}\right)
         \frac{1}{(2\beta^*)^{NdT}} e^{-\frac{\normii{\tS}}{\beta^*}}d\tS\\
         & = C_1 + C_2\int_{\tS\in\sR^{N\times T\times d}} 
         \normii{\tX}
         \frac{1}{(2\beta^*)^{NdT}} e^{-\frac{\normii{\tS}}{\beta^*}}d\tS\\
         & = C_1 + C_2\int_{\tS\in\sR^{N\times T\times d}} 
        \normii{\tS\left(\mI - \mA^*\right)^{-1}}
         \frac{1}{(2\beta^*)^{NdT}} e^{-\frac{\normii{\tS}}{\beta^*}}d\tS\\
    \end{align*}
    Note that, $\normii{\tS\left(\mI - \mA^*\right)^{-1}} = \normii{\tS\left(\mI + \mA^* + ...+(\mA^*)^{dT}\right)}$. From Lemma~\ref{appendix:lemma:submultiplicativeL1} we have that 
    \begin{equation}
        \normii{\tS\left(\mI - \mA^*\right)^{-1}} = \normii{\tS\left(\mI + \mA^* + ...+(\mA^*)^{dT}\right)} \leq \normii{\tS}\left(dT + \normii{\mA^*} + ...+\normii{(\mA^*)}^{dT}\right) \leq \normii{\tS}\cdot C_3    
    \end{equation}

    Thus:
    \begin{align*}
        \expvp{\mW^*}{\left|\loglike{\tX}{\mW}\right|} & < C_1 + C_2\int_{\tS\in\sR^{N\times T\times d}} 
        \normii{\tS\left(\mI - \mA^*\right)^{-1}}
         \frac{1}{(2\beta^*)^{NdT}} e^{-\frac{\normii{\tS}}{\beta^*}}d\tS\\
         & < C_1 + C_2C_3\int_{\tS\in\sR^{N\times T\times d}} 
        \normii{\tS}
         \frac{1}{(2\beta^*)^{NdT}} e^{-\frac{\normii{\tS}}{\beta^*}}d\tS\\
         & = C_1 + C_2C_3\sum_{i,\tau,j}\int_{\tS\in\sR^{N\times T\times d}} 
        |\tS_{i,\tau,j}|
         \frac{1}{(2\beta^*)^{NdT}} e^{-\frac{\normii{\tS}}{\beta^*}}d\tS\\
         & = C_1 + C_2C_3\sum_{i,\tau,j}\int_{\sR} 
        |\tS_{i,\tau,j}|
         \frac{1}{(2\beta^*)} e^{-\frac{|\tS_{i,\tau,j}|}{\beta^*}}d\tS_{i,\tau,j}\\
         & = C_1 + 2C_2C_3\sum_{i,\tau,j}\int_{\sR_{\geq0}} 
        |\tS_{i,\tau,j}|
         \frac{1}{(2\beta^*)} e^{-\frac{|\tS_{i,\tau,j}|}{\beta^*}}d\tS_{i,\tau,j}\\
         & = C_1 + 2C_2C_3\sum_{i,\tau,j}\int_{\sR_{\geq0}} 
        \tS_{i,\tau,j}
         \frac{1}{(2\beta^*)} e^{-\frac{\tS_{i,\tau,j}}{\beta^*}}d\tS_{i,\tau,j}\\
         & = C_1 + 2C_2C_3\sum_{i,\tau,j}\left\{-\frac{\tS_{i,\tau,j}}{2} e^{-\frac{\tS_{i,\tau,j}}{\beta^*}}\Big|_{0}^{\infty}-\int_{\sR_{\geq0}} 
         -\frac{1}{2} e^{-\frac{\tS_{i,\tau,j}}{\beta^*}}d\tS_{i,\tau,j}\right\}\\
         & = const <\infty.
    \end{align*}
\end{proof}

\begin{remark}
    Note that LiNGAM identifiability is true in the entire space of real matrices $\sR^{d(k+1) \times d}$~\citep{shimizu2006lingam, ng2020GOLEM}. In other words for any $\mW\in \sR^{d(k+1) \times d}$ different from the ground truth DAG $\mW^*$ the distribution $f_X(\mX|\mW,\beta) $ induced by $\mW$ is different from that of $\mW^*$, namely $f_X(\mX|\mW^*,\beta)$. The reason we restrict our search space to be a DAG is to constrain the magnitude of the terms of the MLE that contain $\mB_0$. 
\end{remark}

\subsection{\mobius optimization derivation}
\label{appendix:subsec:optimization_derivation}

Here we derive the optimization objective of \mobius for approximating the ground truth window graph parameters of the SVAR of~\eqref{eq:SVAR}, given that the SVAR input $\tS$ entries are distributed independently according to $\text{Laplace}(0,\beta^*)$. We consider $N$ realization of time series $\mX$ collected in a tensor $\tX\in\sR^{N\times T\times d}$. According to~\eqref{appendix:eq:loglikelihoodMLE} the log-likelihood of the data $\tX$ is 
\begin{align}
    \mathcal{L}\left(\tX|\mW,\beta\right) &= \log f_X\left(\tX|\mW,\beta\right) = \log\prod f_X\left(\mX_i|\mW,\beta\right) \\
    &= \sum_{i=1}^N\log f_X\left(\mX_i|\mW,\beta\right)\\
    &=NT\log \left|\text{det}\left(\mI - \mB_0\right)\right|  - NTd \log(2\beta)   -\frac{1}{\beta}\normii{\tX - \tX_{\text{past}}\mW} 
\end{align}
To maximize the log-likelihood with respect to $\beta$ we solve:
\begin{equation}
    \frac{\partial \mathcal{L}}{\partial \beta} = 0\Leftrightarrow  - \frac{NTd}{\beta}  +\frac{1}{\beta^2}\normii{\tX - \tX_{\text{past}}\mW} = 0 \Leftrightarrow \beta = \frac{1}{NTd}\normii{\tX - \tX_{\text{past}}\mW}.
\end{equation}
Note that if $\mW = \mW^*$ this is a reasonable value for $\beta$ as on expectation
\begin{equation}
    \expvp{\beta^*}{\normii{\tX - \tX_{\text{past}}\mW}} = \expvp{\beta^*}{\normii{\tS}} = \sum_{i,\tau,j}\expvp{\beta^*}{\normii{\tS_{i,\tau,j}}} = NTd\beta^*.
\end{equation}
Moreover, 
\begin{equation}
    \frac{\partial^2 \mathcal{L}}{\partial \beta^2} = 0\Leftrightarrow  \frac{NTd}{\beta^2}  -\frac{2}{\beta^3}\normii{\tX - \tX_{\text{past}}\mW} = \frac{NTd}{\beta^3}\left(\beta - \frac{2}{NTd} \normii{\tX - \tX_{\text{past}}\mW}\right) < 0.
\end{equation}
So, $\mathcal{L}\left(\tX|\mW,\beta\right)$ is locally concave at $\beta = \frac{1}{NTd}\normii{\tX - \tX_{\text{past}}\mW}$, which gives a local maximum. Similarly to~\citet{ng2020GOLEM}, we profile out the parameter $\beta$ using its approximation $\widehat{\beta} = \frac{1}{NTd}\normii{\tX - \tX_{\text{past}}\mW}$ to formulate a log-likelihood maximization problem for approximating $\mW$:
\begin{align*}
    \mathcal{L}\left(\tX|\mW,\widehat{\beta}\right) = NT\log \left|\text{det}\left(\mI - \mB_0\right)\right|  - NTd \log\left(\normii{\tX - \tX_{\text{past}}\mW}\right)  + \text{const}.
\end{align*}
The window graph $\mW^*$ is then be approximated as:
\begin{align}
 \widehat{\mW} = \argmax_{\mW\in\ml{W}} \mathcal{L}\left(\mX|\mW\right) &= 
    \argmax_{\mW\in\ml{W}} \left\{ NT\log \left|\text{det}\left(\mI - \mB_0\right)\right|  - NTd \log\left(\normii{\tX - \tX_{\text{past}}\mW}\right)  + \text{const}\right\} \notag \\
    & = 
    \argmin_{\mW\in\ml{W}} \left\{ d\log\left(\normii{\tX - \tX_{\text{past}}\mW}\right) -\log \left|\text{det}\left(\mI - \mB_0\right)\right| \right\}\notag \\
    &= 
    \argmin_{\mW\in\ml{W}} \log\normii{\tX - \tX_{\text{past}}\mW} -\frac{1}{d}\log \left|\text{det}\left(\mI - \mB_0\right)\right|
    \label{appendix:eq:MLE_minimization}
\end{align}

In practice, searching for the minimum of~\eqref{appendix:eq:MLE_minimization} over the space of DAGs is computationally inefficient. Following \citet{ng2020GOLEM} we use the acyclicity as a soft constraint, i.e. a regularizer. This simplifies the optimization algorithm without compromising performance, as will be shown in our experiments. The final optimization of \mobius is the following.

\begin{equation}
    \widetilde{\mW} = \argmin_{\mW\in\sR^{d(k+1) \times d}}  \log\normii{\tX - \tX_{\text{past}}\mW} -\frac{1}{d}\log \left|\text{det}\left(\mI - \mB_0\right)\right| + \lambda_1\cdot \|\mW\|_1  + \lambda_2\cdot h\left(\mB_0\right).
\label{app:eq:cont_opt}
\end{equation}

\section{Applying SparseRC to time-series data}
\label{appendix:exp:sparserc}

SparseRC~\citep{misiakos2024fewrootcauses} is designed to learn a DAG from static data. 
\citet{misiakos2024icassp} applied SparseRC to learn graphs from time-series data by exploiting the structure of the unrolled DAG corresponding to the time series. 
For long time series, such a formulation creates a huge DAG to be learned - ranging from 20 thousand to 1 million nodes in our experiments. 
However, SparseRC can only be executed for $\approx 5000$ nodes at maximum to terminate in a reasonable time~\citep{misiakos2024fewrootcauses}. 
Thus it is impossible to be applied in our scenario in its prior form. 
For this reason, we propose an alternative way to apply SparseRC, which however, comes with a cost in approximation performance.

\paragraph{SVAR as a Linear SEM} To start with we show how an SVAR can be written as a linear structural equation model (SEM), which is the analogous model for generating linear static DAG data.
We consider a time series $\mX$ generated with the SVAR in~\eqref{eq:SVAR} (noiseless for simplicity). 
Consider the single-row vector $\vx = \left(\vx_0,\vx_1,...,\vx_{T-1}\right)\in\R^{1\times dT}$ consisting of the concatenation of the time-series vectors $\vx_0,\vx_1,...,\vx_{T-1}$ along the first dimension. 
Then~\eqref{eq:SVAR} can also be encoded as:
\begin{equation}
    \vx = \vx \mA + \vs,
\end{equation}
where the structural shocks $\vs$ here also have dimension $1\times dT$. 
The matrix $\mA$ is the adjacency matrix of a DAG with a special structure called the \textit{unrolled DAG}~\citep{kim2012temporal}, which occurs by repeating the window graph corresponding to~\eqref{eq:svar_lag_k} for every time step $t\in[T]$:
\begin{equation}
\mA  = \begin{pmatrix}
\mB_0       & \mB_1   & \hdots    & \mB_k   & \hdots  &   \bm{0}\\
\bm{0}      & \mB_0   & \mB_1     &         & \ddots  &  \vdots \\
\vdots      & \ddots  & \ddots    & \ddots  &         &  \mB_k  \\
            &         &           & \ddots  & \ddots  &   \vdots\\
\bm{0}      &         & \hdots    & \bm{0}  & \mB_0   &   \mB_1\\
\bm{0}      &\bm{0}   &           & \hdots  & \bm{0}  &   \mB_0\\
\end{pmatrix}.
\label{eq:unrolledDAG}
\end{equation}
This allows us to rewrite~\eqref{eq:SVAR} as a linear structural equation model (SEM)~\citep{shimizu2006lingam}:
\begin{equation}
    \widetilde{\mX} =\widetilde{\mX}\mA + \widetilde{\mS},
    \label{eq:linear_SEM}
\end{equation}
where $\widetilde{\mX}\in\R^{N\times dT}$ consists of the $N$ time series as rows and $\widetilde{\mS}$ is defined similarly for the structural shocks. 
Since $\mA$ is a DAG, \eqref{eq:linear_SEM} represents a linear SEM. 

\paragraph{Original SparseRC} We now explain how SparseRC can be applied to learn the window graph from time series according to~\citet{misiakos2024icassp}. 
SparseRC can be used to learn $\mA$ from (many samples of) $\vx$ stacked as a matrix $\widetilde{\mX}\in \R^{N\times dT}$, generated from a linear SEM~\eqref{eq:linear_SEM}. 
Its optimization objective aims to minimize the number of approximated non-zero structural shocks $\widetilde{\mS}$ in~\eqref{eq:linear_SEM}. 
This is expressed with the following discrete optimization problem
\begin{equation}
    \widehat{\mA} = \argmin_{\mA\in \R^{dT\times dT}} \normo{\widetilde{\mX} - \widetilde{\mX}\mA},\quad \text{s.t. } \mA \text{  is acyclic}.
    \label{appendix:eq:opt_discrete}
\end{equation}
The window graph $\widehat{\mW}$ can be then extracted from the first row of the approximated $\widehat{\mA}$. 
SparseRC in practice uses a continuous relaxation to solve optimization problem~\eqref{appendix:eq:opt_discrete}, but here we keep the discrete formulation for simplicity.

It can be seen that the DAG $\mA$ consists of $dT$ nodes. In our smallest experiment this equals to $20\times 1000 = 20000$ nodes, which is already out of reach for SparseRC. In contrast, \mobius requires to learn only $(k+1)\times$ DAGs with $d$ nodes each. 
Thus, we necessarily need to formulate SparseRC differently to be able to compare against it. 

\paragraph{Modified SparseRC} The idea is to reduce the size of $\mA$ by getting rid of the $\bm{0}'$s in~\eqref{eq:unrolledDAG}. Specifically, instead of feeding SparseRC $\widetilde{\mX}$ we feed as input $\tX_{\text{past}}$. 
The resulting algorithm aims to find an $\widehat{\mA}$ according to:
\begin{equation}
    \widehat{\mA} = \argmin_{\mA\in \R^{(k+1)d\times (k+1)d}} \normo{\tX_{\text{past}} - \tX_{\text{past}}\mA},\quad \text{s.t. } \mA \text{  is acyclic}.
    \label{eq:app:opt_discrete_k+1}
\end{equation}
To be compatible with the data-generating process, the following structure  is assumed for $\mA$:
\begin{equation}
\mA  = \begin{pmatrix}
\mB_0  & \bm{0}       & ...        & \bm{0}         & \bm{0} \\
\mB_1  & \mB_0   & \ddots     &                & \bm{0} \\
\vdots      & \mB_1   & \ddots     & \ddots         & \vdots \\
\mB_{k-1}   &              & \ddots     & \mB_0   & \bm{0} \\
\mB_{k}     & \mB_{k-1}       &  \hdots    & \mB_1    & \mB_0,   \\
\end{pmatrix}
\label{eq:small_unrolledDAG}
\end{equation}
The optimization objective~\eqref{eq:app:opt_discrete_k+1} is different from $\normo{\tX - {\tX}_{\text{past}}\mW}$ used from \mobius and promotes a different convention in the data generating process. 
In particular by setting $\widetilde{\tS} = \tX_{\text{past}} - \tX_{\text{past}}\mA$ the structural shock $\widetilde{\vs}_{t-j}$ corresponding to the position $j$ of row $t$ of $\vx_{t,\text{past}} = \left(\vx_t, \vx_{t-1},...,\vx_{t-j},...,\vx_{t-k}\right)$ of a sample $i$ of ${\tX}_{\text{past}}$ would be:
\begin{equation}
    \widetilde{\vs}_{t-j} = \vx_{t-j} - \vx_{t-j}\mB_0 + \vx_{t-j - 1}\mB_1 + ... + \vx_{t-k}\mB_{k-j} \neq \vs_{t-j}.
\end{equation}
This implies that the approximation of the structural shocks is not consistent with the data generation in~\eqref{eq:SVAR}, except when $j=0$. 
Thus, only the first column of $\mA$ promotes the correct equations and the rest undermine the performance of SparseRC. 
Resolving this discrepancy and keeping only the first column as trainable parameters is among the technical contributions of our paper.

\section{\mobius optimization and comparison with baselines}
\label{appendix:sec:spinsvar_implementation}

\paragraph{\mobius} Our implementation in PyTorch is outlined in Algorithm \ref{algo:mobius}. 
It parametrizes the window graph matrix $\mW$ using a single PyTorch linear layer 
and optimizes the objective function \eqref{eq:cont_opt} with the Adam optimizer.  
The overall computational complexity of the algorithm is:
\begin{equation}
    \mathcal{O}\left(M \cdot (NT d^2 k + d^3)\right),
\end{equation}
where $M$ is the total number of epochs (up to $10^4$).

The primary term in our objective, $N\left\{\log\left\|\tX - L\left(\tX\right)\right\|_1 - \frac{1}{d}\log\left|\text{det}\left(\mI - \mB_0\right)\right|\right\}$, 
represents a fundamental difference from prior work on causal discovery in time series. 
Methods such as VAR-based optimization approaches \citep{pamfil2020dynotears,sun2023ntsnotears} 
typically rely on a mean-squared error loss supplemented by an $\normlone$ penalty 
to promote sparsity in the DAG. In contrast, both the main term and the regularizer in our objective are $\normlone$ norms, 
promoting sparsity not only in the DAG but also in the SVAR input. 
This design aligns with the assumption of sparse SVAR input. Potentially, $\normltwo$  leads to longer convergence times, which makes our algorithm terminate faster in the experiments.

\begin{algorithm}[h]
\caption{\mobius: DAG Learning from Time Series with Few structural shocks}
\label{algo:mobius}
\begin{algorithmic}[1]
\REQUIRE Time series data tensor \( \tX \in \R^{N \times T \times d} \), \( \lambda_1, \lambda_2 \) regularization parameters and threshold $\omega$.
\ENSURE Weighted window graph \(\widehat{\mW} = \pmat{\mB_0\\ \vdots \\ \mB_k} \) and structural shocks $\widehat{\tS}$.

\STATE \textbf{Initialize:} 
\STATE  A single linear layer $L(\text{input: }d(k+1),\,\text{output: }d)$ in PyTorch that represents $\widehat{\mW}$.
\STATE  Tensor $\tX_{\text{past}}\in\R^{N\times T\times d(k+1)}$, where the $(n,t)$ entry is the vector $\vx_{t,\text{past}} = (\vx_t,\, \vx_{t-1},\,...,\,\vx_{t-k})\in\R^{1\times d(k+1)}$.
\vspace{5pt}
\STATE \textbf{Iterate:} 
\FOR{each training epoch up to $M=10^4$}
    \STATE Compute the loss:
    \[
    N\left\{\log\left\|\tX - L\left(\tX\right)\right\|_1 - \frac{1}{d}\log\left|\text{det}\left(\mI - \mB_0\right)\right|\right\} + \lambda_1 \|\mW\|_1 + \lambda_2 h(\mB_0),
    \]
    where \( h(\mB) = \text{tr}\left(e^{\mB\odot\mB}\right) - d \) .
    \STATE Update the linear layer parameters $\widehat{\mW}$ with Adam optimizer.
    \STATE Stop early if the loss doesn't improve for 40 epochs.
\ENDFOR
\vspace{5pt}
\STATE \textbf{Post-processing:}
\STATE Set the entries $w_{ij}$ of $\mW$ with \(|w_{ij}| <  \omega \) to zero.
\STATE Compute the unweighted version \( \mU\in\{0,1\}^{d(k+1)\times d} \) of $\mW$. 
\STATE Compute the approximated structural shocks:
\[
\widehat{\tS} = \tX - \tX_{\text{past}} \widehat{\mW}.
\]

\RETURN \( \widehat{\mW}, \widehat{\mU}, \widehat{\tS} \)

\end{algorithmic}
\end{algorithm}

\subsection{Comparison with baselines}
\label{appendix:subsec:comparison_baselines}

\paragraph{SparseRC} As we explained in the main text, the method from~\citet{misiakos2024fewrootcauses} 
is infeasible to execute for long time series data. 
In its original form, SparseRC has complexity $\mathcal{O}\left(M \cdot (Nd^2T^2 + d^3T^3)\right)$, where $M$ is the total number of iterations. SparseRC learns a $dT \times dT$ unrolled DAG, which for our smaller scenario, results in a DAG with $d \times T = 20 \times 1000 = 20000$ nodes that goes beyond its computational reach~\citep{misiakos2024fewrootcauses}. 

In Appendix \ref{appendix:exp:sparserc}, we design a modified version of SparseRC 
that learns a $(k+1)d \times (k+1)d$ adjacency matrix, which ultimately leads to a complexity of $\mathcal{O}\left(M \cdot (NT d^2 k^2 + d^3k^3)\right)$.
This adaptation can be executed in most scenarios but comes at the cost of reduced model performance.

\paragraph{\varlingam} 
First, the method fits a VAR model to the data:
\begin{equation}
    \vx_t = \widetilde{\mB}_1\vx_{t-1} + ... + \widetilde{\mB}_k\vx_{t-k} + \vn_t,
\end{equation}
and then performs Independent Component Analysis (ICA) to compute the self-dependencies matrix $\mB_0$:
\begin{equation}
    \vn_t = (\mI - \mB_0)\vn_t + \vs_t.
\end{equation}
The resulting matrices are calculated as 
\[
\mB_{\tau} = (\mI - \mB_0)\widetilde{\mB}_{\tau}.
\]

The ICA step can be replaced with Direct LiNGAM~\citep{shimizu2011directlingam}, 
which guarantees convergence in a finite number of steps (under certain assumptions). 
This variation leads to the method \dlingam. However, both approaches have worse complexity compared to ours: 
\begin{itemize}
    \item For Direct LiNGAM: $\mathcal{O}\left(NTd^2k+ NTd^3M^2 + d^4M^3\right)$, where $M$ is the number of iterations of Direct LiNGAM.
    \item For ICA LiNGAM: $\mathcal{O}\left(NTd^2k + NTd^3 + d^4\right)$, which lacks convergence guarantees.
\end{itemize}

In the large-DAG regime, these algorithms are inevitably slower than ours.

\paragraph{\clingam} \citet{akinwande2024acceleratedlingam} accelerate \dlingam by implementing a parallelized version on GPUs. While this method is faster than \dlingam, our experiments show that it still times out, likely due to high convergence times.

\paragraph{DYNOTEARS} 
Here, the mean-square error (MSE) is used, transforming the optimization into a quadratic problem:
\begin{equation}
    \frac{1}{2NT}\left\|\tX - \tX_{\text{past}}\mW\right\|_2 + \lambda_w \|\mW\|_1 + \frac{\rho}{2} h(\mB_0)^2 + a h(\mB_0),
\end{equation}
where the \(L^2\) norm in the first term doesn't enforce sparsity on the structural shocks. 
As a result, this method experiences longer convergence times and produces a poor approximation of the ground truth window graph.

\paragraph{TCDF} This method fits convolutional neural networks (CNNs) to predict the time series at each node, 
based on the time-series values of other nodes in previous time steps. 
The approximation is optimized using the MSE loss. 
However, both the non-linearity of CNNs and the MSE loss do not align with our data generation process, 
which limits the method's effectiveness for our specific task.

\paragraph{NTS-NOTEARS} Similar to TCDF, this method also uses CNNs and MSE loss to approximate the window graph. 
In addition, the acyclicity regularizer from NOTEARS is applied. 
For similar reasons, we anticipate low performance in our experiments with this method as well, 
due to the mismatch between the assumptions of the method and the characteristics of our data.

\paragraph{tsFCI, PCMCI} 
For the constraint-based baselines, there is no clear comparison in terms of optimization. 
These methods rely on statistical independence tests to infer causal dependencies between nodes at different time points. 
Empirically, however, these methods perform poorly, likely due to their inability to determine the causal direction for every edge they discover.

\section{Sparsity properties of Laplace distribution} 
\label{appendix:subsec:Laplace_properties}

A random variable $X$ follows a Laplace distribution~\citep{eltoft2006multivariatelaplacedist}, denoted as $\text{Laplace}(\mu,\beta)$, if its probability density function is given by:
\begin{equation}
    f_X(x|\mu, \beta) = \frac{1}{2\beta}e^{-\frac{|x-\mu|}{\beta}}.
\end{equation}

We now analyze why the Laplace distribution is better suited for modeling sparse vectors compared to the Gaussian distribution. Specifically, we consider Laplacian noise variables centered at zero, setting $\mu=0$. Our motivation is that Laplace-distributed variables are more likely to produce large outliers, whereas Gaussian-distributed variables tend to be concentrated around zero.

To investigate sparsity, we consider three approaches for achieving approximately $5\%$ sparsity. The first follows our experimental procedure described in Section~\ref{subsec:synthetic}, which combines Bernoulli and uniform distributions. The second uses a Gaussian distribution, and the third uses a Laplace distribution. Since strict sparsity cannot be achieved, $95\%$ of the values will be approximately zero. 

We compare these distributions in terms of their sparsity-inducing properties by addressing the following question: \textit{How much more significant are the nonzero values compared to the approximately zero ones?} To do so, we define a threshold $\omega$ that classifies values above $\omega$ as significant and those below $\omega$ as approximately zero.

For each scenario, we generate a random vector $\vs$ with $d$ entries $(s_1, \dots, s_d)$ and consider a threshold of $\omega = 0.1$.

\paragraph{Bernoulli \& Uniform} 
Each $c_i$ is generated independently, and with probability $1 - p = 0.95$, it is set to zero.  
Otherwise, with probability $p = 0.05$, it takes a uniform random value from the range $[-0.4, -0.1] \cup [0.1,0.4]$.
The upper bound of $0.4$ ensures that the maximum absolute value is comparable to that of the Laplace distribution, as described later.
To each $c_i$, we then add Gaussian noise with a standard deviation of $0.03$. Since $99\%$ of the Gaussian noise values lie within $[-0.09,0.09]$, this noise does not significantly affect the sparsity structure.
Thus, given $\omega = 0.1$, approximately $95\%$ of the entries in $\vs$ will have absolute values below $\omega$, effectively maintaining sparsity.

\paragraph{Gaussian} 
For a Gaussian-distributed variable, it is known that approximately $95\%$ of values lie within $[-2\sigma, 2\sigma]$. To achieve the required sparsity threshold $\omega = 0.1$, we set the standard deviation to $\sigma = 0.05$.

\paragraph{Laplace} 
For a Laplace-distributed variable $X$, the probability that its absolute value does not exceed $\omega$ is given by:
\begin{align*}
    \prob{|X| \leq \omega} &= \int_{-\omega}^{\omega}\frac{1}{2\beta}e^{-\frac{|x|}{\beta}}dx \\
    &= 2\int_{0}^{\omega}\frac{1}{2\beta}e^{-\frac{x}{\beta}}dx \\
    &= \int_{0}^{\omega}\frac{1}{\beta}e^{-\frac{x}{\beta}}dx \\
    &= -e^{-\frac{x}{\beta}}\Big|_{0}^{\omega} = 1 - e^{-\frac{\omega}{\beta}}.
\end{align*}
Setting $\beta = \omega / 3$ ensures that $\prob{|X| \leq \omega} \approx 0.95$, thereby achieving the desired sparsity.

\paragraph{Empirical Evaluation} 
With the distribution parameters set, we empirically evaluate the sparsity patterns of the generated vectors. Our goal is to demonstrate that the Laplace distribution is better suited for generating sparse vectors compared to the Gaussian. For each distribution listed in Table~\ref{appendix:tab:distributions_sparsity}, we generate a vector $\vs$ with $d = 10^6$ entries and compute the following evaluation metrics:
\begin{itemize}
    \item \textbf{Sparsity fraction}: The percentage of values with absolute values greater than $\omega$.
    \item \textbf{Maximum absolute value}: $\max_i |c_i|$.
    \item \textbf{Contrast ratio}: 
    \begin{equation}
    \text{Contrast ratio} = \frac{\frac{1}{M} \sum_{|c_i| > \omega} |c_i|}{\omega},
\end{equation}
where $M$ is the number of entries satisfying $|c_i| > \omega$.
\item \textbf{Signal-to-noise ratio (SNR)}:
\begin{equation}
    \text{SNR} = \frac{\sum_{|c_i| > \omega} |c_i|^2}{\sum_{|c_i| < \omega} |c_i|^2}.
\end{equation}
\end{itemize}

The computational results are shown in Table~\ref{appendix:tab:distributions_sparsity}. Given the $5\%$ sparsity constraint, the best performance is achieved by the Bernoulli-Uniform method, which produces higher values with a maximum magnitude of $0.51$ and exhibits a superior contrast ratio and SNR, indicating better sparsity characteristics. The Laplace distribution achieves the second-best performance.

\begin{table}[ht]
    \footnotesize
    \centering
    \caption{Empirical sparsity evaluation for different distributions.}
    \begin{tabular}{@{}lllll@{}}
        \toprule
        Method  & Sparsity &  Maximum Absolute Value & Contrast Ratio & SNR \\
        \midrule
        Bernoulli \& Uniform &  $5.0\%$  &  $ 0.51   $  & $2.50$& $3.40$  \\ 
        Gauss $\mathcal{N}(0,0.051^2)$& $5.0\%$  &  $   0.25  $   &  $1.19$ & $0.38$  \\ 
        Laplace$\left(0,\frac{1}{3}\right)$ &  $5.0\%$  &  $ 0.53 $ &  $1.33$ & $0.73$  \\ 
        \bottomrule
    \end{tabular}
    \label{appendix:tab:distributions_sparsity}
\end{table}

\section{Additional Experiments}
\label{app:sec:more_experiments}
In this section we include additional synthetic experiments and additional results regarding the simulated and real-world financial datasets.


\subsection{Empirical Stability of Time Series}
\label{appendix:exp:stability}

According to Theorem~\ref{appendix:th:stabilitySEM}, the stability of the time-series data $\mX$ requires that the weight matrices $\mB_0, \mB_1, \dots, \mB_k$ satisfy an upper bound $w$ such that: 
\begin{equation}
    (5 + 2 + 2)w = 9w < 1, \quad \text{or equivalently,} \quad w < 0.11.
\end{equation}
However, to allow for a greater variety of edge weights, we instead assign uniformly random weights from the range $[0.1, 0.5]$. 
In practice, $\mX$ is typically observed to converge. If any generated dataset results in unbounded values—specifically, if the average value of $\mX$ exceeds $10^6 \cdot NdT$—we discard the sample and repeat the data generation process.

\subsection{Additional metrics}
\label{appendix:exp:additional_metrics}

In Figs. \ref{app:fig:synthetic_plots_laplace},\ref{app:fig:synthetic_plots_bernoulli} we provide the additional metrics AUROC (area under ROC curve), F1-score, the normalized mean square error (NMSE) and the SHD and NMSE on the input $\widehat{\tS}$ approximation. 
Formally, if $\widehat{\mW}$ and $\widehat{\tS}$ are the approximations of the ground truth window graph $\mW$ and structural shocks $\mS$ then:
\begin{equation}
    \text{NMSE} = \frac{\left\|\widehat{\mW} - \mW\right\|_2}{\left\|\mW \right\|_2},\quad \mS\text{ NMSE}=\frac{\left\|\widehat{\tS} - \tS\right\|_2}{\left\|\tS \right\|_2}.
\end{equation}

In Fig.\ref{app:fig:synthetic_plots_laplace} the computation of $\widehat{\tS}$ NMSE is numerically unstable for all methods and is not reported.

\begin{figure}[H]
    \centering
    \begin{subfigure}{0.26\linewidth}
        \centering
        \includegraphics[width=\linewidth]{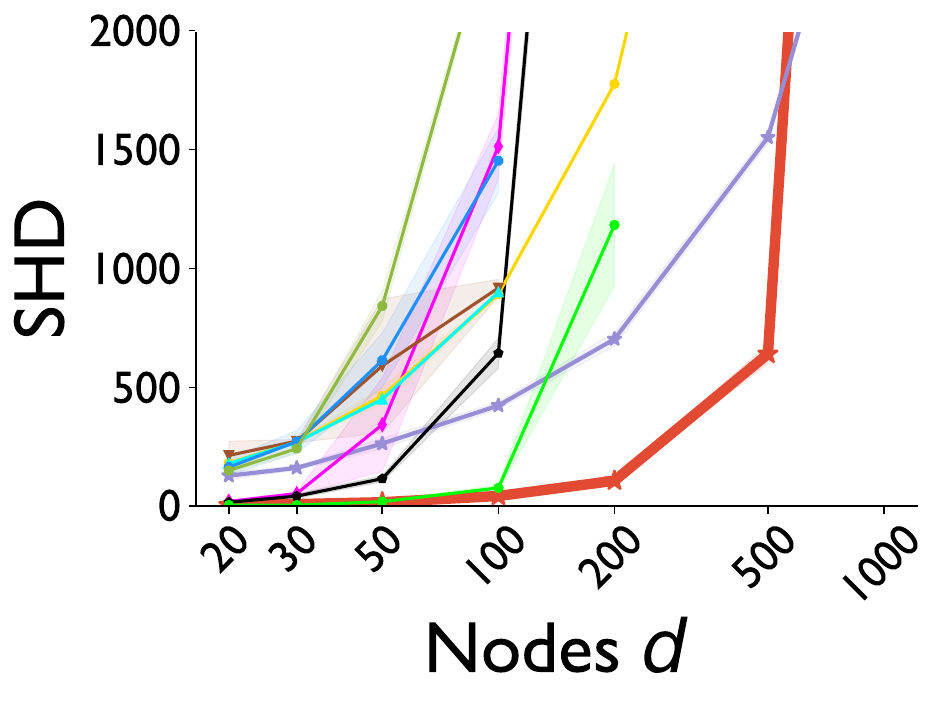}

        \includegraphics[width=\linewidth]{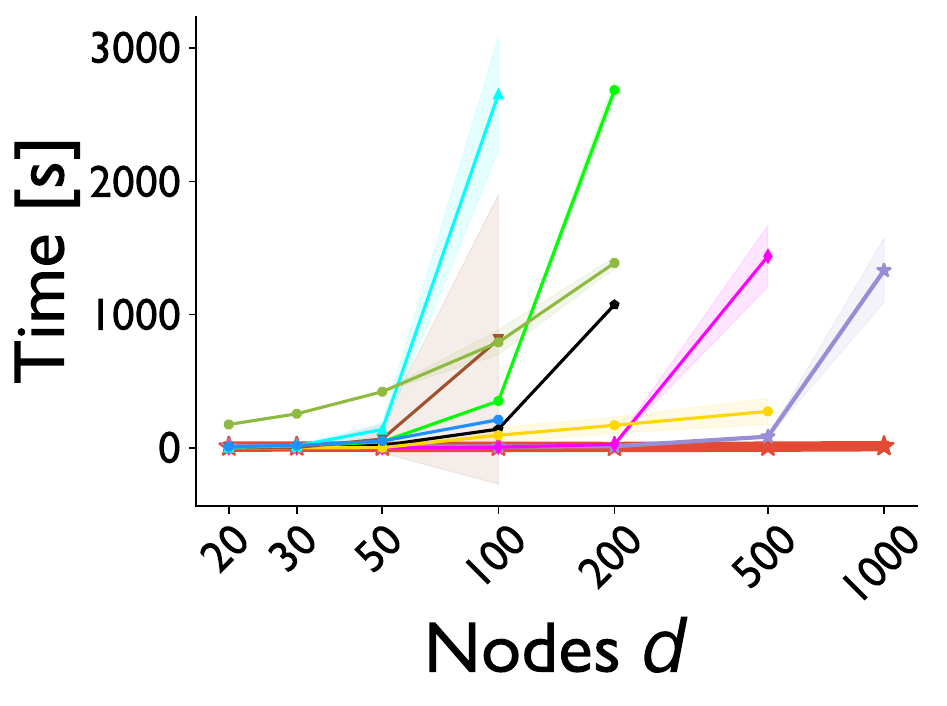}
        
        \includegraphics[width=\linewidth]{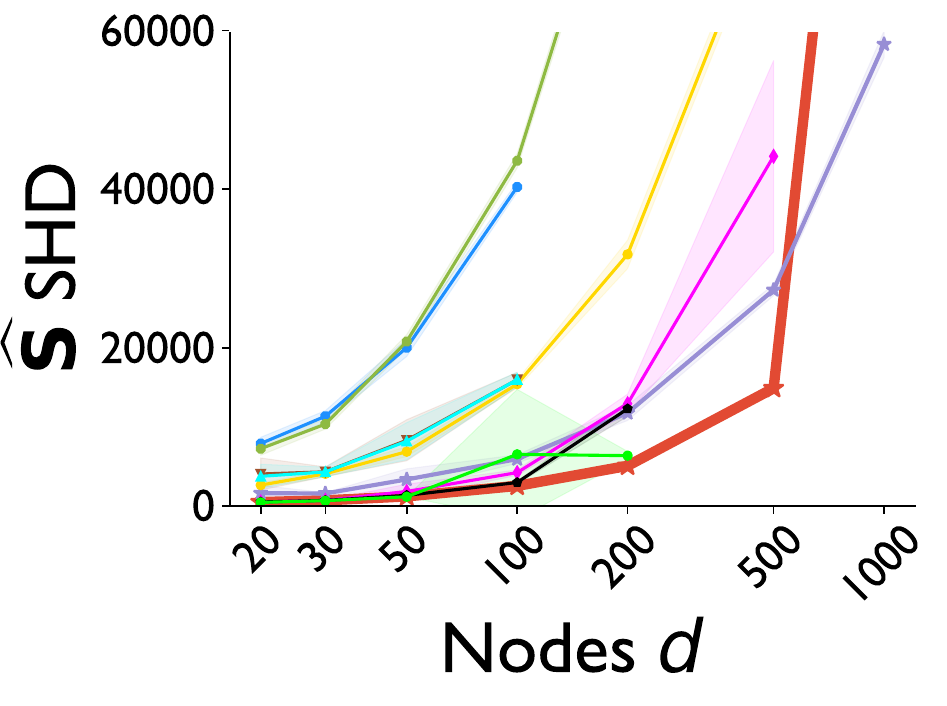}

        \includegraphics[width=\linewidth]{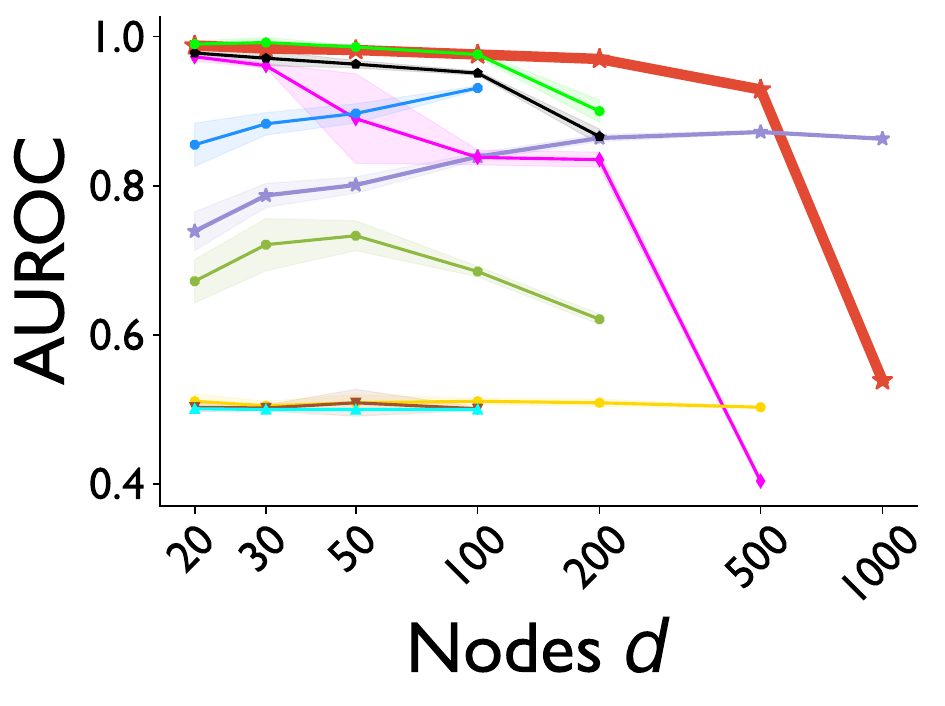}
        
        \includegraphics[width=\linewidth]{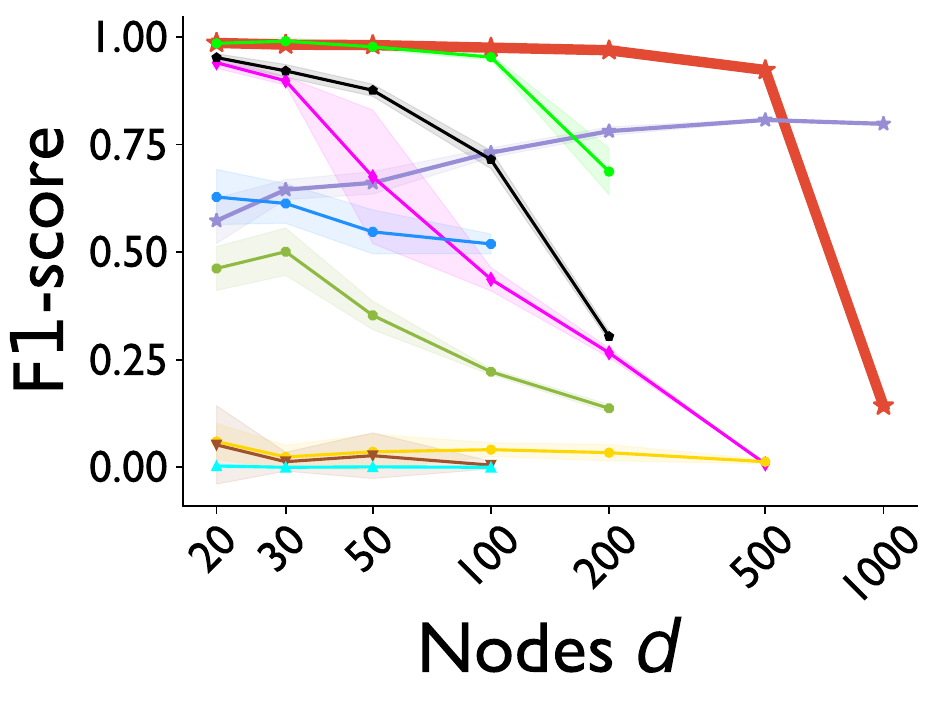}

        \includegraphics[width=\linewidth]{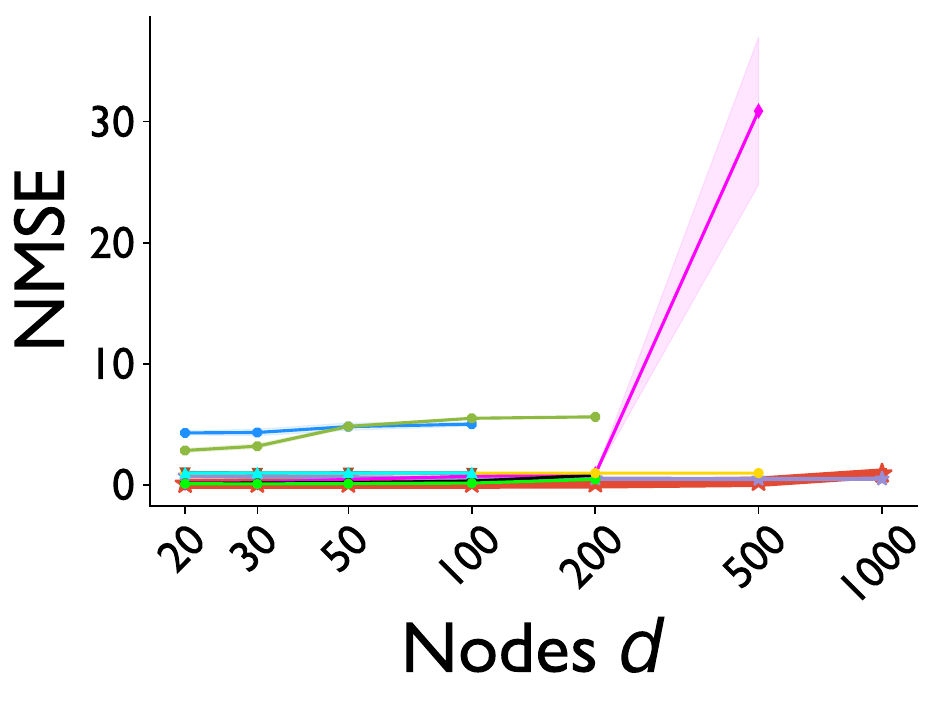}
        \caption{$N=1$, $T=1000$}
    \end{subfigure}
    \hfill
    \begin{subfigure}{0.26\linewidth}
        \centering
        \includegraphics[width=\linewidth]{figures/plot_noise_laplace_noise_std_0.033_sparsity_0.0_timeout_10000_dataset_laplace__shd.pdf}

        \includegraphics[width=\linewidth]{figures/plot_noise_laplace_noise_std_0.033_sparsity_0.0_timeout_10000_dataset_laplace__time.pdf}

        \includegraphics[width=\linewidth]{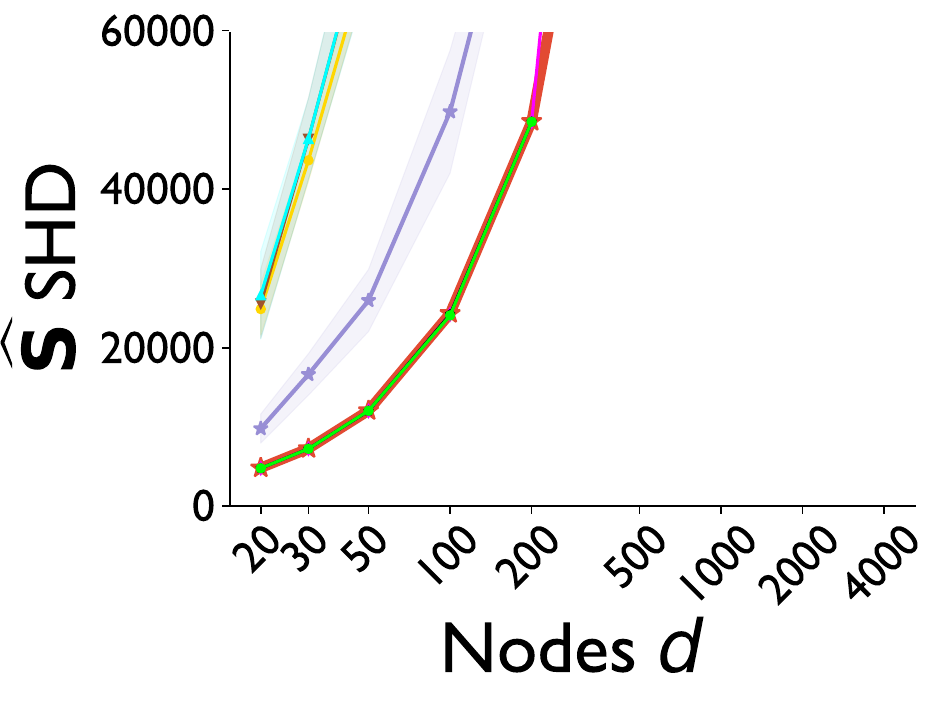}

        \includegraphics[width=\linewidth]{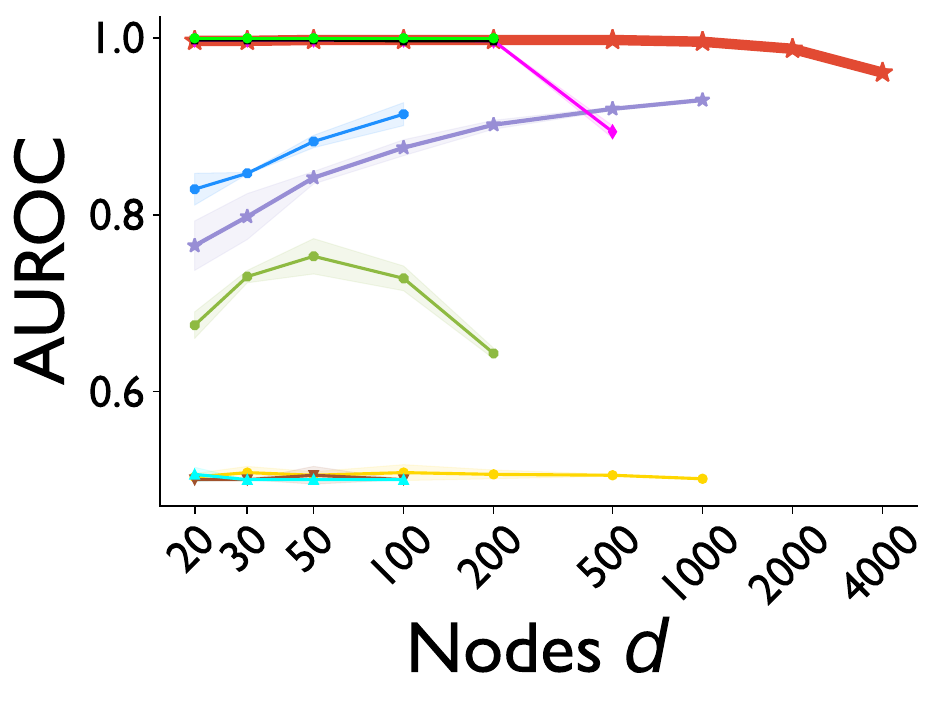}
        
        \includegraphics[width=\linewidth]{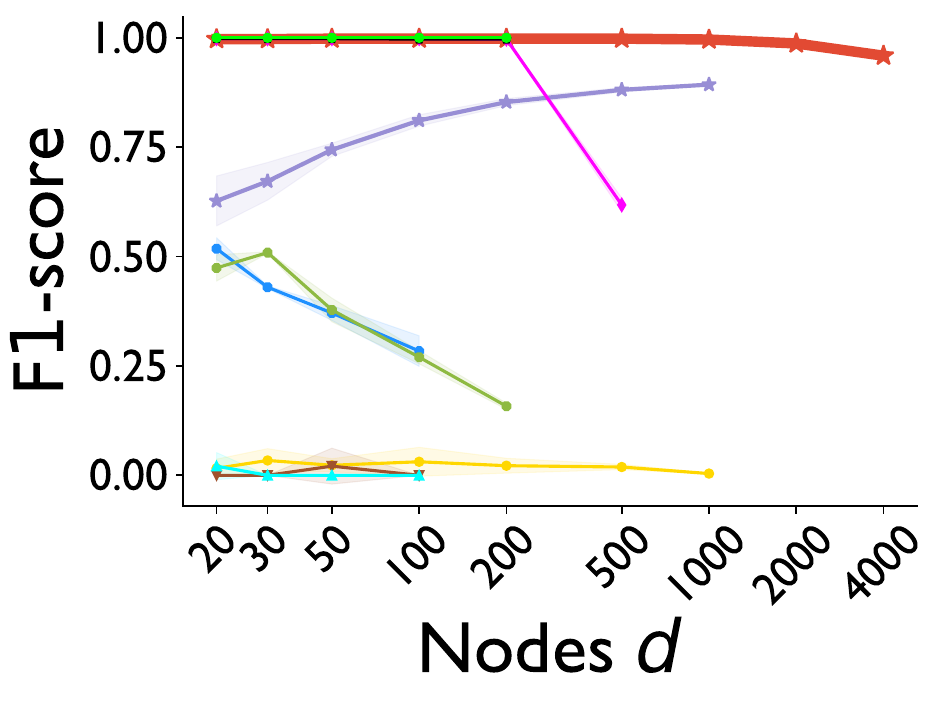}

        \includegraphics[width=\linewidth]{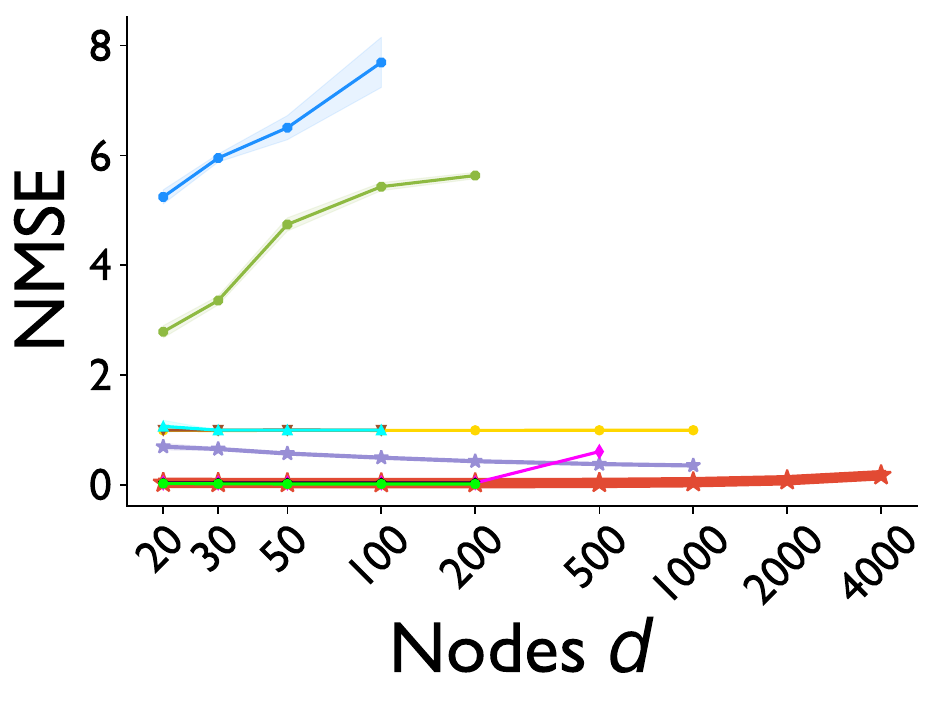}
        \caption{$N=10$, $T=1000$}
    \end{subfigure}
    \hfill
    \begin{subfigure}{0.26\linewidth}
        \centering
        \includegraphics[width=\linewidth]{figures/plot_noise_laplace_noise_std_0.033_sparsity_0.0_samples__1,_2,_3,_5,_10,_20__timeout_10000_dataset_laplace__shd.pdf}
                
        \includegraphics[width=\linewidth]{figures/plot_noise_laplace_noise_std_0.033_sparsity_0.0_samples__1,_2,_3,_5,_10,_20__timeout_10000_dataset_laplace__time.pdf}
        
        \includegraphics[width=\linewidth]{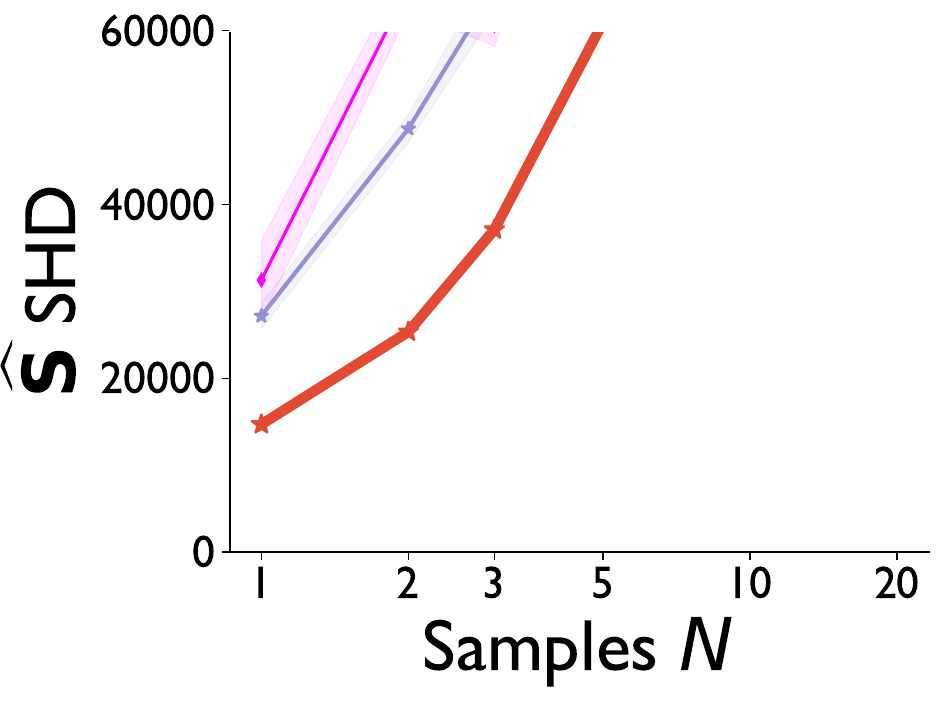}

        \includegraphics[width=\linewidth]{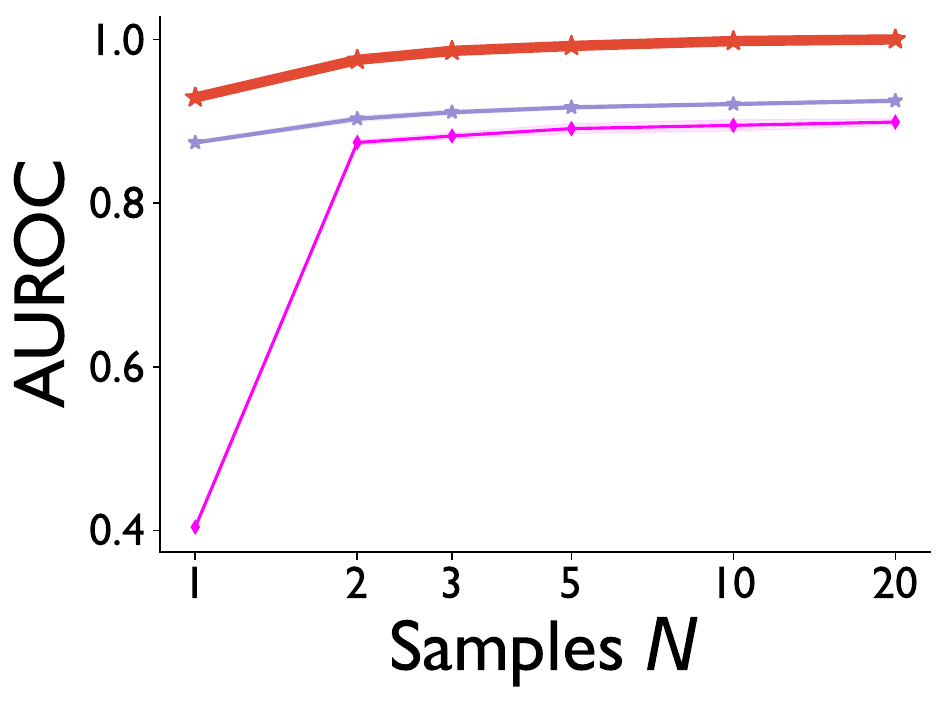}
        
        \includegraphics[width=\linewidth]{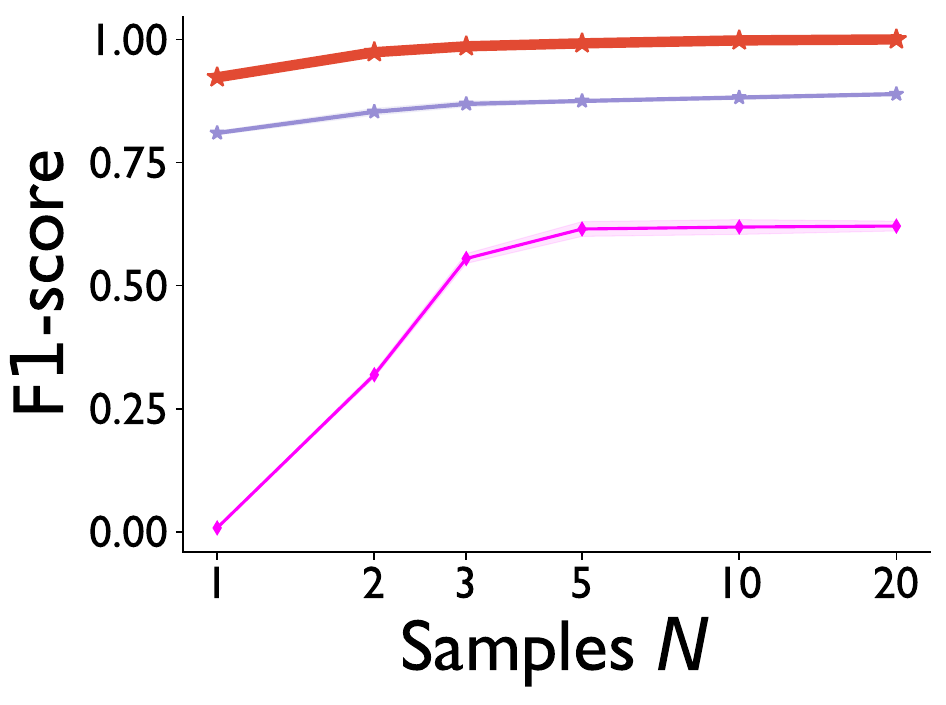}

        \includegraphics[width=\linewidth]{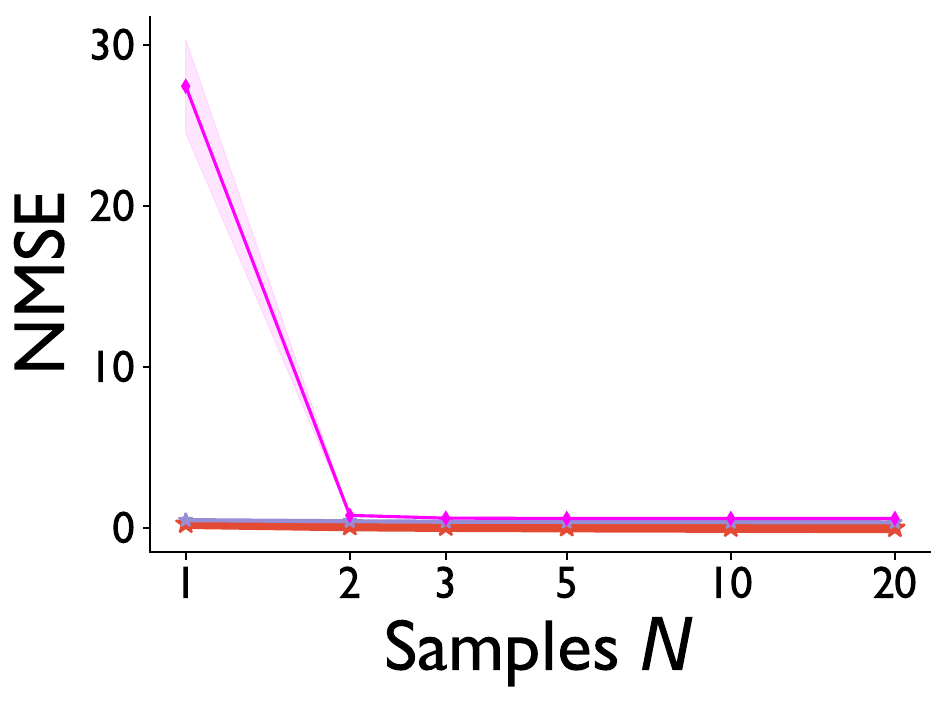}
        \caption{$d=500$, $T=1000$}
    \end{subfigure}
    \hfill
    \begin{subfigure}{0.18\linewidth}
        \hspace{30pt}
        \includegraphics[trim={6.5cm 0 0 0}, clip, width=\linewidth]{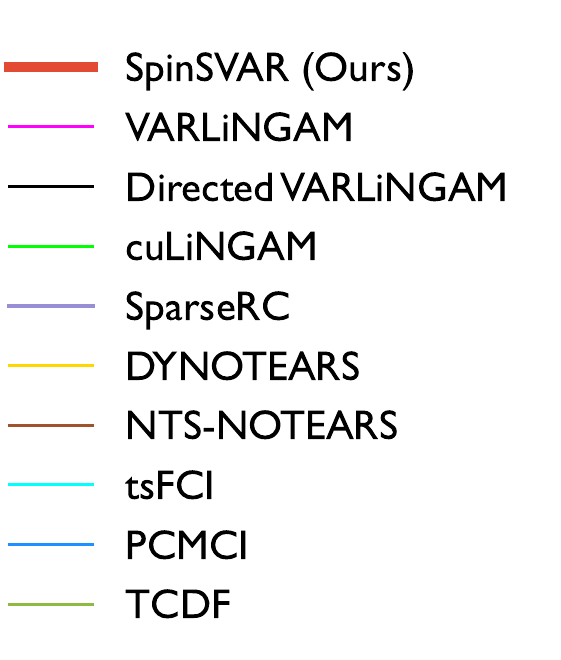}
        \vspace{250pt}
    \end{subfigure}
    \caption{Performance on synthetic data (Laplacian distributed input): AUROC ($\uparrow$), F1-score ($\uparrow$) NMSE ($\downarrow$) and structural shocks NMSE ($\downarrow$). (a), (b) correspond to $N= 1$ and $N=10$ samples of time-series with $T=1000$ and  varying number of nodes. (c) corresponds to $d=500$ nodes and varying samples $N$ of time-series of length $T=1000$.}
    \label{app:fig:synthetic_plots_laplace}
\end{figure}

\begin{figure}[H]
    \centering
    \vspace{-25pt}
    \begin{subfigure}{0.26\linewidth}
        \centering
        \includegraphics[width=\linewidth]{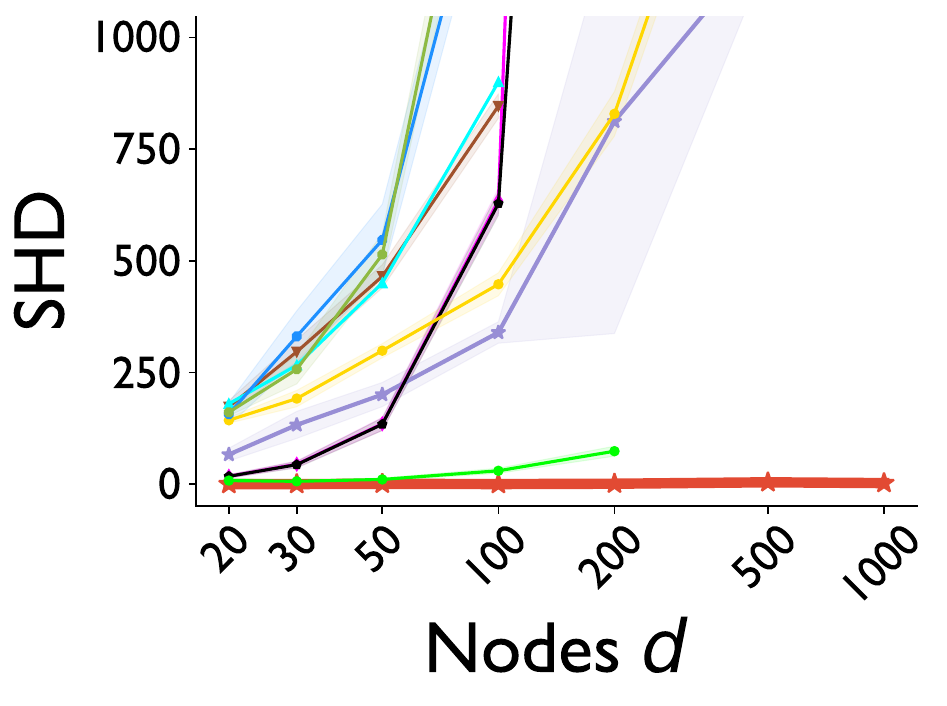}
        
        \includegraphics[width=\linewidth]{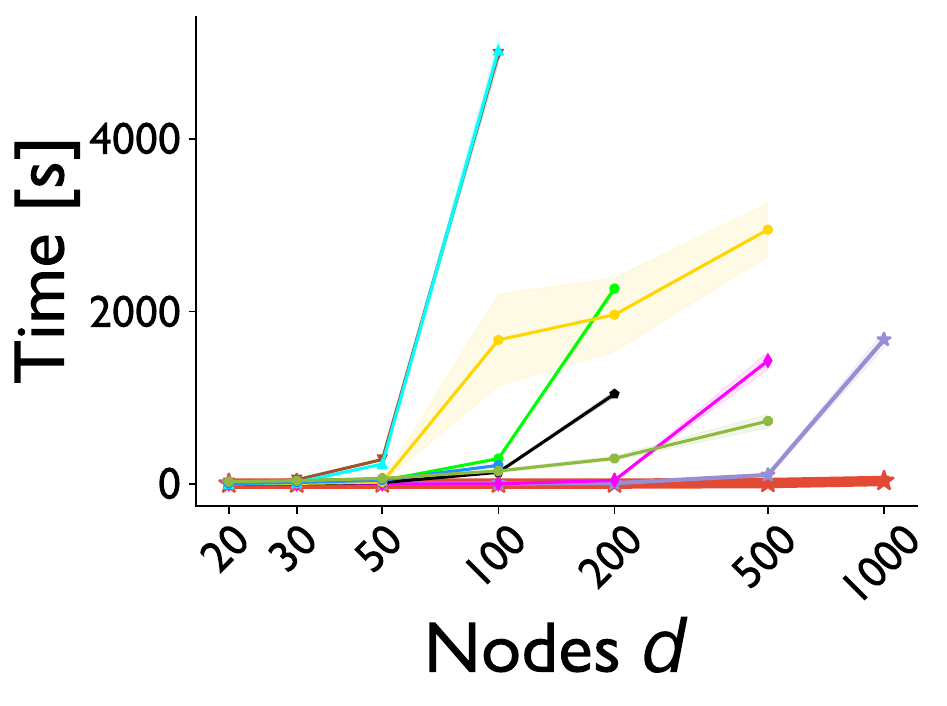}
        
        \includegraphics[width=\linewidth]{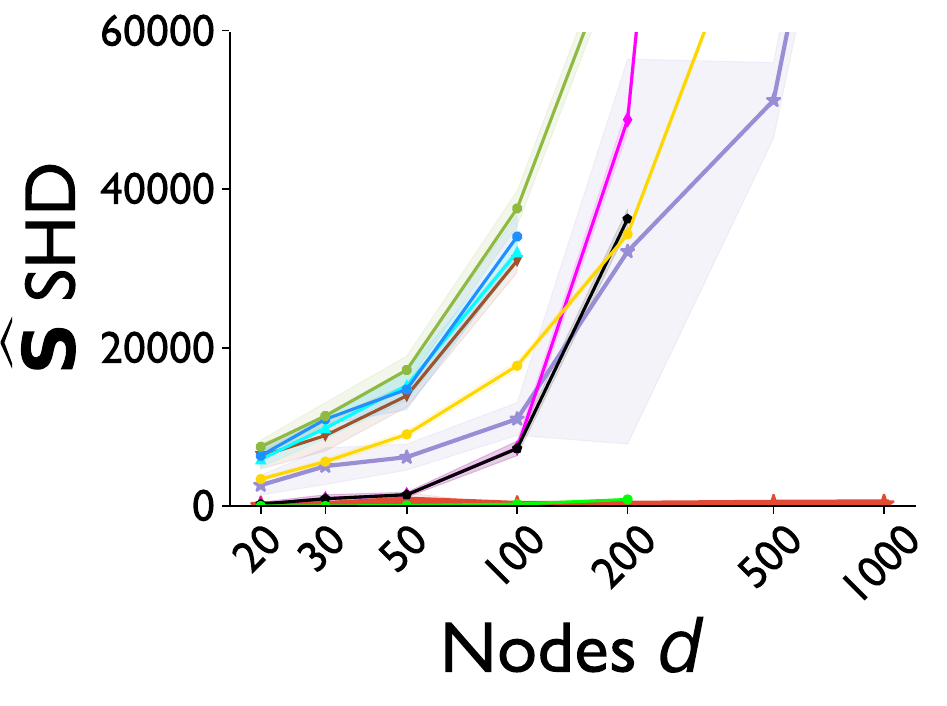}
        
        \includegraphics[width=\linewidth]{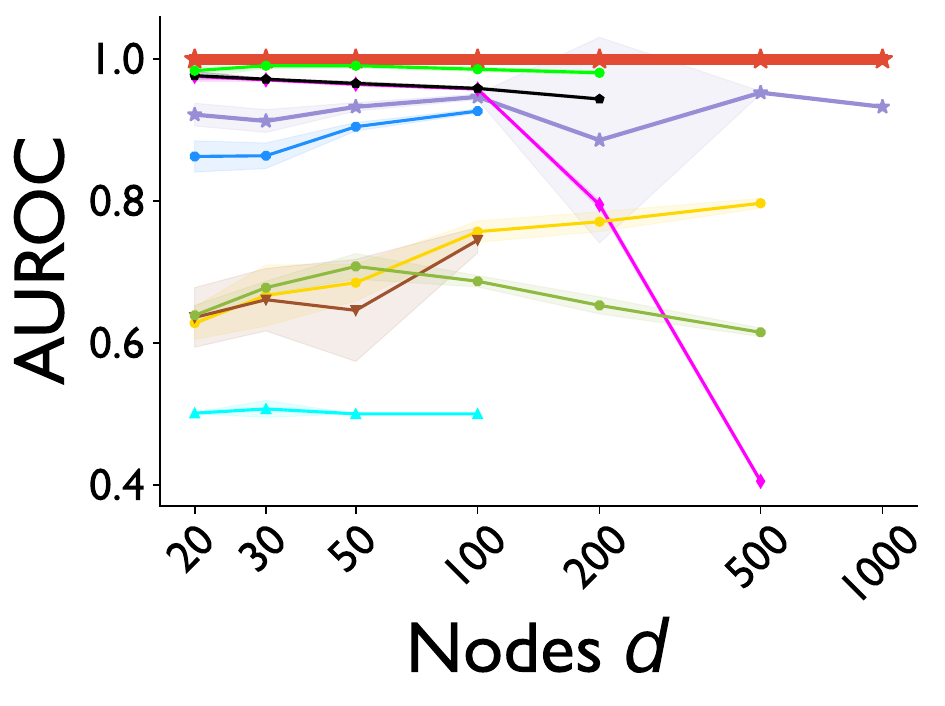}
        
        \includegraphics[width=\linewidth]{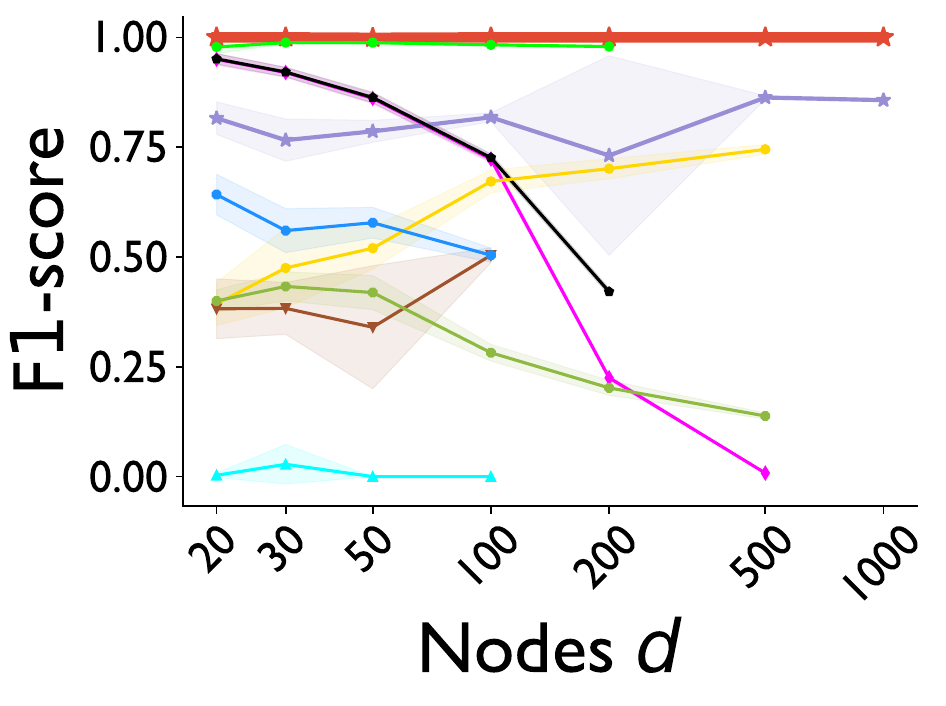}

        \includegraphics[width=\linewidth]{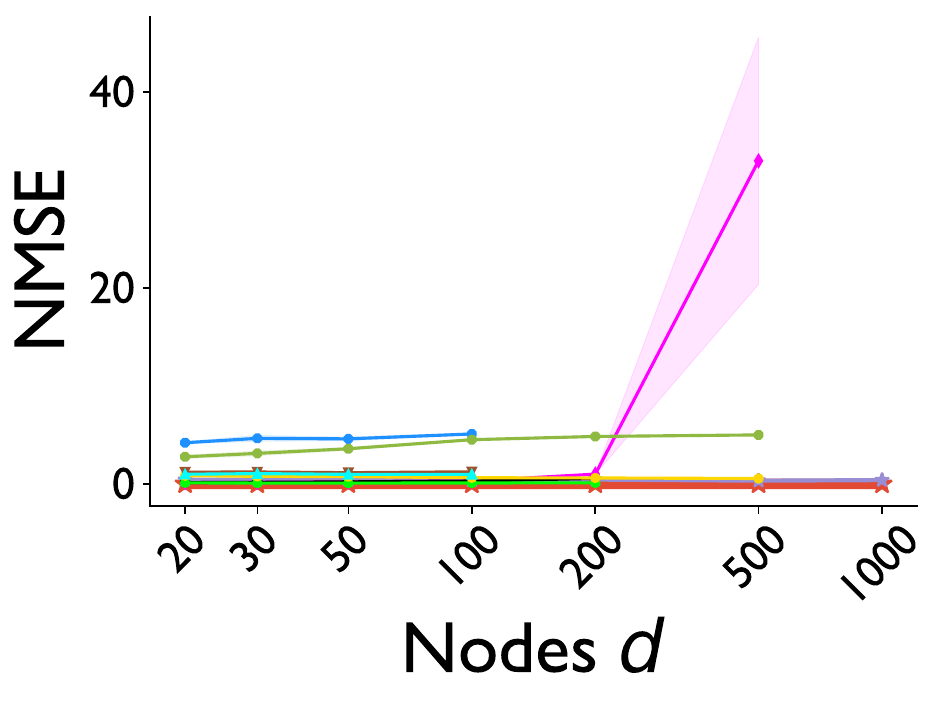}

         \includegraphics[width=\linewidth]{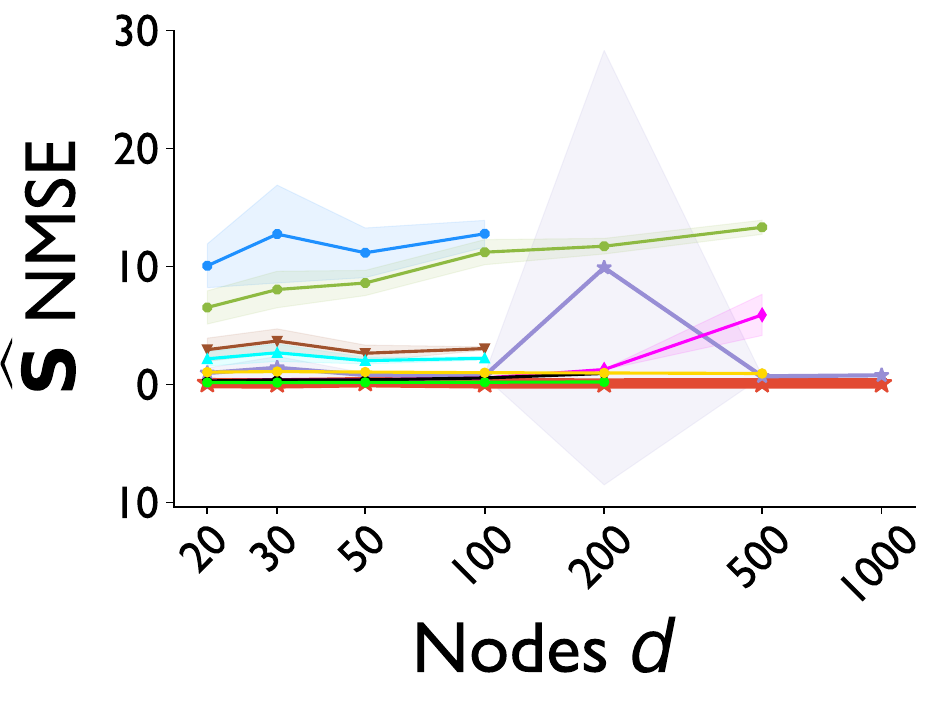}
        \caption{$N=1$, $T=1000$}
    \end{subfigure}
    \hfill
    \begin{subfigure}{0.26\linewidth}
        \centering        
        \includegraphics[width=\linewidth]{figures/plot_timeout_10000__shd.pdf}
        
        \includegraphics[width=\linewidth]{figures/plot_timeout_10000__time.pdf}
        
        \includegraphics[width=\linewidth]{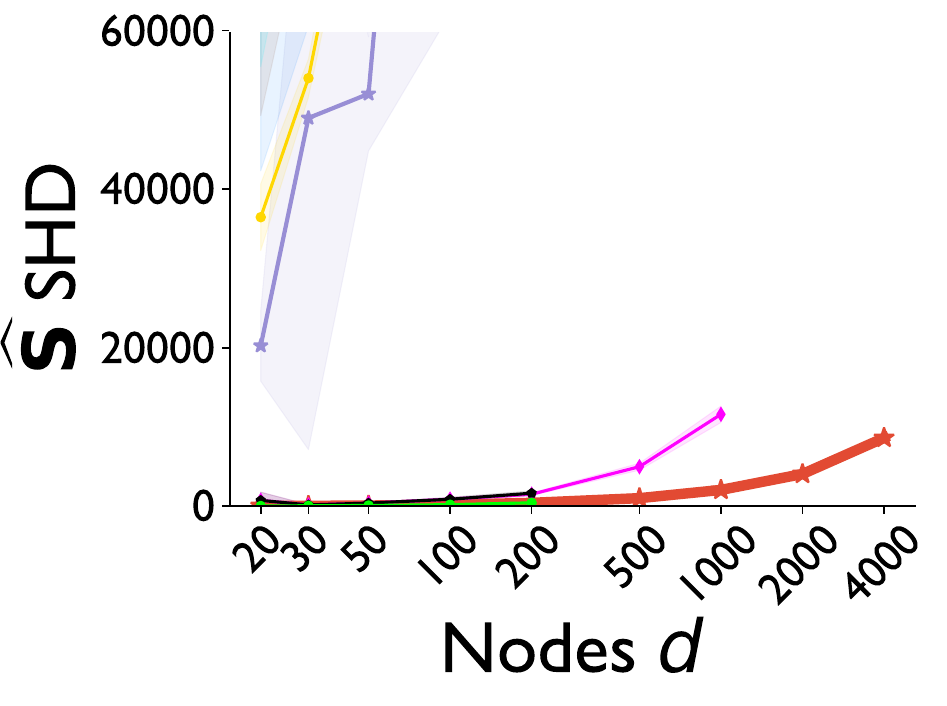}
        
        \includegraphics[width=\linewidth]{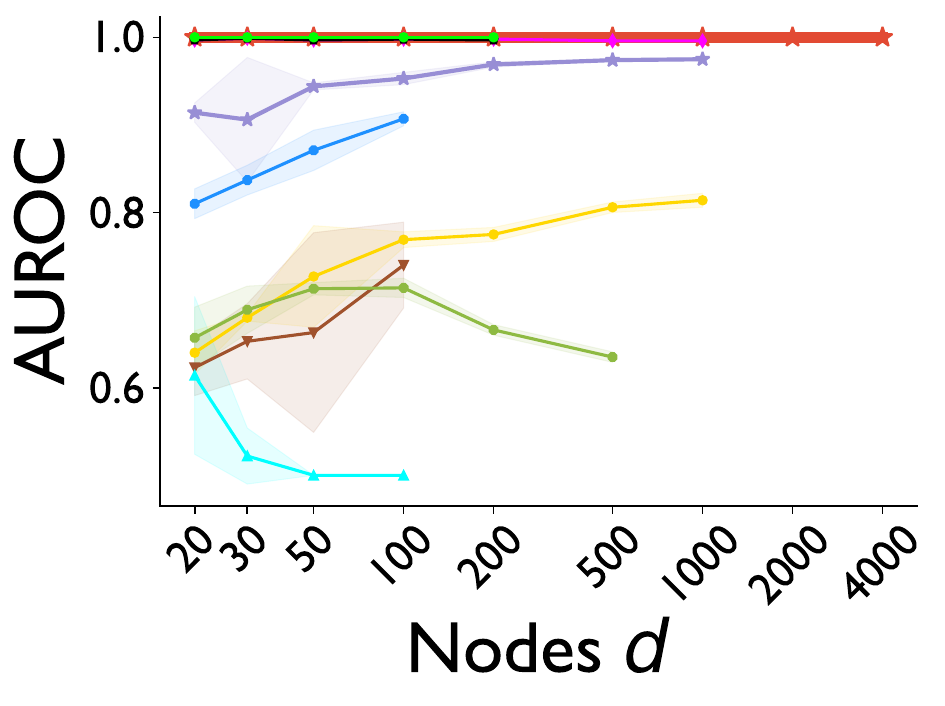}
        
        \includegraphics[width=\linewidth]{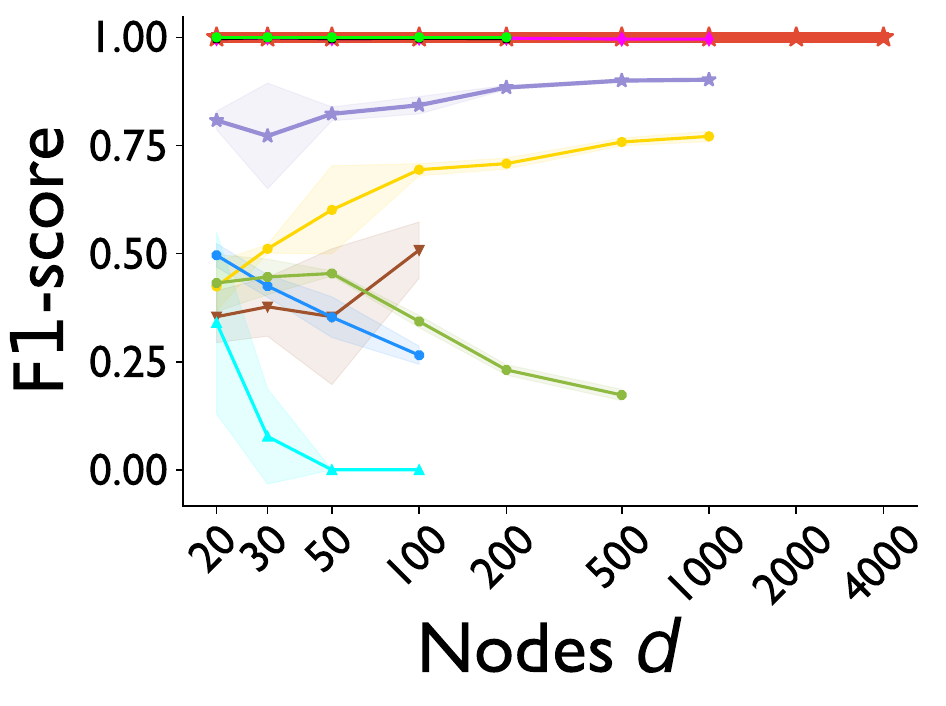}

        \includegraphics[width=\linewidth]{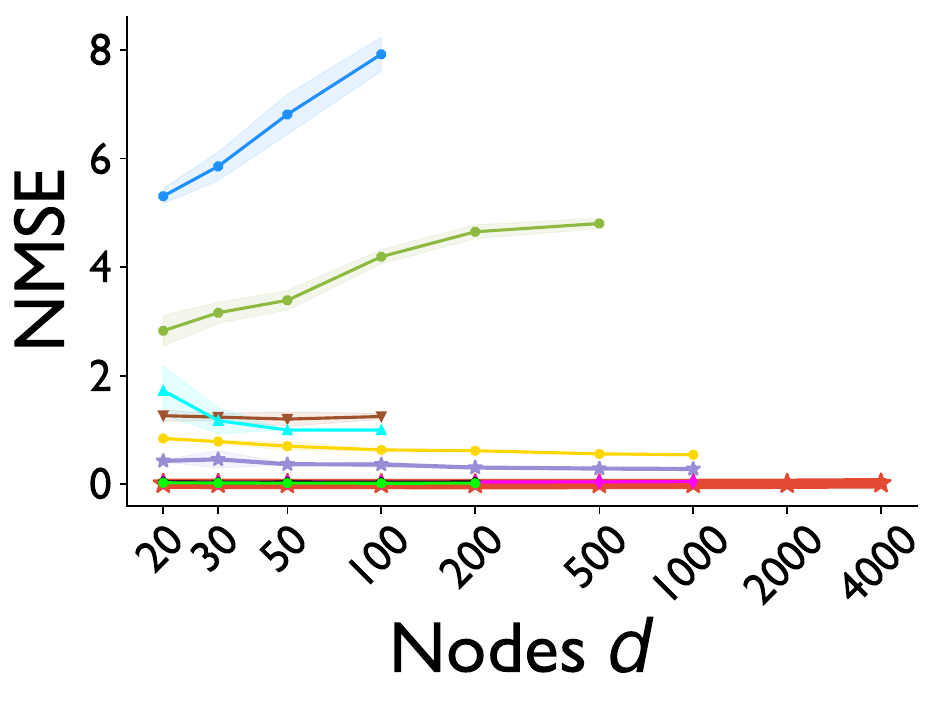}

        \includegraphics[width=\linewidth]{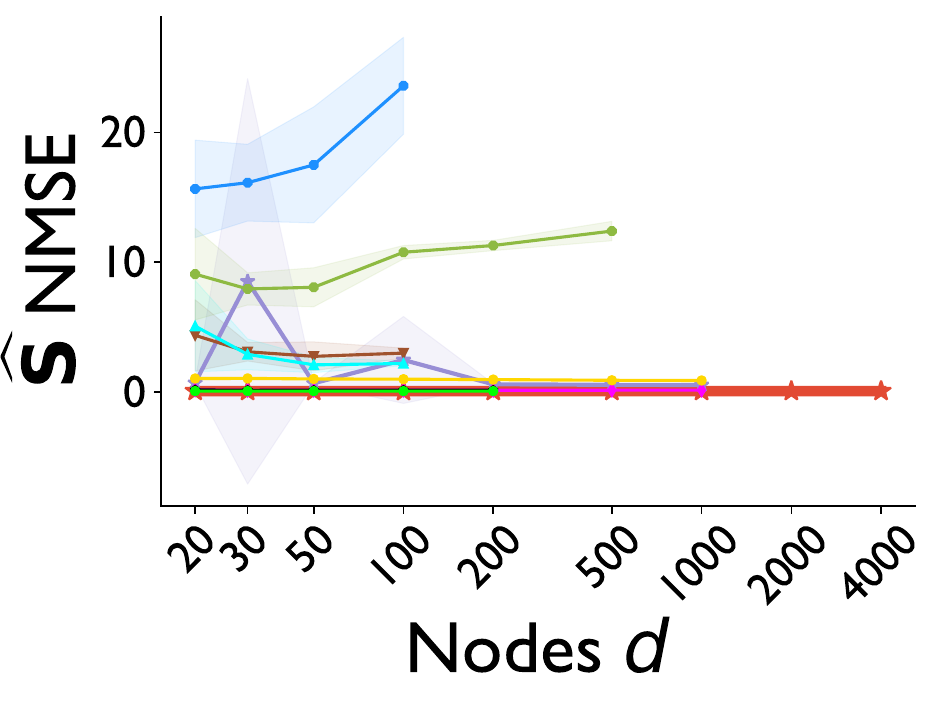}
        \caption{$N=10$, $T=1000$}
    \end{subfigure}
    \hfill
    \begin{subfigure}{0.26\linewidth}
        \centering
        \includegraphics[width=\linewidth]{figures/plot_samples__1,_2,_3,_5,_10,_20__timeout_10000__shd.pdf}

        \includegraphics[width=\linewidth]{figures/plot_samples__1,_2,_3,_5,_10,_20__timeout_10000__time.pdf}
        
        \includegraphics[width=\linewidth]{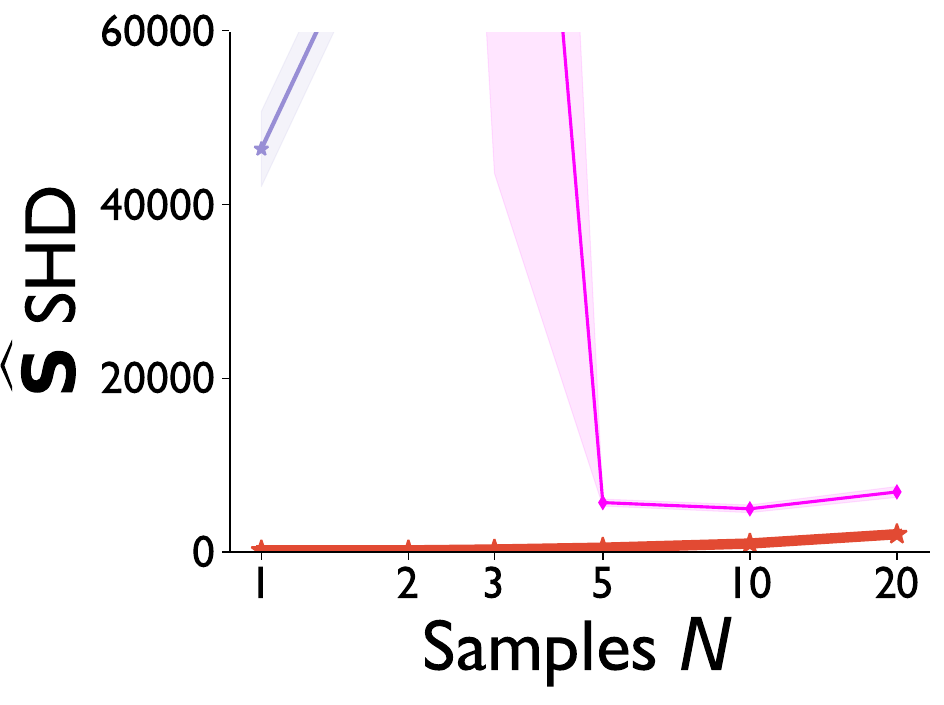}

        \includegraphics[width=\linewidth]{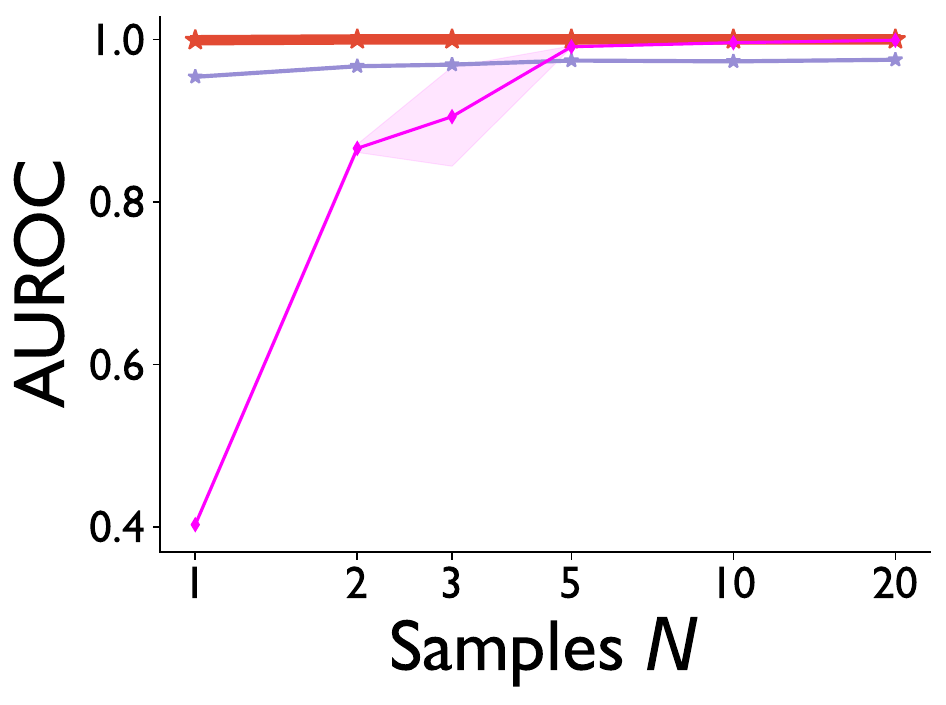}
        
        \includegraphics[width=\linewidth]{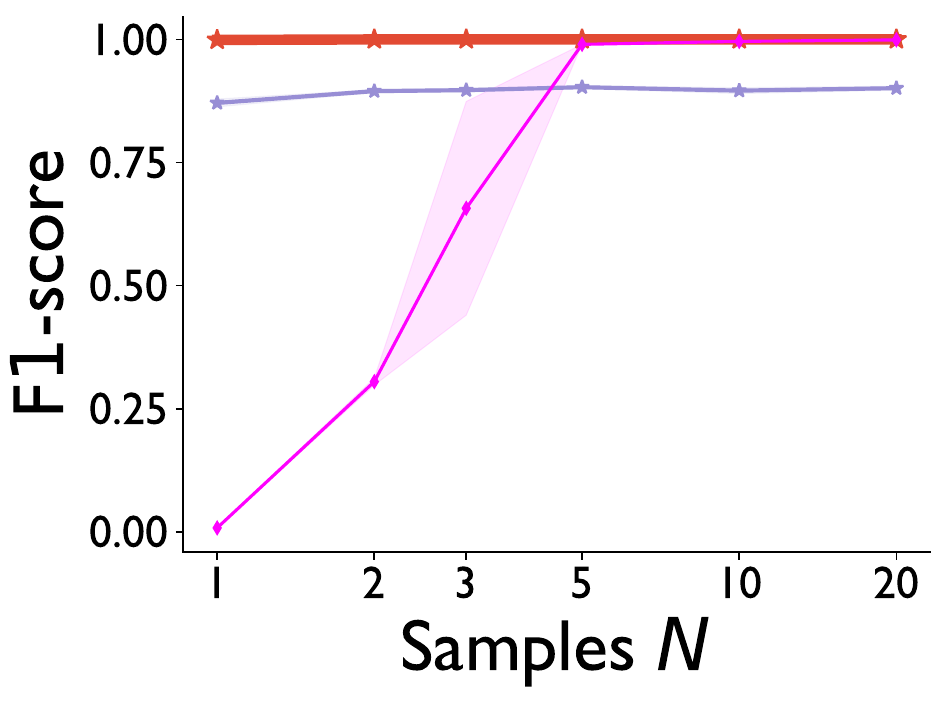}

        \includegraphics[width=\linewidth]{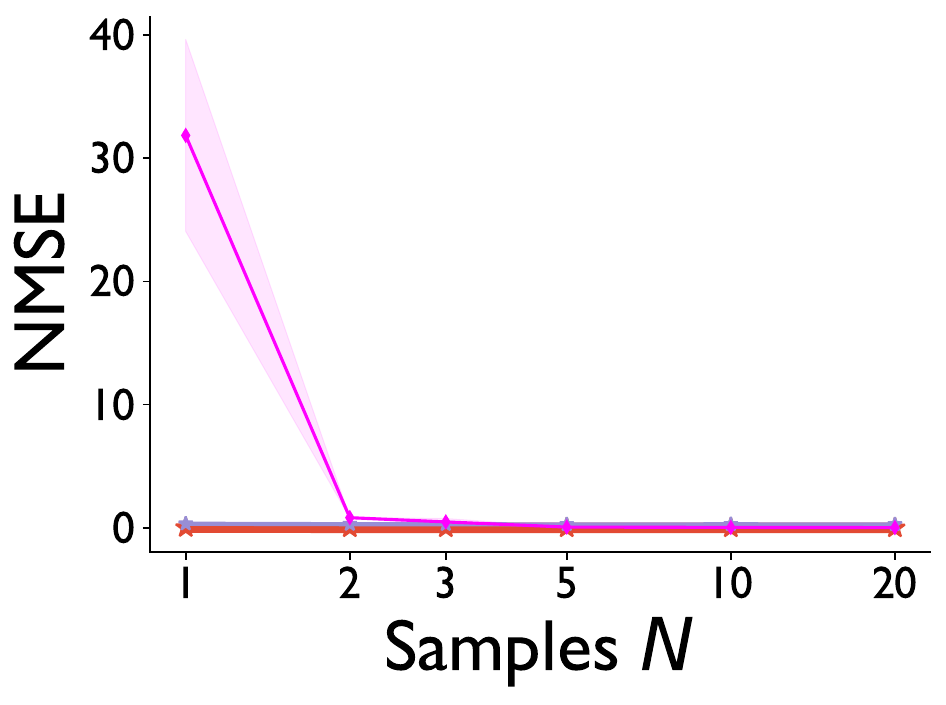}

        \includegraphics[width=\linewidth]{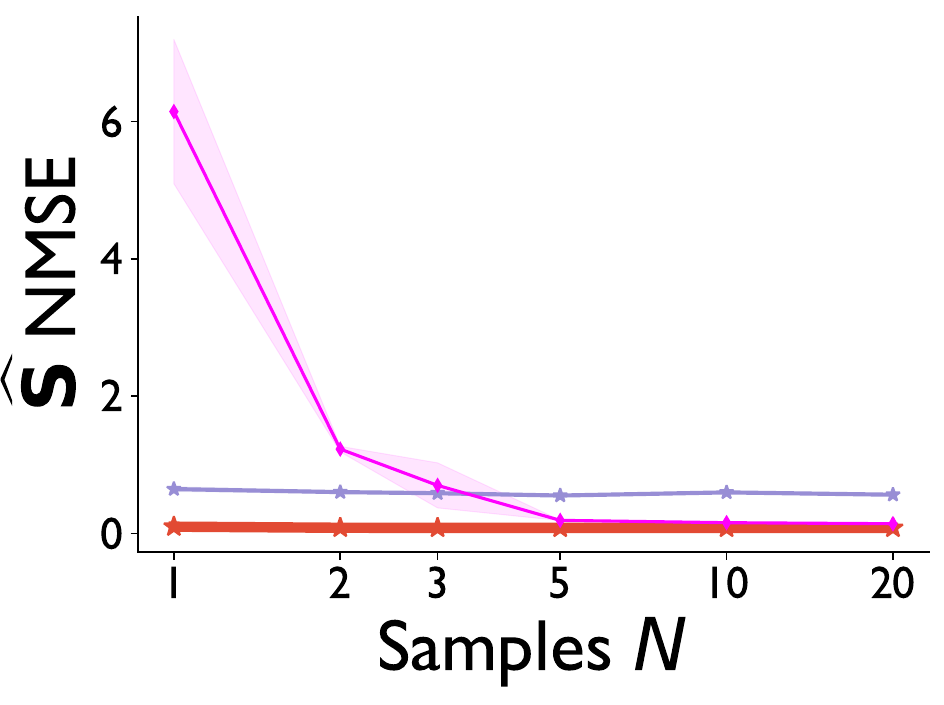}
        \caption{$d=500$, $T=1000$}
    \end{subfigure}
    \begin{subfigure}{0.18\linewidth}
        \hspace{30pt}
        \includegraphics[trim={7cm 0 0 0}, clip, width=\linewidth]{figures/plot_timeout_10000__legend_only.pdf}
        \vspace{300pt}
    \end{subfigure}
    \caption{Performance on synthetic data (Bernoulli distributed input).}
    \label{app:fig:synthetic_plots_bernoulli}
\end{figure}

\subsection{Larger time lag} 
\label{appendix:subsec:large_time_lag}

In Figs.~\ref{fig:appendix:large_lag_laplace},\ref{fig:appendix:large_lag_bernoulli}, we present an experiment with a larger number of time lags, setting $k=5$. This experiment considers $N=10$ samples of time series, each of length $T=1000$, while varying the number of nodes. All other experimental settings remain the same as in the main experiment, except for the weight bounds of $\mW$, which are set to $[0.1,0.2]$. This adjustment is necessary because a larger number of lags requires smaller weights to ensure bounded data, as dictated by Theorem~\ref{appendix:th:stabilitySEM}.
The results are consistent with those in Fig.~\ref{fig:synthetic_plots}, with \mobius performing better than the baselines.

\begin{figure}[H]
    \centering
    \begin{subfigure}[t]{0.26\linewidth}
        \centering
        \includegraphics[width=\linewidth]{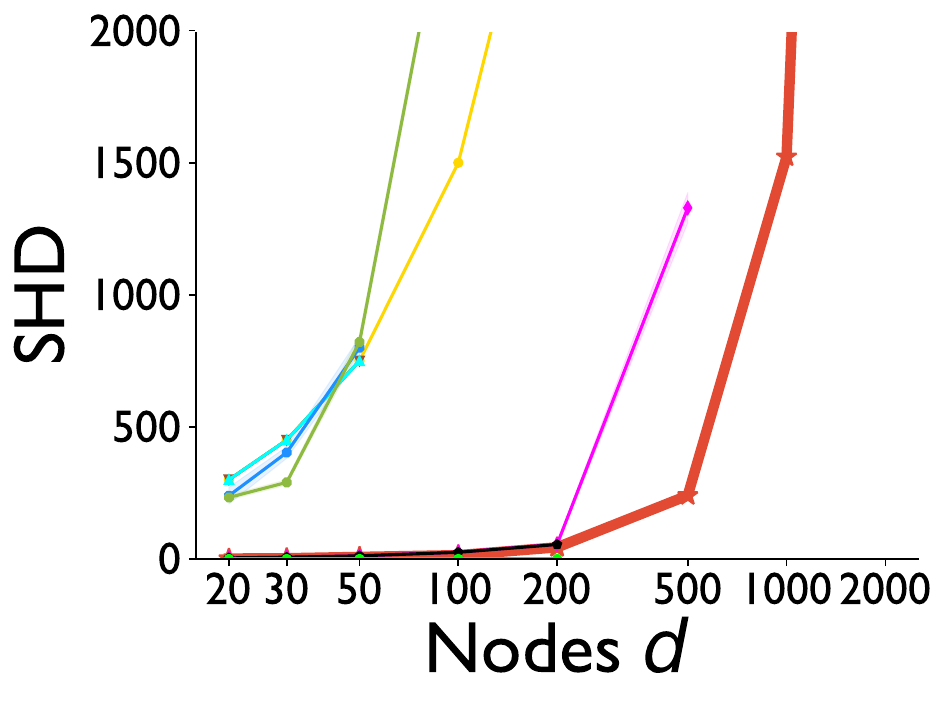}
        
        \includegraphics[width=\linewidth]{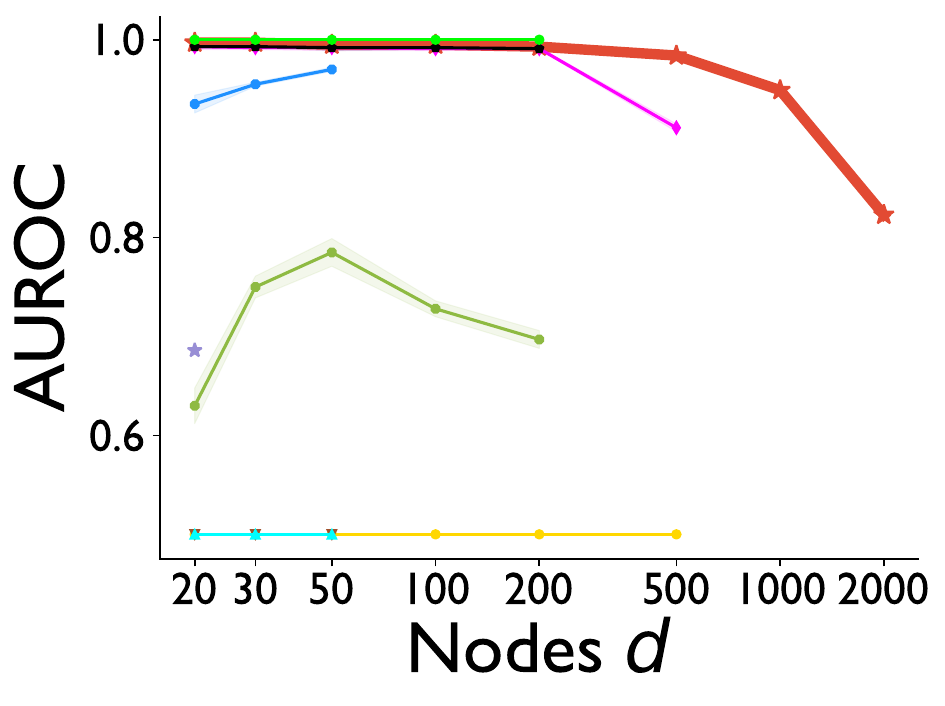}
    \end{subfigure}
    \hfill
    \begin{subfigure}[t]{0.26\linewidth}
        \centering
        \includegraphics[width=\linewidth]{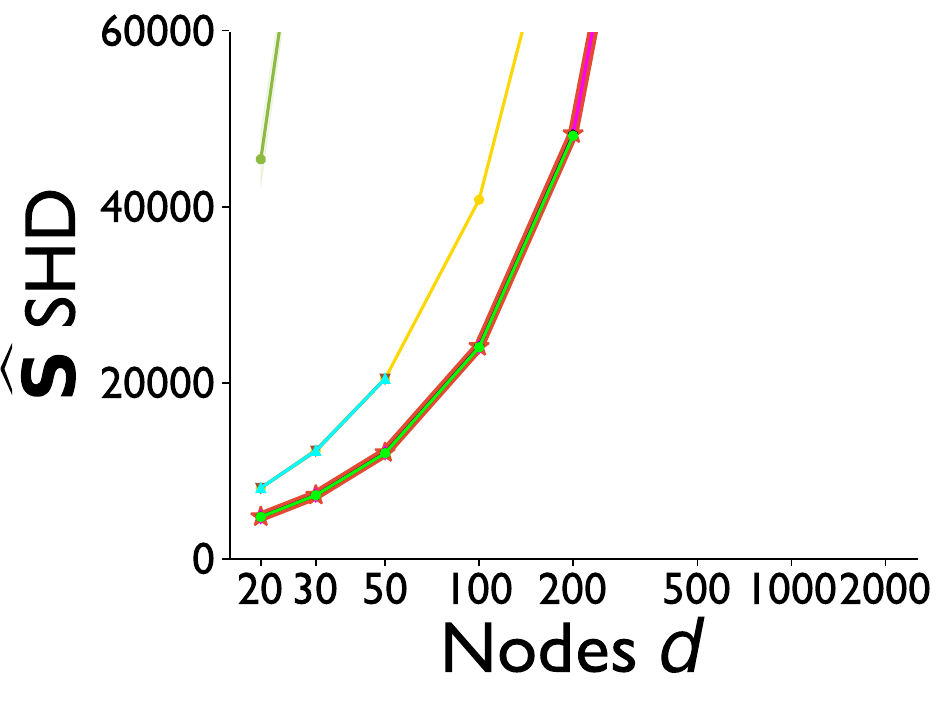}
        
        \includegraphics[width=\linewidth]{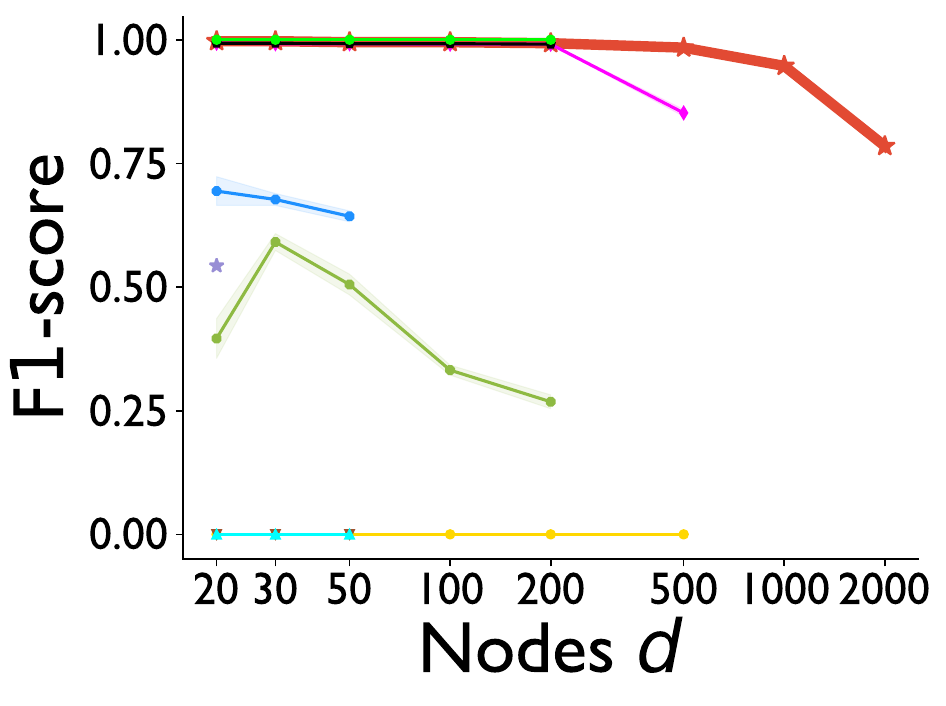}
    \end{subfigure}
    \hfill
    \begin{subfigure}[t]{0.26\linewidth}
        \centering
        \includegraphics[width=\linewidth]{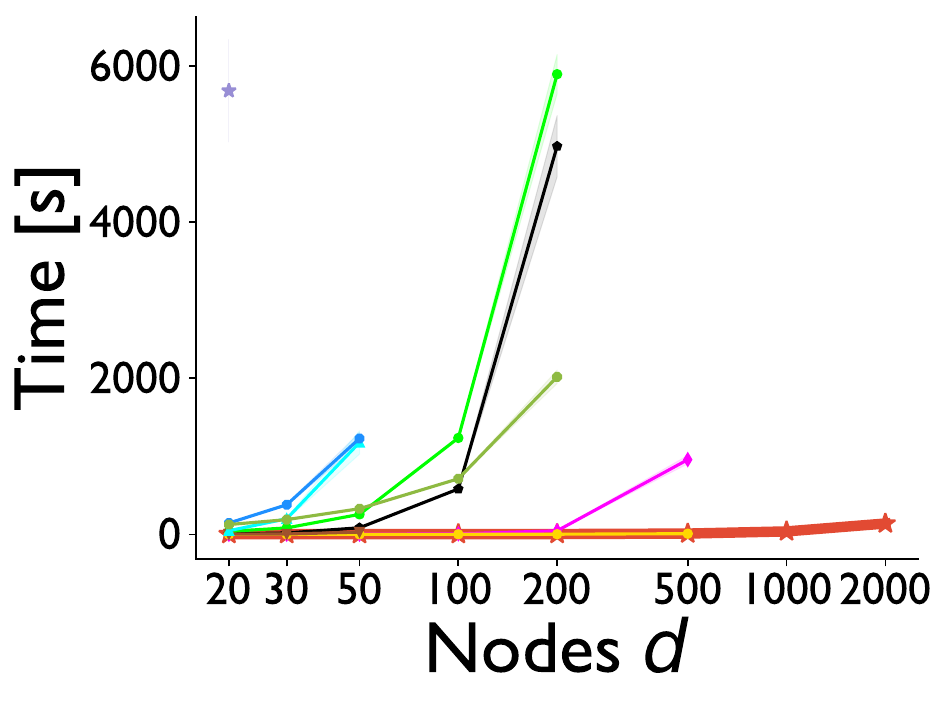}
        
        \includegraphics[width=\linewidth]{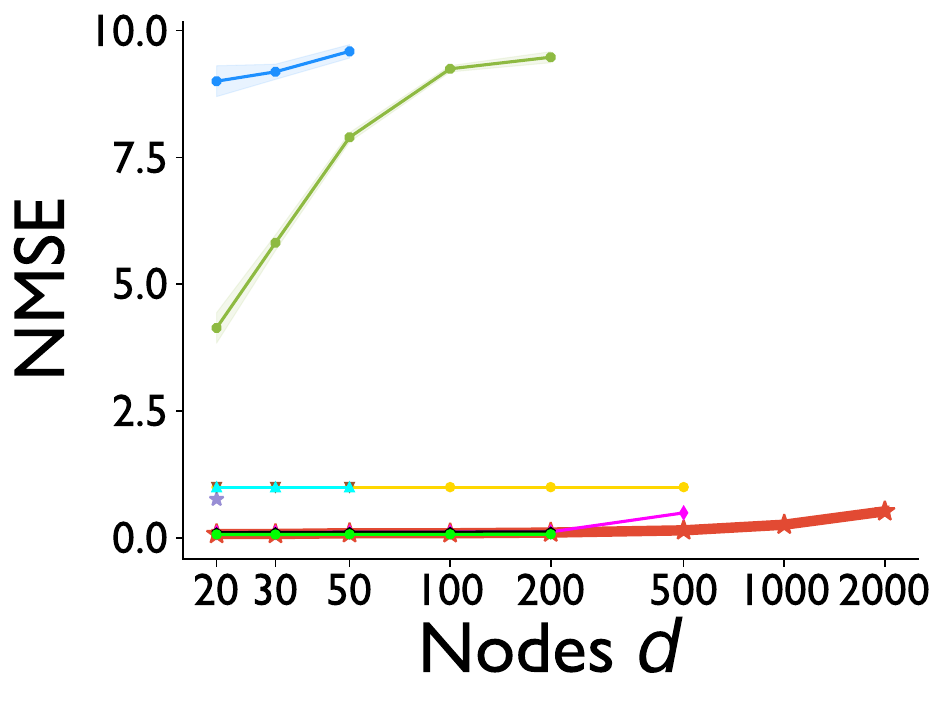}
    \end{subfigure}
    \begin{subfigure}{0.18\linewidth}
        \hspace{30pt}
        \includegraphics[trim={7cm 0 0 0}, clip, width=\linewidth]{figures/plot_timeout_10000__legend_only.pdf}
        \vspace{-50pt}
    \end{subfigure}
    \caption{Synthetic experiment with with larger time lag $k=5$, assuming input with Laplacian distribution. The number of samples is set to $N=10$ and each time series sample has length $T=1000$. The plots show performance for varying number of nodes.}
    \label{fig:appendix:large_lag_laplace}
\end{figure}
\vfill

\begin{figure}[H]
    \centering
    \begin{subfigure}[t]{0.26\linewidth}
        \centering
        \includegraphics[width=\linewidth]{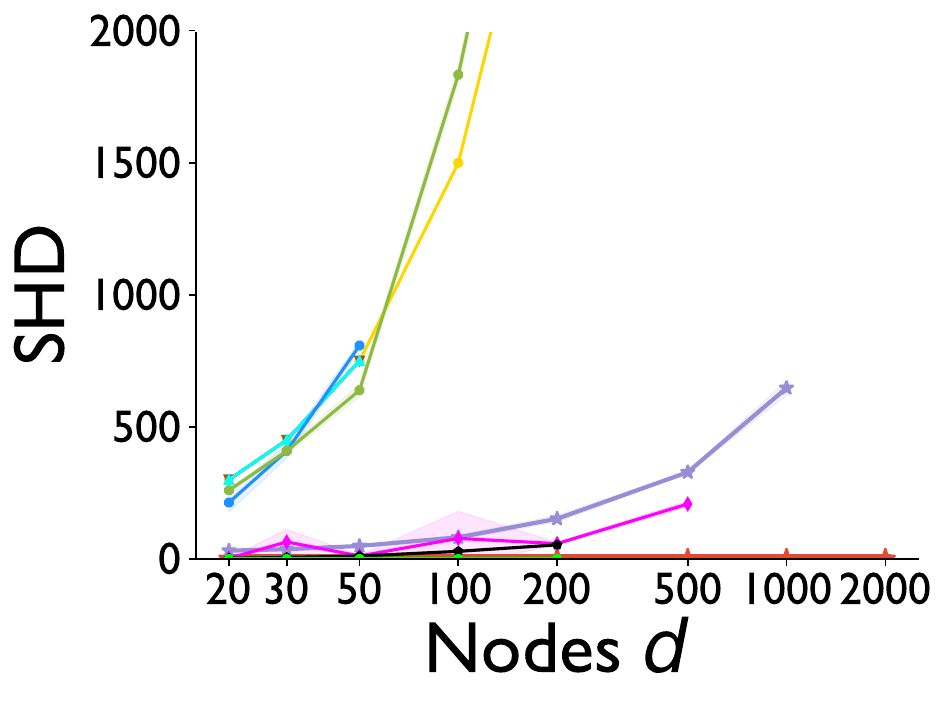}
        
        \includegraphics[width=\linewidth]{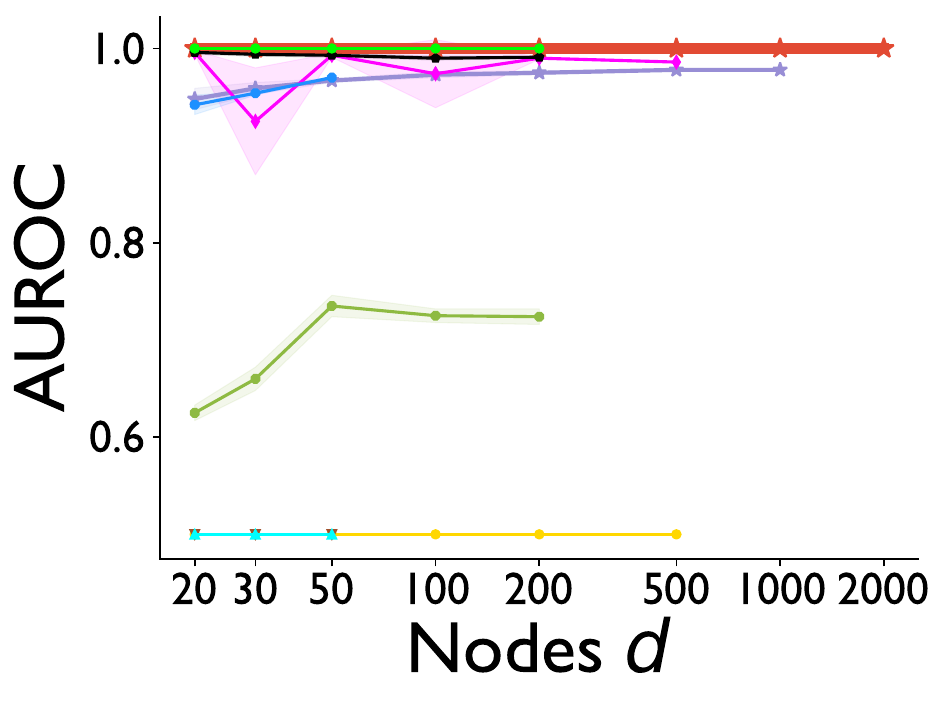}

        \includegraphics[width=\linewidth]{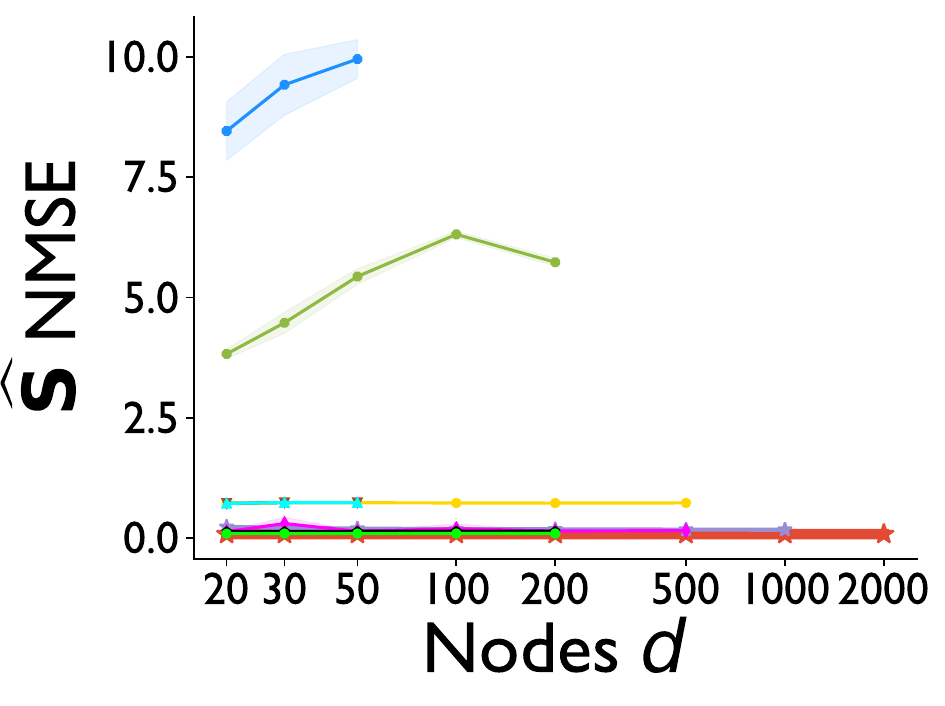}
    \end{subfigure}
    \hfill
    \begin{subfigure}[t]{0.26\linewidth}
        \centering
        \includegraphics[width=\linewidth]{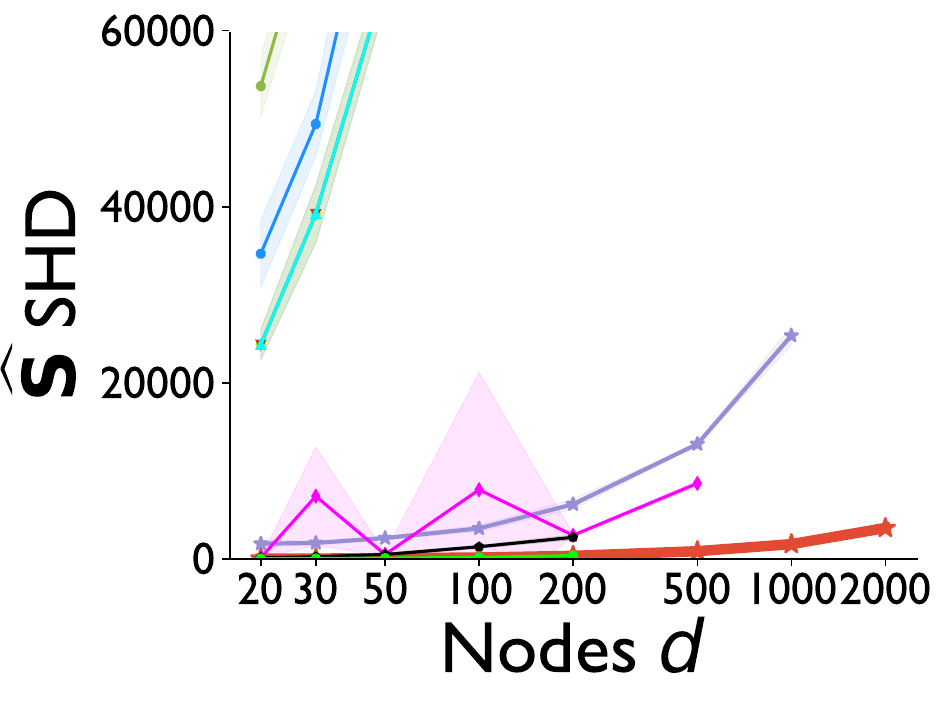}
        
        \includegraphics[width=\linewidth]{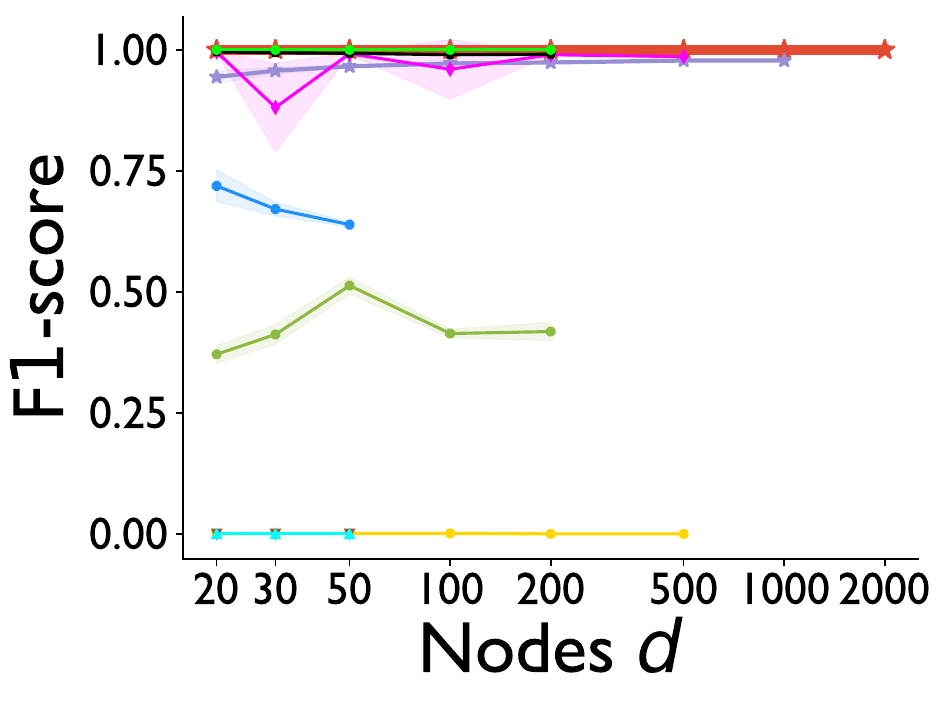}
    \end{subfigure}
    \hfill
    \begin{subfigure}[t]{0.26\linewidth}
        \centering
        \includegraphics[width=\linewidth]{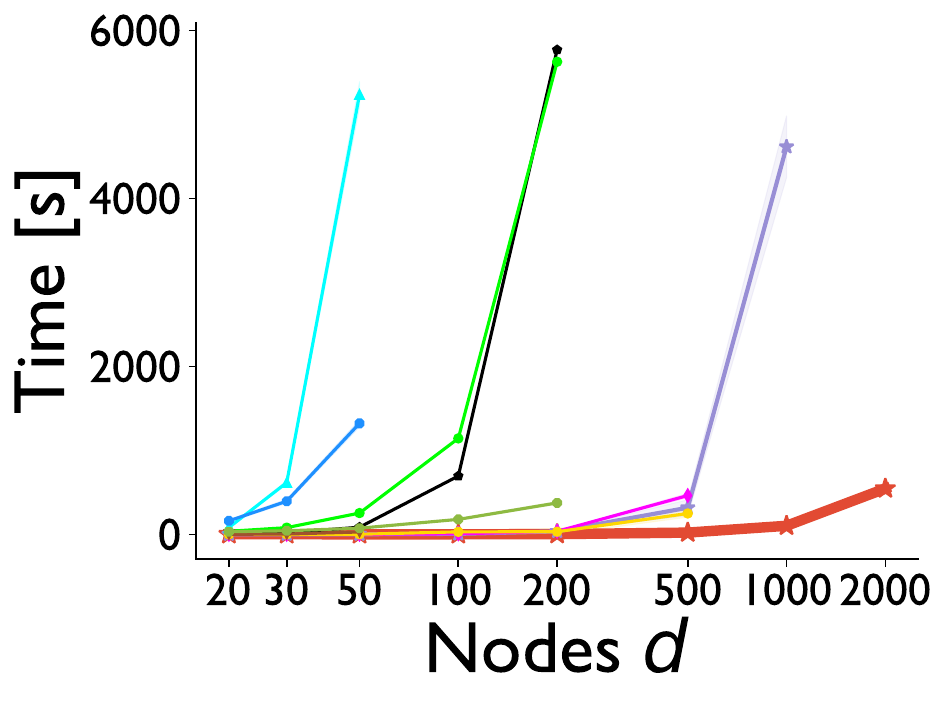}
        
        \includegraphics[width=\linewidth]{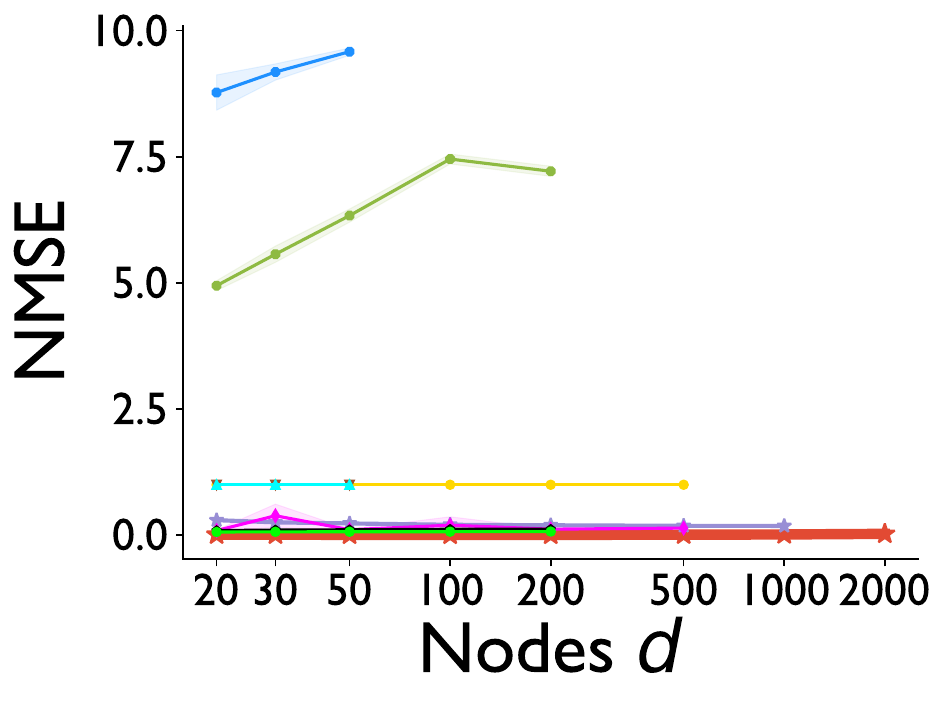}
    \end{subfigure}
    \begin{subfigure}{0.18\linewidth}
        \hspace{30pt}
        \includegraphics[trim={7cm 0 0 0}, clip, width=\linewidth]{figures/plot_timeout_10000__legend_only.pdf}
        \vspace{-50pt}
    \end{subfigure}
    \caption{Synthetic experiment with larger time lag $k=5$, assuming input with Bernoulli distribution.}
    \label{fig:appendix:large_lag_bernoulli}
\end{figure}

\subsection{Sensitivity of time lag}
\label{appendix:subsec:more_time_lags}

We examine the sensitivity of the time lag parameter in the algorithms using the experiment shown in Figs.\ref{fig:appendix:varying_lag_bernoulli},\ref{fig:appendix:varying_lag_laplacian}. This experiment follows standard synthetic settings with $d=1000$, $T=1000$, and a true time lag of $k=3$.

\paragraph{Bernoulli-uniform input} 
When \mobius is provided with a time lag parameter $k' \geq k=3$, its approximation remains optimal. This indicates that, as long as \mobius is given a sufficiently large time lag, it can correctly identify the true maximum time lag $k$ of the system. We observed a similar behavior in our real-world stock market experiment (Fig.~\ref{fig:real_stocks}), where \mobius did not detect any time-lagged dependencies, as expected—the stock market typically reacts almost instantaneously. Conversely, if \mobius is given a time lag $k' < 3$, its performance deteriorates significantly.

SparseRC performs well as long as $k' \geq k$, though its approximation remains worse than that of \mobius. Additionally, SparseRC has a higher execution time and fails to complete (times out) when $k' = 6$. \varlingam performs reasonably when provided with the exact time lag $k$, but also times out when $k' > 3$.

Other baseline methods either timed out or exhibited poor performance.

\begin{figure}[H]
    \centering
    \begin{subfigure}[t]{0.26\linewidth}
        \centering
        \includegraphics[width=\linewidth]{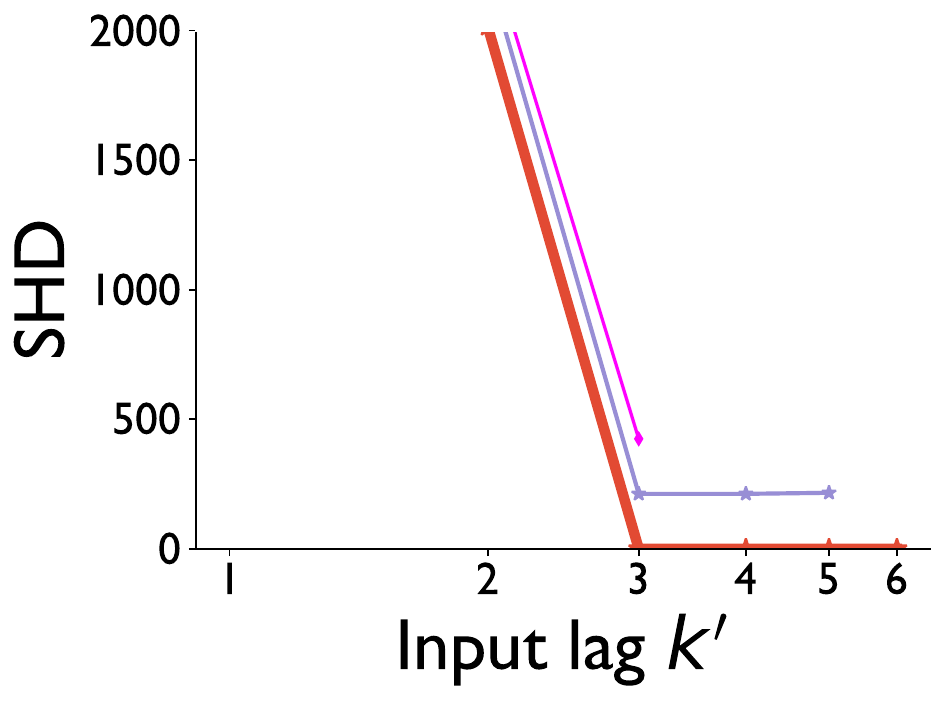}
        
        \includegraphics[width=\linewidth]{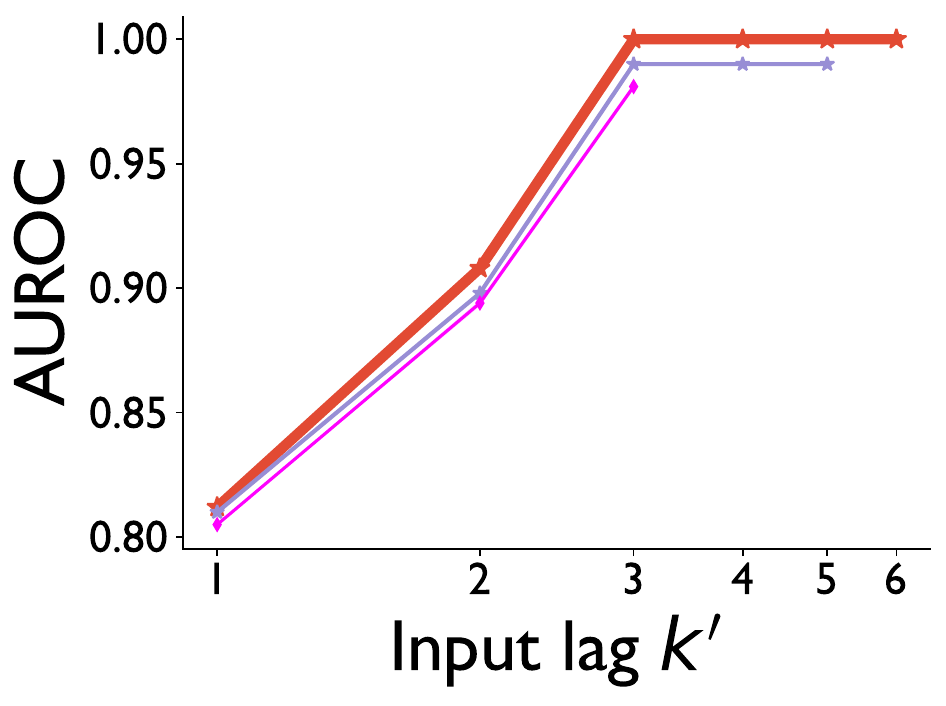}

        \includegraphics[width=\linewidth]{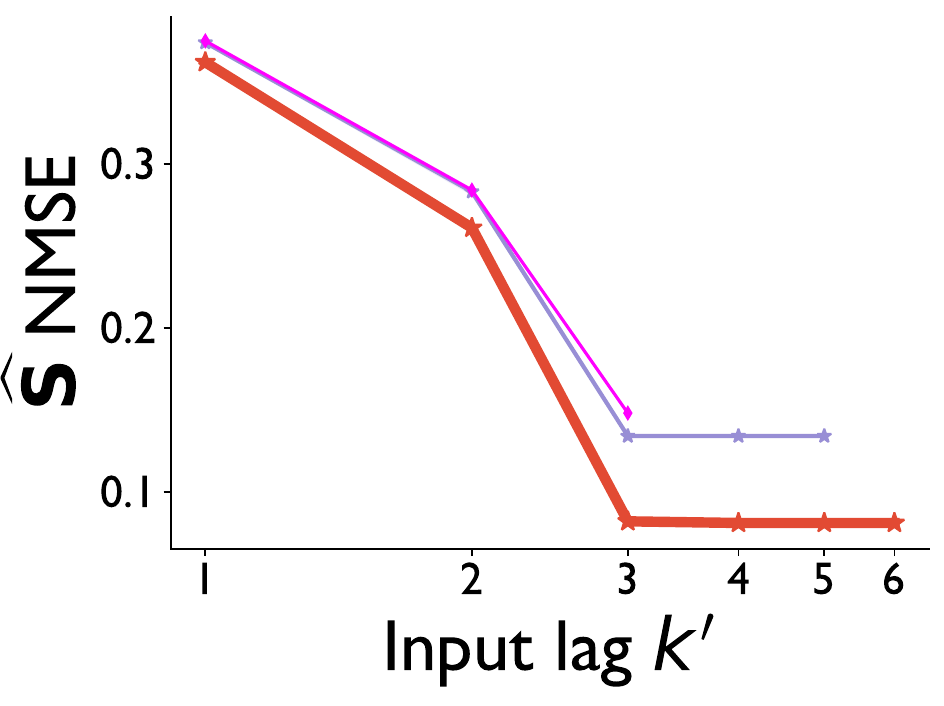}
    \end{subfigure}
    \hfill
    \begin{subfigure}[t]{0.26\linewidth}
        \centering
        \includegraphics[width=\linewidth]{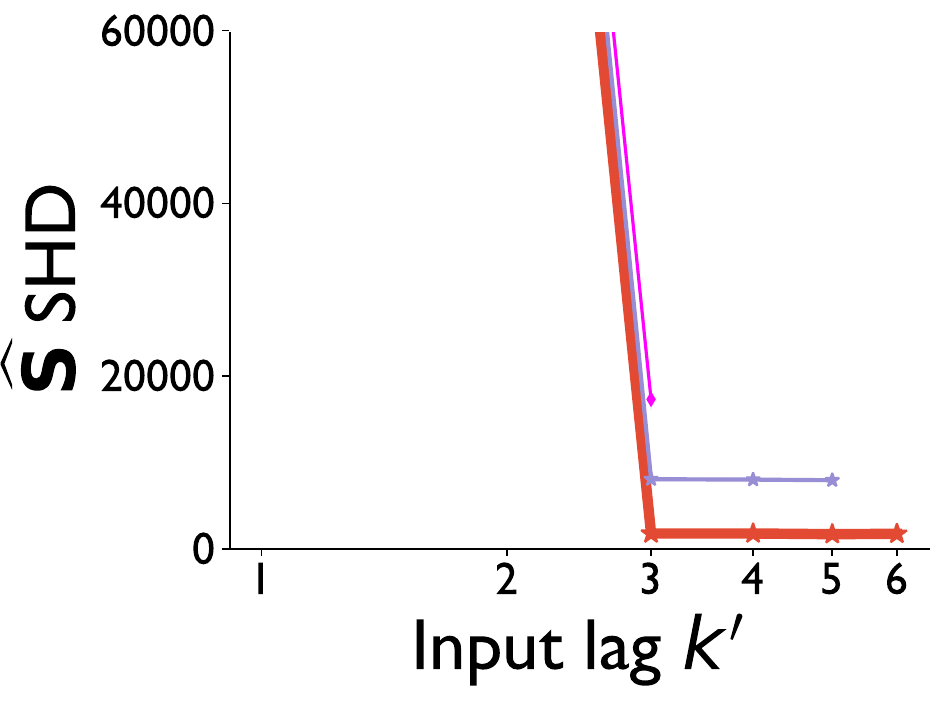}
        
        \includegraphics[width=\linewidth]{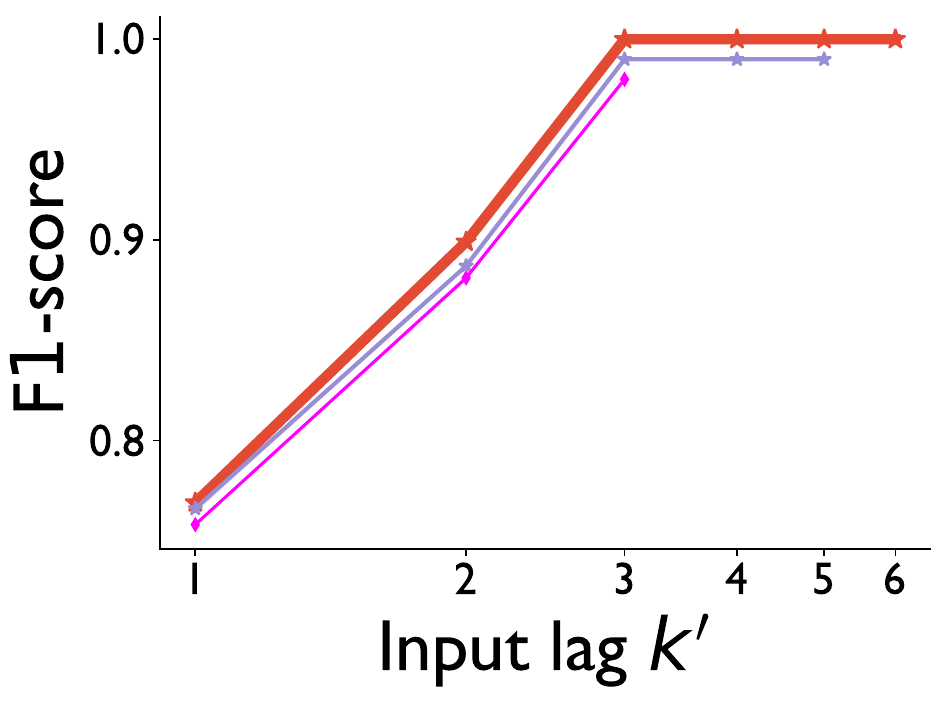}
    \end{subfigure}
    \hfill
    \begin{subfigure}[t]{0.26\linewidth}
        \centering
        \includegraphics[width=\linewidth]{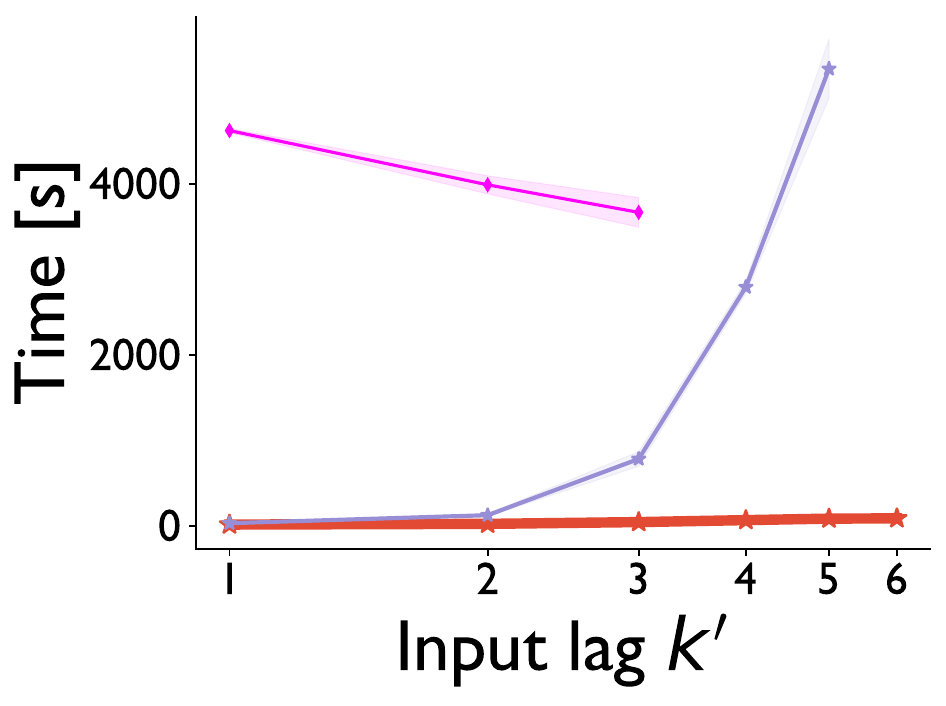}
        
        \includegraphics[width=\linewidth]{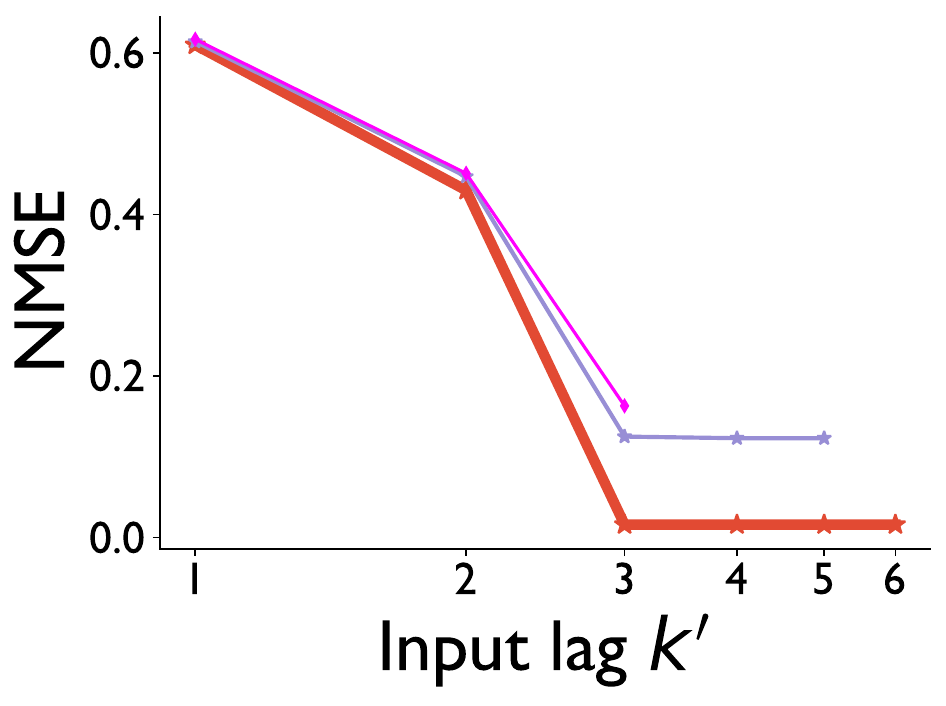}
    \end{subfigure}
    \begin{subfigure}{0.18\linewidth}
        \hspace{30pt}
        \includegraphics[trim={7cm 0 0 0}, clip, width=\linewidth]{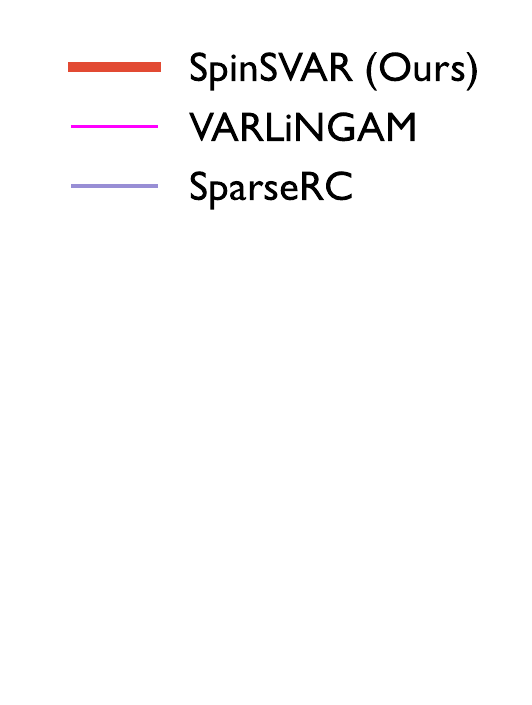}
        \vspace{-50pt}
    \end{subfigure}
    \caption{Evaluating the sensitivity of the time lag $k$ in synthetic settings with original $k=3$, $d=1000$ nodes, $T=1000$ and $N=10$ samples and Bernoulli-uniform input. The algorithms have varying time lag from $1$ to $6$.}
    \label{fig:appendix:varying_lag_bernoulli}
\end{figure}
\paragraph{Laplacian input}
In this setting, the behavior differs slightly. When $k'\geq k$, \mobius continues to perform well but is unable to recover the exact ground truth. This limitation explains the failure of the $\widehat{\tS}$ metric. Nevertheless, \mobius still outperforms the baseline methods. Notably, \varlingam times out in this scenario.

\vfill
\newpage

\begin{figure}[H]
    \centering
    \begin{subfigure}[t]{0.26\linewidth}
        \centering
        \includegraphics[width=\linewidth]{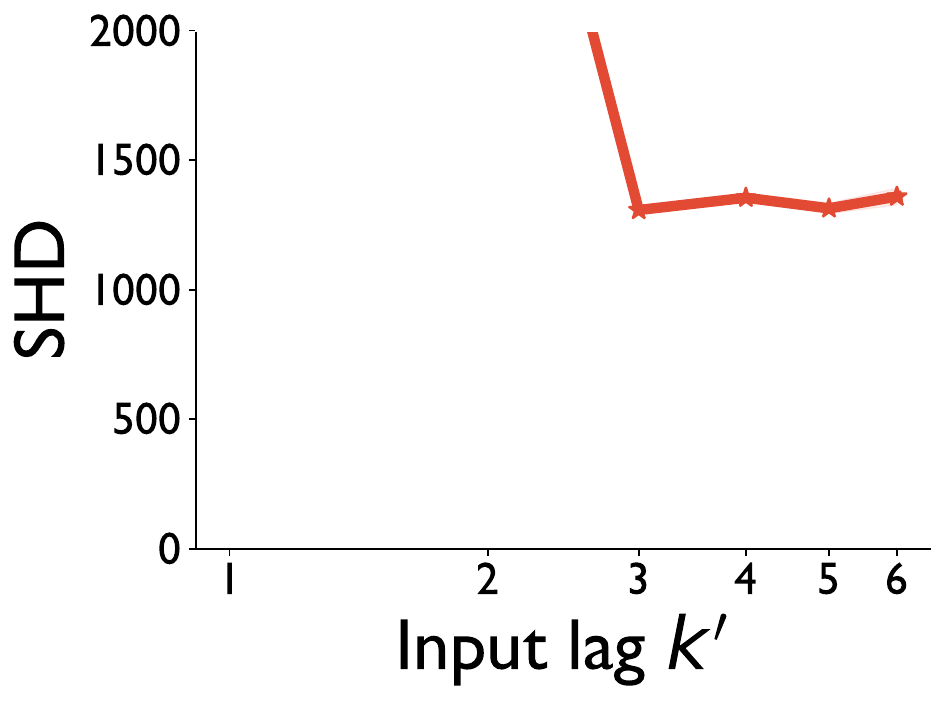}

        \includegraphics[width=\linewidth]{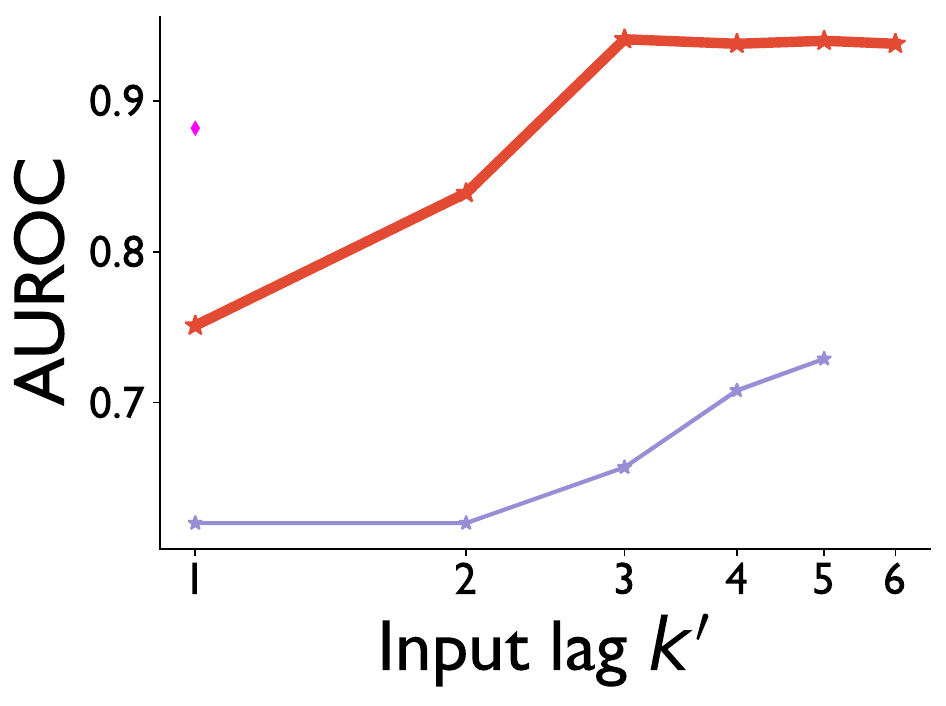}
    \end{subfigure}
    \hfill
    \begin{subfigure}[t]{0.26\linewidth}
        \centering
        \includegraphics[width=\linewidth]{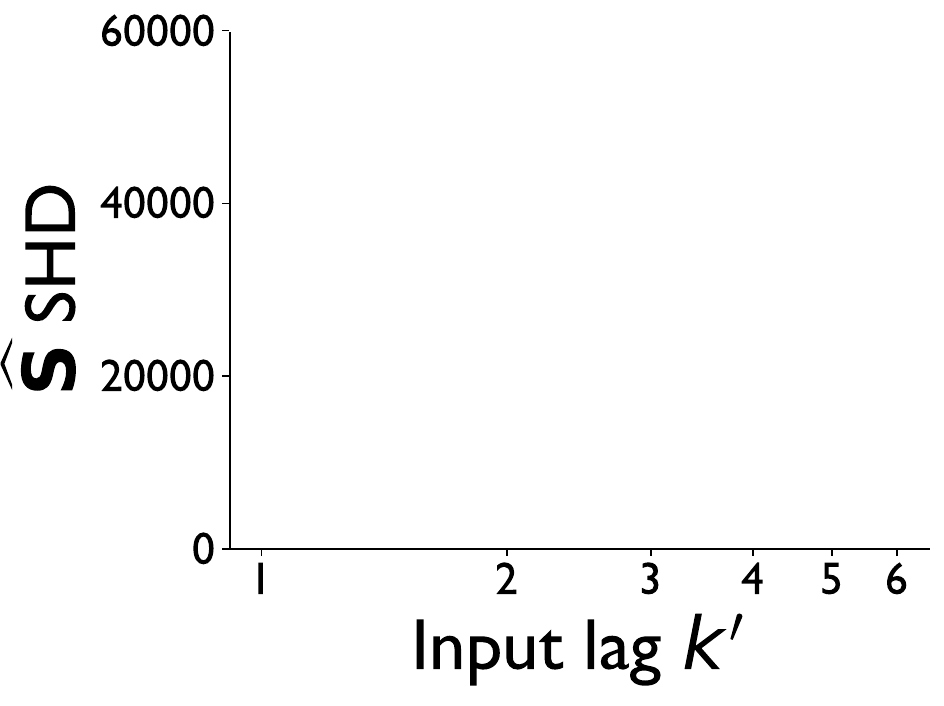}

        \includegraphics[width=\linewidth]{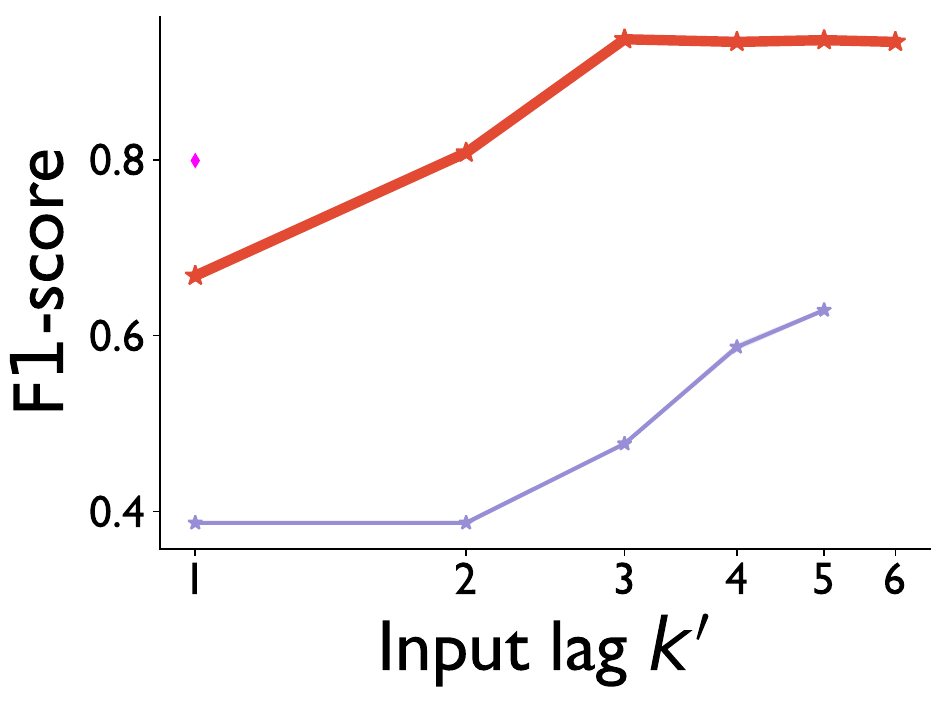}
    \end{subfigure}
    \hfill
    \begin{subfigure}[t]{0.26\linewidth}
        \centering
        \includegraphics[width=\linewidth]{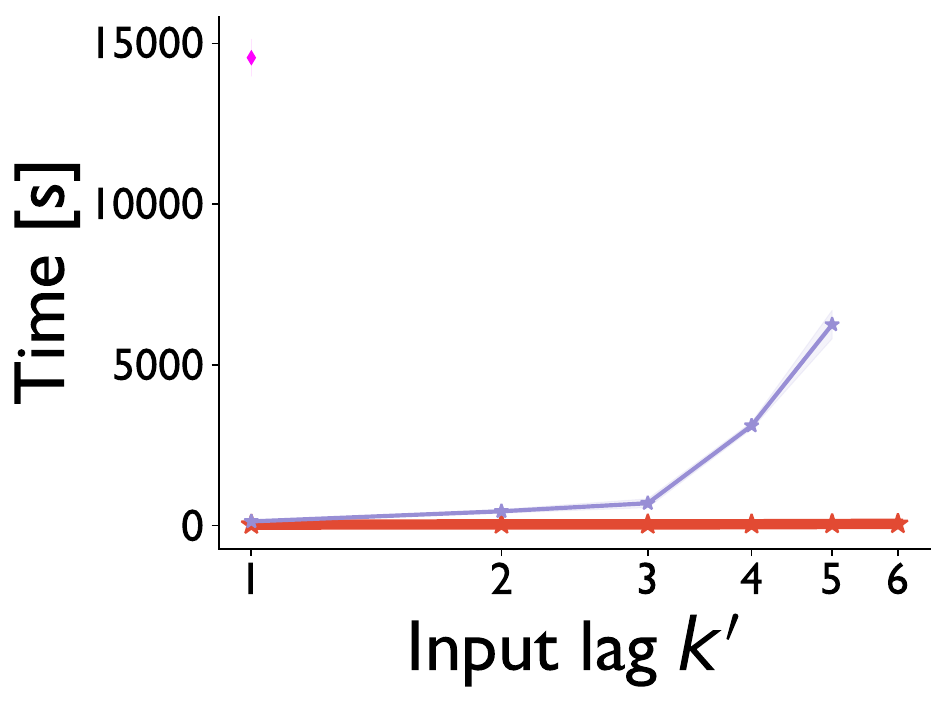}

        \includegraphics[width=\linewidth]{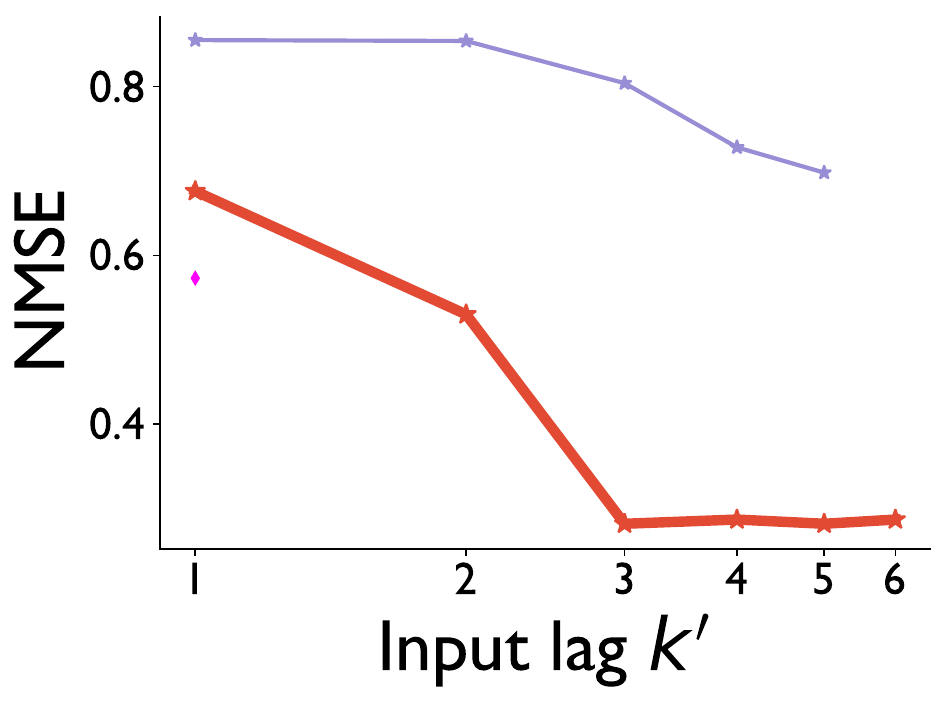}
        \end{subfigure}
    \begin{subfigure}{0.18\linewidth}
        \hspace{30pt}
        \includegraphics[trim={7cm 0 0 0}, clip, width=\linewidth]{figures/plot_weight_bounds__0.1,_0.2__number_of_lags_3_timeout_10000_algo_lags_6__legend_only.pdf}
        \vspace{-50pt}
    \end{subfigure}
    \caption{Evaluating the sensitivity of the time lag $k$ in synthetic settings with original $k=3$, $d=1000$ nodes, $T=1000$ and $N=10$ samples and Laplacian input. The algorithms have varying time lag from $1$ to $6$.}
    \label{fig:appendix:varying_lag_laplacian}
\end{figure}

\subsection{Larger DAGs}
\label{appendix:subsec:larger_DAGs}
Here we include the time-outs of \varlingam for $d=2000$ and the performance of SparseRC which is poor compared to \varlingam and \mobius.

\begin{table}[H]
    \centering
    \caption{SHD report for large DAGs ($T = 1000$).}
    \begin{tabular}{@{}llllll@{}}
    \toprule
      \mobius \hfill $N=$ & $1$ & $2$ & $4$ & $8$ & $16$\\
    \midrule
        $d=1000$, $\tS\sim$ Laplace & $8.3k$ & $1k$ & $371$ & $112$ & $\boldsymbol{27}$ \\
    $d=1000$, $\tS\sim$ Bernoulli & $2$ & $\boldsymbol{0}$ & $\boldsymbol{0}$ & $\boldsymbol{0}$ & $\boldsymbol{0}$ \\
    $d=2000$, $\tS\sim$ Laplace & $18k$ & $17k$ & $2.1k$ & $645$ & $183$ \\
    $d=2000$, $\tS\sim$ Bernoulli  & $12$ & $\boldsymbol{0}$ & $\boldsymbol{0}$ & $\boldsymbol{0}$ & $\boldsymbol{0}$  \\
    $d=4000$, $\tS\sim$ Laplace & $36k$ & $36k$ & $33k$ & $4.5k$ & $1.2k$ \\
    $d=4000$, $\tS\sim$ Bernoulli & $164$ & $27$ & $15$ & $\boldsymbol{7}$ & $\boldsymbol{9}$  \\
    \midrule
    \varlingam  \hfill $N=$& $1$ & $2$ & $4$ & $8$ & $16$\\
    \midrule
    $d=1000$, $\tS\sim$ Laplace  & $-$ & $-$ & $-$ & $-$ & $-$ \\
    $d=1000$, $\tS\sim$ Bernoulli  & $-$ & $-$ & $-$ & $115$ & $29$ \\
    $d=2000$, $\tS\sim$ Laplace  & $-$ & $-$ & $-$ & $-$ & $-$ \\
    $d=2000$, $\tS\sim$ Bernoulli  & $-$ & $-$ & $-$ & $-$ & $-$ \\
    \midrule
    SparseRC  \hfill $N=$& $1$ & $2$ & $4$ & $8$ & $16$\\
    \midrule
    $d=1000$, $\tS\sim$ Laplace  & $3.3k$ & $2.2k$ & $2k$ & $1.8k$ & $1.8k$ \\
    $d=1000$, $\tS\sim$ Bernoulli  & $2.6k$ & $1.7k$ & $1.6k$ & $1.6k$ & $1.7k$ \\
    $d=2000$, $\tS\sim$ Laplace  & $-$ & $-$ & $-$ & $-$ & $-$ \\
    $d=2000$, $\tS\sim$ Bernoulli  & $-$ & $-$ & $-$ & $-$ & $-$ \\
    \bottomrule
    \end{tabular}
\end{table}

\subsection{Simulated financial portfolios}
\label{appendix:exp:simulated}

We evaluate our method on simulated financial time-series data from~\citet{kleinberg2013finance}, generated using the Fama-French three-factor model~\citep{fama1970efficient} (volatility, size, and value). The return $\evx_{i,t}$ of stock $i$ at time $t$ is computed as $\evx_{t,i} = \sum_{j} b_{ij}f_{t,i} + \epsilon_{t,i}$,
where $f_{t,i}$ are the three factors, $b_{ij}$ are their corresponding weights and $\epsilon_{t,i}$ are (correlated) idiosyncratic terms. We use 16 datasets from this benchmark, each incorporating time lags up to $k=3$. The data consists of daily returns for $d=25$ stocks, with ground truth DAGs containing an average of 22 edges. Each dataset provides a multivariate time series $\mX$ with $4000$ time steps, which we segment into non-overlapping windows of $50$ time steps, yielding a dataset $\tX$ of shape $80 \times 50 \times 25$.

Table~\ref{appendix:tab:FinanceCPT} reports the SHD and runtime for each method.  Since the true structural shocks are unknown, we do not evaluate them in this setting. 
Hyperparameters were selected via grid search, as detailed in Appendix~\ref{appendix:subsubsec:hyper_simulated}. The best-performing methods are \mobius and SparseRC, suggesting that assuming a sparse set of structural shocks is valid for financial data. SparseRC slightly outperforms \mobius, likely due to the dataset’s small scale—both in terms of time lags and number of nodes—though it remains slower. The fastest method, \varlingam, exhibits weaker performance. The other baselines perform poorly in this dataset.

\begin{table}[t]
    \footnotesize
        \centering
        \caption{Performance on the simulated financial dataset~\citep{kleinberg2013finance}.}
        \begin{tabular}{@{}lll@{}}
        \toprule
        Method  &  SHD ($\downarrow$)  & Time [s] \\
        \midrule
        \mobius (Ours) &  $    12.89\pm7.87 $  &  $    5.43\pm0.65 $  \\ 
        SparseRC &  $\bm{9.92\pm8.22}$  &  $    9.74\pm1.21 $  \\ 
        \varlingam &  $    19.25\pm10.64 $  &  $\bm{1.64\pm0.10}$  \\ 
        \dlingam &  $    15.31\pm9.38 $  &  $    4.85\pm0.31 $  \\ 
        \clingam &  $    15.22\pm8.44 $  &  $    12.88\pm0.42 $  \\ 
        TCDF &  $    19.06\pm10.18 $  &  $    33.56\pm1.01 $  \\ 
         DYNOTEARS &  $    33.92\pm9.09 $  &  $    112.91\pm29.59 $  \\ 
        NTS-NOTEARS &  $    57.83\pm37.22 $  &  $    16.40\pm14.45 $  \\ 
        tsFCI &  $21.94\pm9.52$  &  $    17.50\pm12.82 $  \\ 
        PCMCI &  $    361.69\pm67.80 $  &  $16.23\pm4.69$  \\ 
        \bottomrule
        \end{tabular}
        \label{appendix:tab:FinanceCPT}
    \end{table}

\subsection{Dream3 Challenge dataset}
\label{appendix:exp:dream}

\begin{table}[h]
\footnotesize
    \centering
    \caption{AUROC report on the Dream3 challenge dataset \citep{marbach2009dream3,prill2010dream3}. The methods are partitioned into non-linear and linear for a fair comparison. Best performances are marked with bold. }
    \vspace{5pt}
    \begin{tabular}{@{}lllllll@{}}
    \toprule
     & Model &  E.coli-1  & E.coli-2 &  Yeast-1  & Yeast-2  & Yeast-3 \\
    \midrule
    \multirow{6}{1.5cm}{Non-linear}& MLP      & $0.644$   & $0.568$   & $0.585$   & $0.506$   & $0.528$\\  
    & LSTM     & $0.629$   & $0.609$   & $0.579$   & $0.519$   & $0.555$\\  
    & TCDF     & $0.614$   & $0.647$   & $0.581$   & $0.556$   & $\textbf{0.557}$\\  
    & SRU      & $0.657$   & $\textbf{0.666}$   & $0.617$   & $\textbf{0.575}$   & $0.55$\\  
    & eSRU     & $\textbf{0.66}$    & $0.629$   & $\textbf{0.627}$   & $0.557$   & $0.55$\\ 
    & PCMCI  & $0.594$    & $0.545$   & $0.498$   & $0.491$   & $0.508$\\ 
    & NTS-NOTEARS  & $0.592$ & $0.471$  & $0.551$   & $0.551$   & $0.507$\\    
    & tsFCI & $0.5$   & $0.5$   & $0.5$   & $0.5$   & $0.5$\\ 
    \hline \\
    \multirow{5}{1cm}{Linear} & \textbf{\mobius (Ours)} & $0.547$   & $0.525$   & $0.551$   & $0.508$   & $\textbf{0.513}$\\  
    & SparseRC            & $0.543$   & $0.516$   & $\textbf{0.554}$   & $0.507$   & $0.512$\\  
    & VARLiNGAM           & $0.545$   & $0.519$   & $0.516$   & $0.509$   & $0.502$\\  
    & Directed VARLiNGAM  & $0.504$ & $0.501$  & $0.514$   & $0.501$   & $0.510$\\  
    & DYNOTEARS           & $\textbf{0.590}$ & $\textbf{0.547}$  & $0.527$   & $\textbf{0.526}$   & $0.510$\\    
    \bottomrule
    \end{tabular}
    \label{tab:dream3}
\end{table}

In Table~\ref{tab:dream3} we report the AUROC performance of our method compared to baselines. There, Component-wise MLP and LSTM are from \citep{tank2021neuralGranger} and SRU and eSRU from \citep{khanna2019eSRU}. while the rest of the methods are present in the main paper. The results of the first 5 rows are taken from \citep{khanna2019eSRU} and DYNOTEARS from \citep{gong2022rhino}. The methods are partitioned into non-linear and linear for a fair comparison.

Our method is competitive to other linear-model baselines but worse than those assuming a nonlinear model. Apparently, one of the two assumptions, either the sparse SVAR input assumption or linearity of the data generation does not hold in this dataset and our method might not be the most appropriate.

\subsection{S\&P 500 real experiment}
\label{appendix:exp:real}

In Figs.~\ref{fig:appendix:real1} and \ref{fig:appendix:real2} we show the performance of SparseRC, \varlingam, TCDF and PCMCI on the S\&P 500 stock market index. As also mentioned in the main text, SparseRC approximates a DAG similar to \mobius. 
This is due to the few structural shock assumption that both methods use.
\begin{itemize}
    \item \textbf{\varlingam} seems to identify significant edges for any random stock combination, thus producing a poor result. 
    Also, the approximated structural shocks $\widehat{\mS}$ are less expressive than ours in the sense that out of the $4507$ discovered structural shocks only $33.7\%$ of them align with the data changes.
    \item \textbf{TCDF} produces a very sparse DAG with not enough information. 
    \item \textbf{PCMCI} outputs a zero graph for time lag $0$ and a not well-structured graph for time lag $1$. 
    As a consequence, we don't see a meaningful pattern in the structural shocks. 
    \item \textbf{DYNOTEARS} had as output an empty graph and thus its performance is not reported. Regarding its hyperparameters, we minimized the weight threshold up to $0$ (all weights included as edges) and we tried both $λ_w = λ_a = 0.01$ and $k=2$, which were the optimal from our synthetic experiments and $λ_w = λ_a = 0.1$ which is the reported best in the S\&P 100 experiment in~\citep{pamfil2020dynotears}.
    \item \dlingam, \clingam, tsFCI and NTS-NOTEARS had time-out in this experiment.
\end{itemize}
\vfill
\newpage

\begin{figure}[H]
    \begin{subfigure}{0.45\linewidth}
        \centering
        \includegraphics[width=\linewidth]{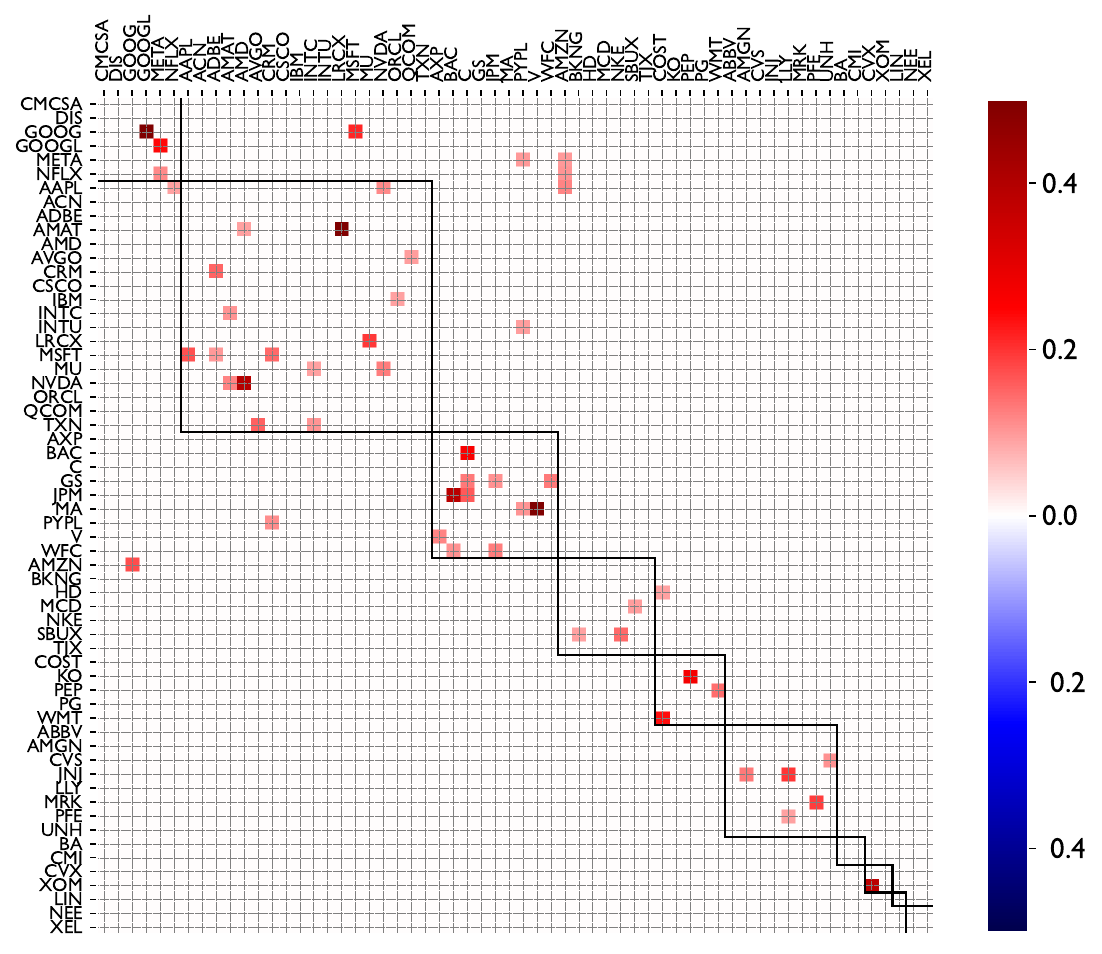}
        \caption{SparseRC estimate for $\widehat{\mB}_0$}
        \label{fig:stocks_lag0:appendix:sparserc}
    \end{subfigure}
    \begin{subfigure}{0.45\linewidth}
        \centering
        \includegraphics[width=1.25\linewidth]{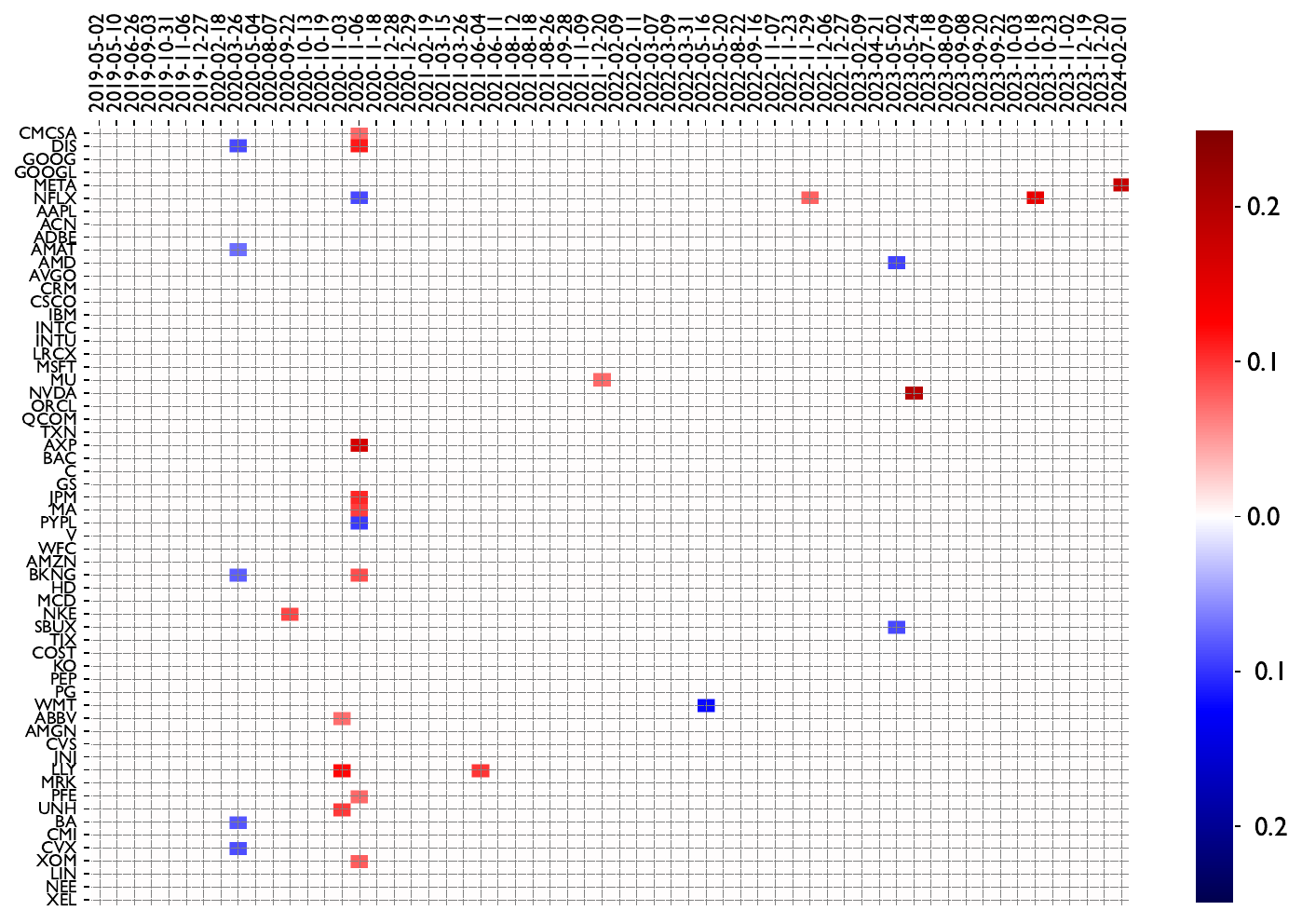}
        \caption{SparseRC estimate for $\widehat{\mS}$}
        \label{fig:stocks_rootcauses:appendix:sparserc}
    \end{subfigure}

    \begin{subfigure}{0.45\linewidth}
        \centering
        \includegraphics[width=\linewidth]{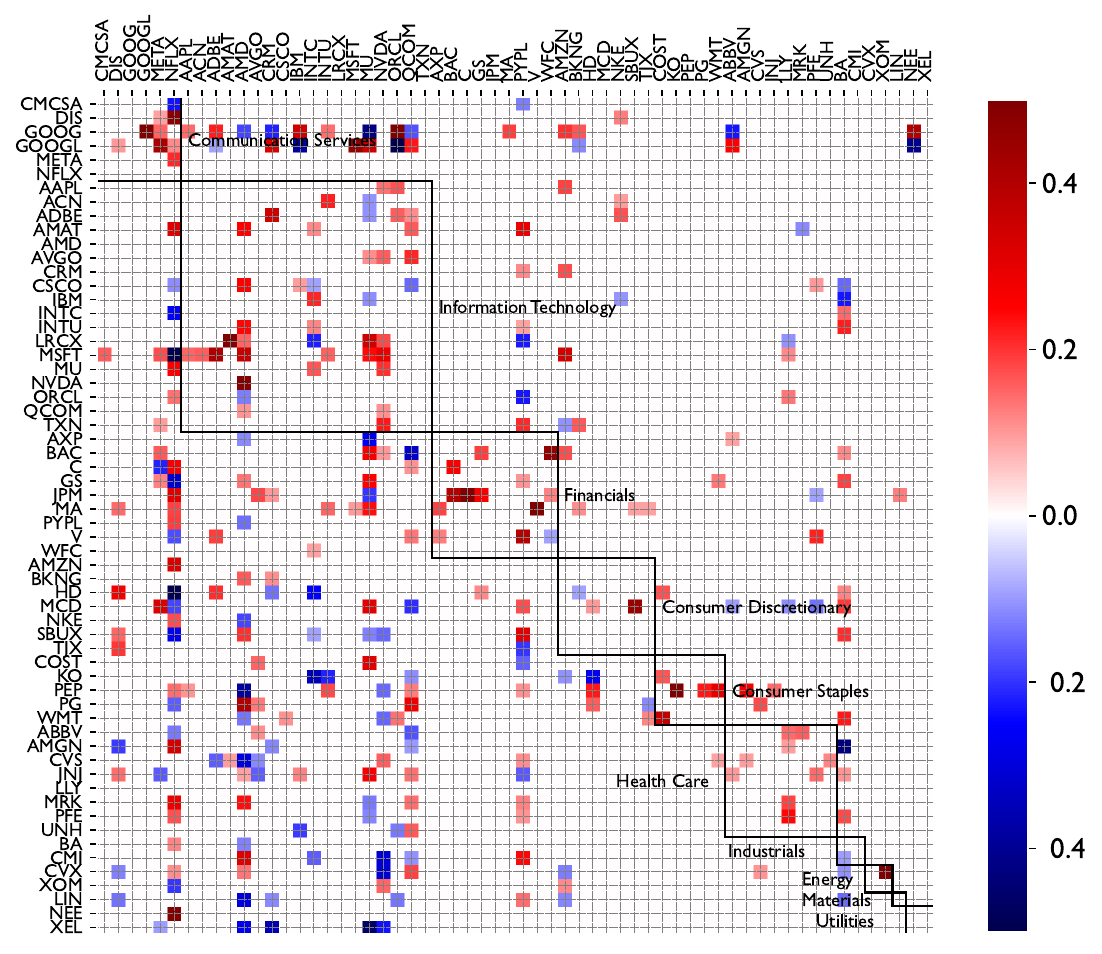}
        \caption{\varlingam estimate for $\widehat{\mB}_0$}
        \label{fig:stocks_lag0:appendix:varlingam}
    \end{subfigure}
    \begin{subfigure}{0.45\linewidth}
        \centering
        \includegraphics[width=1.25\linewidth]{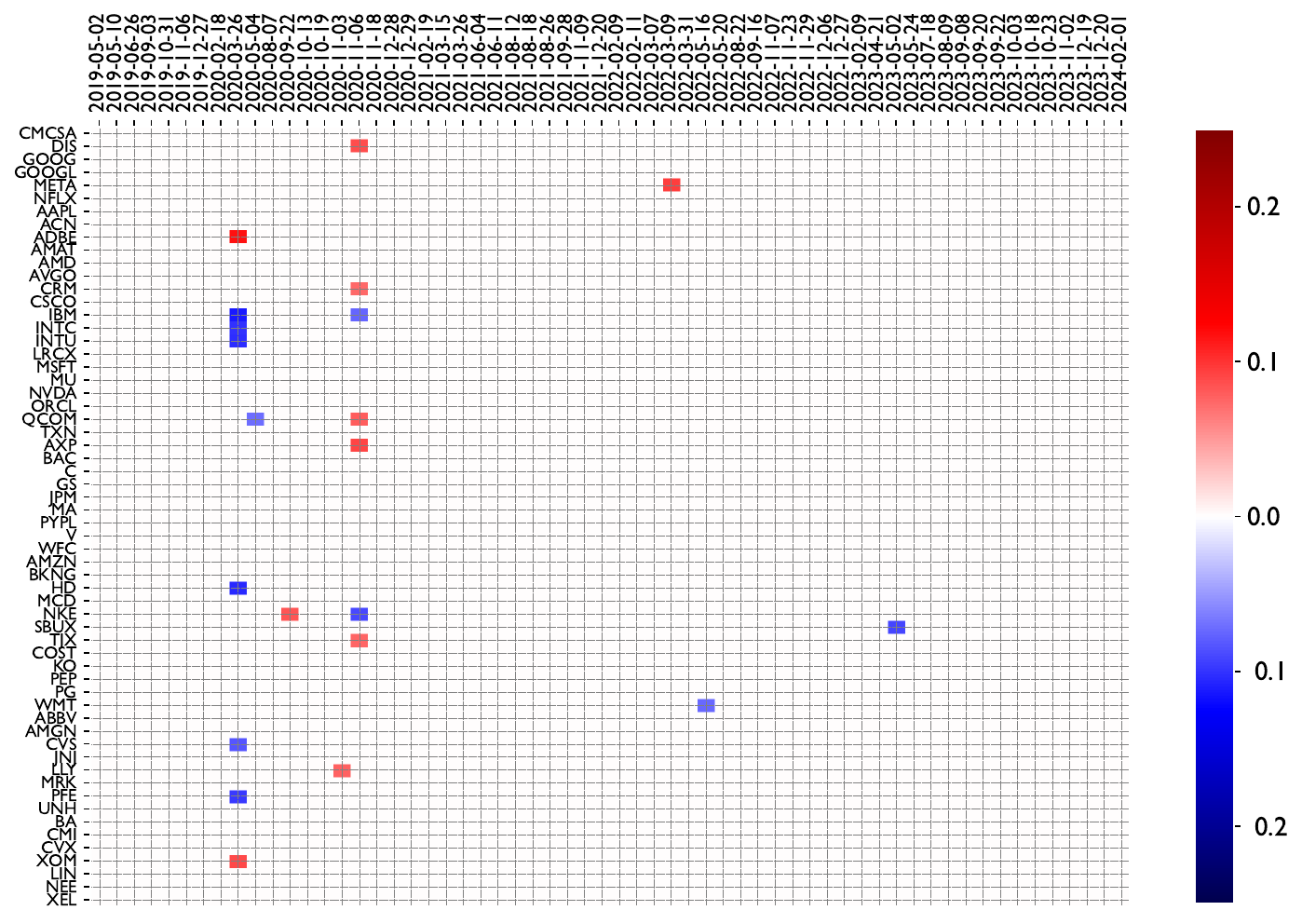}
        \caption{\varlingam estimate for $\widehat{\mS}$}
        \label{fig:stocks_rootcauses:appendix:varlingam}
    \end{subfigure}

    \caption{Evaluating baselines on the real experiment with S\&P 500 stock market index. (a) Instantaneous relations between the $45$ highest weighted stocks within S\&P 500 and (b) the discovered structural shocks for $60$ dates.}
    \label{fig:appendix:real1}
\end{figure}

\begin{figure}[H]
    \begin{subfigure}{0.45\linewidth}
        \centering
        \includegraphics[width=\linewidth]{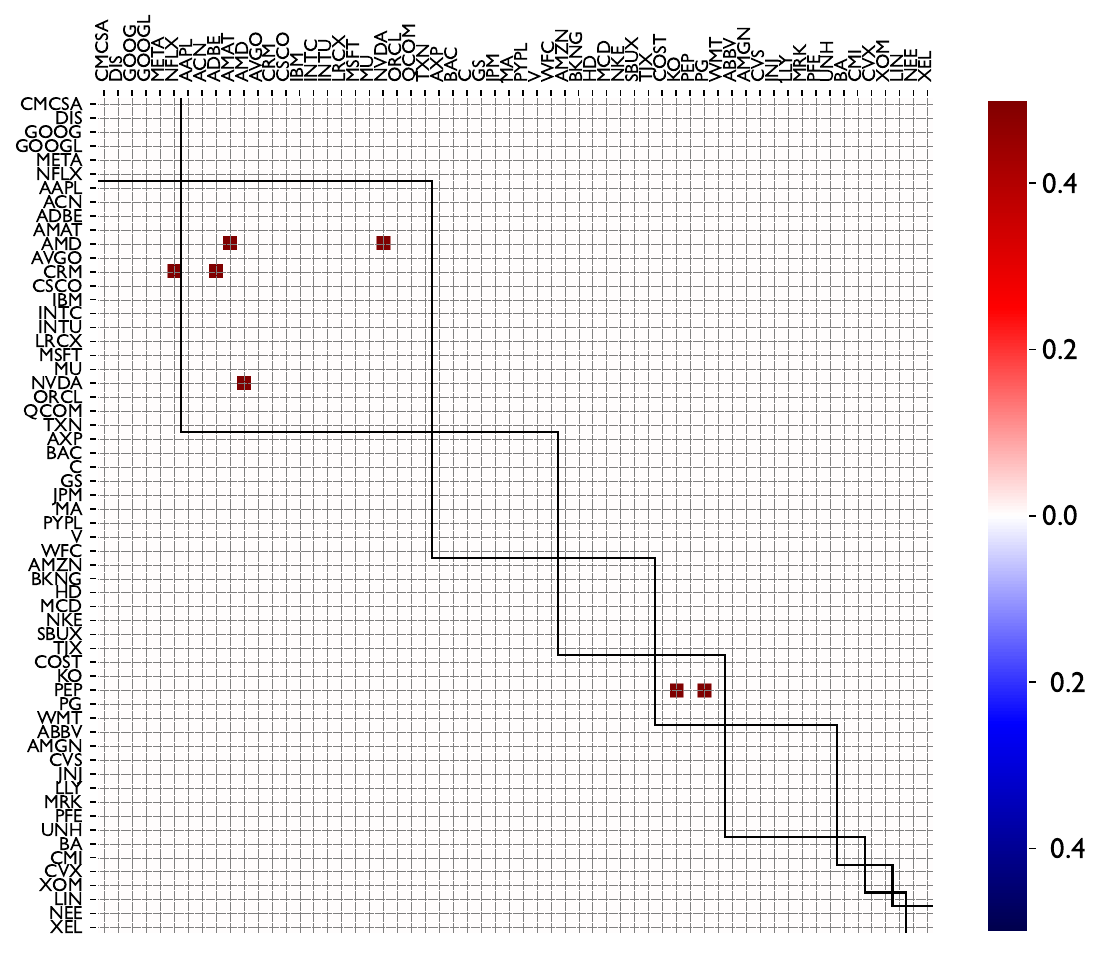}
        \caption{TCDF estimate for $\widehat{\mB}_0$}
        \label{fig:stocks_lag0:appendix:TCDF}
    \end{subfigure}
    \begin{subfigure}{0.45\linewidth}
        \centering
        \includegraphics[width=1.25\linewidth]{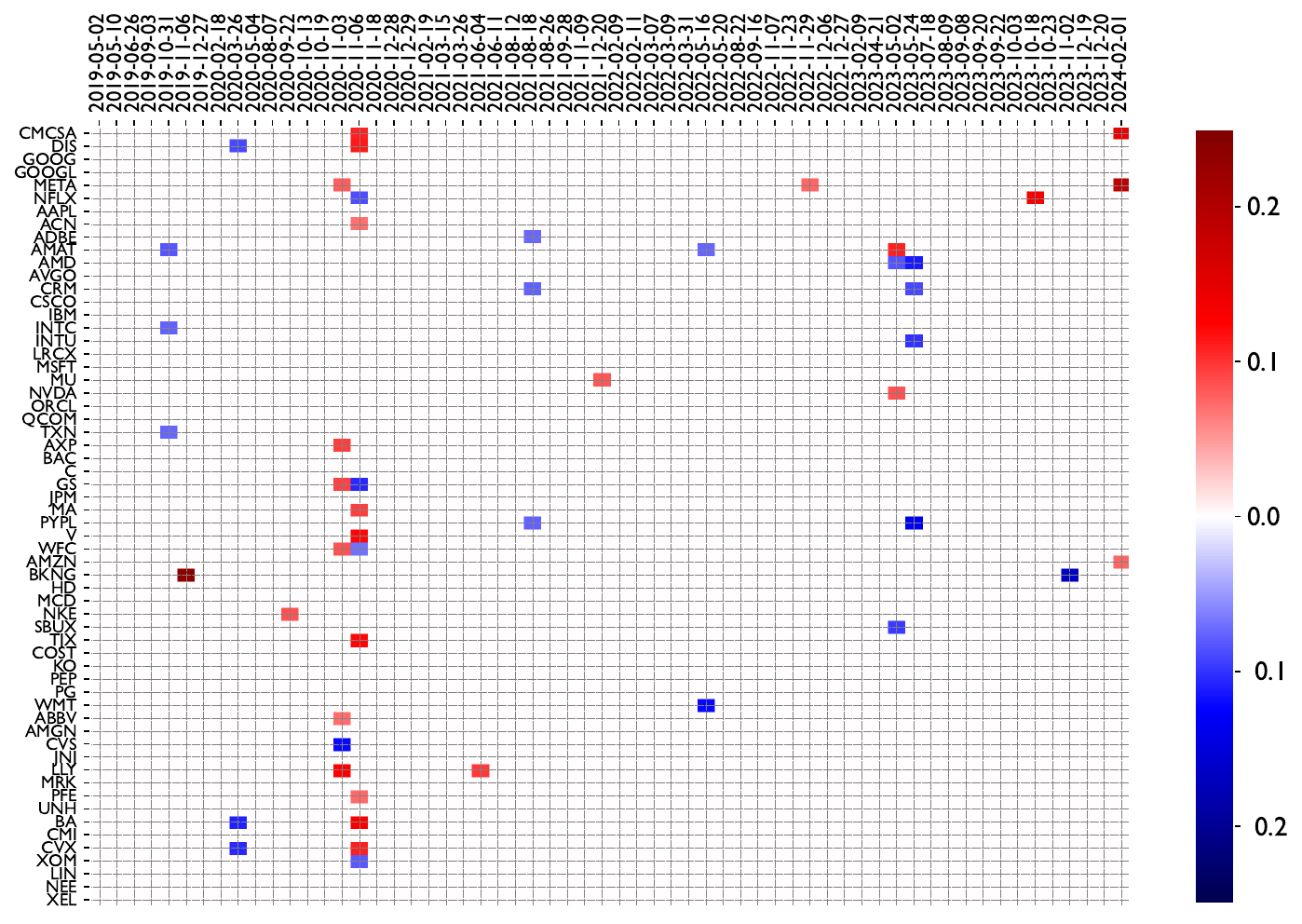}
        \caption{TCDF estimate for $\widehat{\mS}$}
        \label{fig:stocks_rootcauses:appendix:TCDF}
    \end{subfigure}

    \begin{subfigure}{0.45\linewidth}
        \centering
        \includegraphics[width=\linewidth]{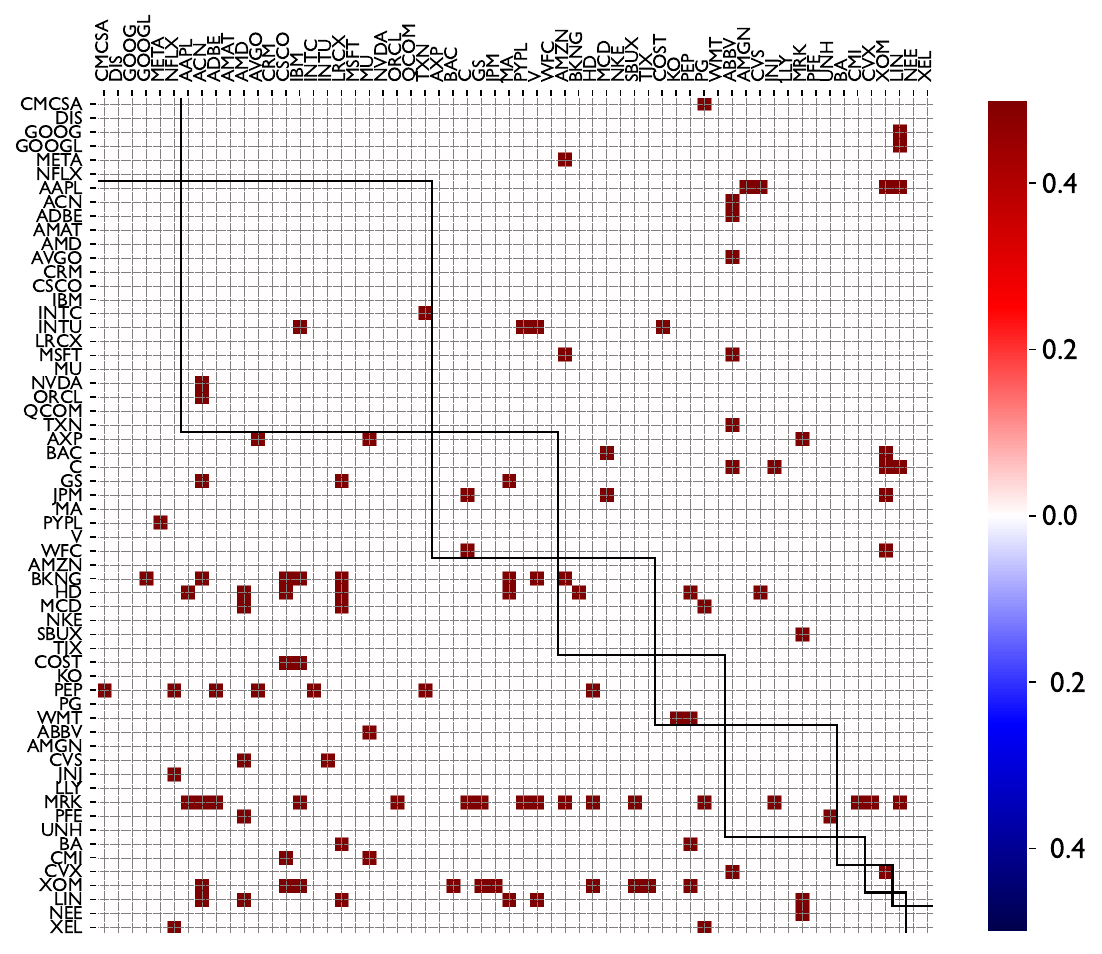}
        \caption{PCMCI estimate for $\widehat{\mB}_1$}
        \label{fig:stocks_lag0:appendix:pcmci}
    \end{subfigure}
    \begin{subfigure}{0.45\linewidth}
        \centering
        \includegraphics[width=1.25\linewidth]{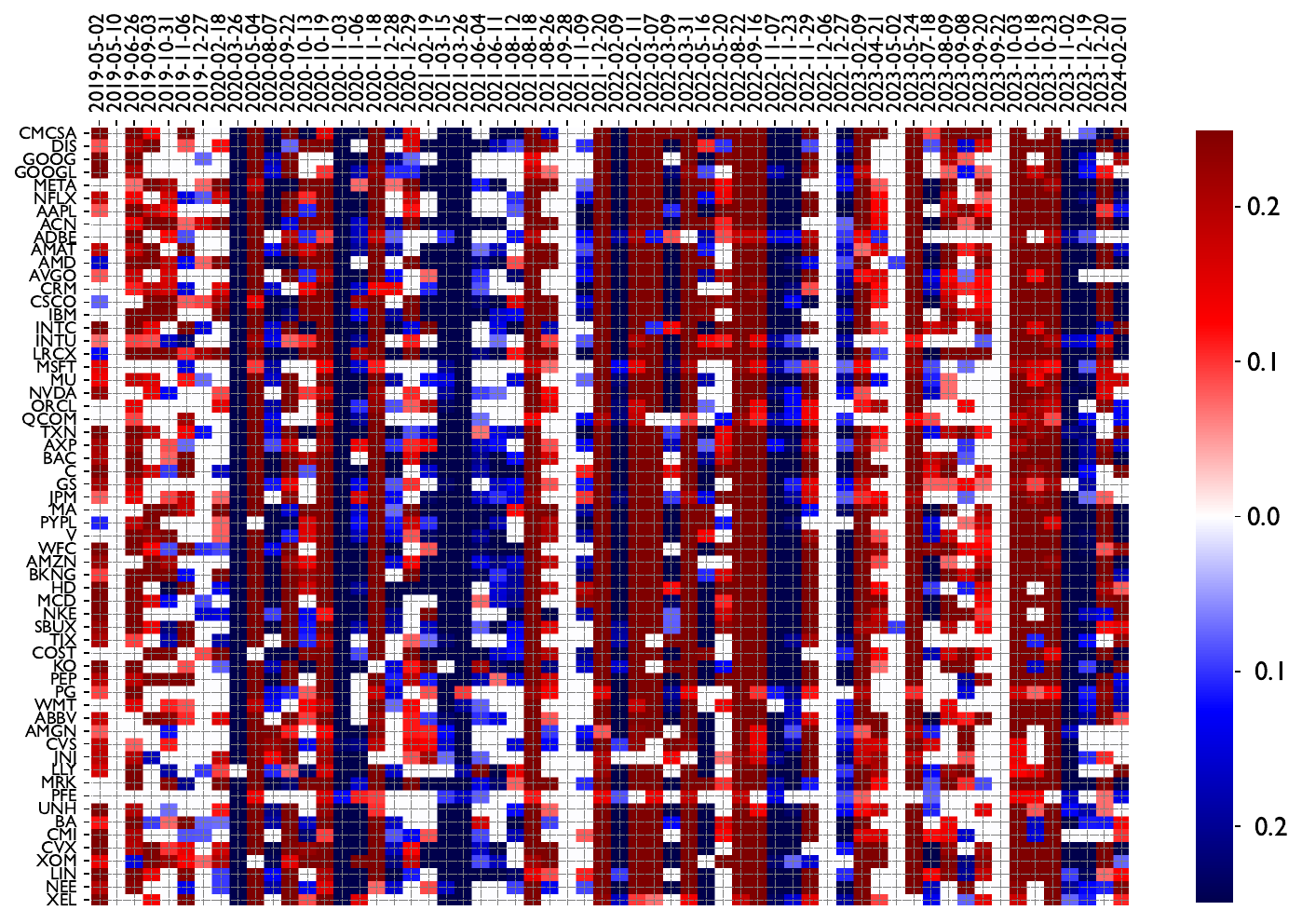}
        \caption{PCMCI estimate for $\widehat{\mS}$}
        \label{fig:stocks_rootcauses:appendix:pcmci}
    \end{subfigure}
    \caption{Evaluating PCMCI on the real experiment with S\&P 500 stock market index. (a) Relations between the $45$ highest weighted stocks within S\&P 500 with time lag $1$ and (b) the discovered structural shocks for $60$ dates.}
    \label{fig:appendix:real2}
\end{figure}

\subsection{Hyperparameter search}
\label{appendix:exp:hyperparameter}

To find the most suitable hyperparameter selection for each method in our synthetic and simulated experiments we perform a grid search and choose the parameter combination that achieves the best SHD performance. 

\subsubsection{Synthetic experiments with Laplacian input} 
\label{appendix:subsubsec:hyper_laplace}
For convenience we perform the grid search on small synthetic experimental settings ($N=1$ sample, $T=1000$ time steps, $d=20$ nodes) where all methods have reasonable execution time. 
Note that for all methods we set their parameters regarding the number of lags correctly, to equal the ground truth lag (default $k=2$). 
Any non-relevant hyperparameter that is not mentioned is set to its default value.
The hyperparameter search gave the following optimal hyperparameters for each method:


\paragraph{\mobius} We set $λ_1=0.0005,\, λ_2=0.5$ the coefficients for the $L^1$ and acyclicity regularizer, respectively and  $\omega=0.09$.
We let \mobius run for $10000$ epochs, although usually it terminates earlier as we have an early stopping activated when for $40$ consecutive epochs the loss didn't decrease.

\paragraph{SparseRC} We set $λ_1=0.001,\, λ_2=1,\, λ_3= 0.001$ the coefficients for the $L^1$, acyclicity and block-Toeplitz regularizers, respectively and $\omega=0.09$.
We similarly let SparseRC run for $10000$ epochs, although usually it terminates earlier using early stopping as with \mobius.

\paragraph{\varlingam} We may choose between ICA or Direct LiNGAM. 
In our experiments, we consider both cases (\varlingam and \dlingam). 
The weight threshold is set to $0.09$ both for \varlingam and \dlingam but for \clingam is set to $0.05$.

\paragraph{DYNOTEARS} The resulting values are $λ_w = λ_a = 0.01$ and $\omega=0.01$

\paragraph{NTS-NOTEARS} The resulting values are $λ_1 =0.002,\, λ_2 = 0.01$ and $\omega=0.01$
The $h_{tol}$ and the dimensions of the neural network were left to default.

\paragraph{tsFCI} Significance level is set to $0.1$ and $\omega=0.01$.
Note that the output of tsFCI is a partial ancestral graph (PAG), which we therefore need to interpret as a DAG. 
For this scope we follow the rules of DYNOTEARS~\citep{pamfil2020dynotears}, meaning that whenever there is ambiguity in the directionality of the discovered edge we assume that tsFCI made the correct choice (this favors and over-states the performance of tsFCI). 
In particular, we translate the edge between nodes $i$ and $j$ in the following ways (i) if $i\rightarrow$ we keep it, (ii) if $i\leftrightarrow j$ in the PAG we discard it, (iii) either $i\circ \rightarrow j$ or $i \circ -\circ j$ we assume tsFCI made the correct choice, by looking at the ground truth graph.

\paragraph{PCMCI} 
The ParCorr conditional independence test was chosen. 
We do so because this test is suitable for linear additive noise models. 
Parameters are set as $pc_a =0.1,\,a_{level} = 0.01$ and $\omega=0.01$.
The output can sometimes be ambiguous ($\circ-\circ$) because the algorithm can only find the graph up to the Markov equivalence class, or there be conflicts ($x-x$) in the conditional independence tests. 
In the former case, we assume that PCMCI made the correct choice and in the latter we disregard the edge.

\paragraph{TCDF} 
Here the kernel size and the dilation coefficient are set as the number of lags $+1$ ($k+1 = 3$). 
The other parameters are $significance=1$ and $epochs=1000$ and $\omega=0.01$.

\subsubsection{Synthetic experiments with Bernoulli-uniform input}
\label{appendix:subsubsec:hyper_bernoulli}
Similarly for the Laplacian input, we perform hyperparameter search for $N=1$ sample, $T=1000$ time steps and $d=20$ nodes. The hyperparameter search gave the following optimal hyperparameters for each method:


\paragraph{\mobius} We set $λ_1=0.0001,\, λ_2=0.1$  and $\omega = 0.09$.
We let \mobius run for $10000$ epochs.

\paragraph{SparseRC} We set $λ_1=0.001,\, λ_2=1,\, λ_3= 0.001$ and $\omega = 0.09$.
We similarly let SparseRC run for $10000$.

\paragraph{\varlingam} 
In our experiments, we consider both cases (\varlingam and \dlingam). 
The weight threshold is set to $0.09$ both for \varlingam and \dlingam but for \clingam is set to $0.05$.

\paragraph{DYNOTEARS} The resulting values are $λ_w = λ_a = 0.01$ and $\omega=0.09$.

\paragraph{NTS-NOTEARS} The resulting values are $λ_1 =0.002,\, λ_2 = 0.01$ and $\omega=0.09$.
The $h_{tol}$ and the dimensions of the neural network were left to default.

\paragraph{tsFCI} Significance level is set to $0.1$ and $\omega=0.09$.

\paragraph{PCMCI} 
Parameters are set as $pc_a =0.1,\,a_{level} = 0.01$ and $\omega=0.09$.

\paragraph{TCDF} 
Here the kernel size and the dilation coefficient are set as the number of lags $+1$ ($k+1 = 3$). 
The other parameters are $significance=1$ and $epochs=1000$ and $\omega=0.09$.

\subsubsection{Simulated financial data} 
\label{appendix:subsubsec:hyper_simulated}
Here we perform the grid search on the first available dataset of the simulated data (out of the $16$ available) and choose the hyperparameters offering the best SHD performance.  Here, we search for the most compatible weight threshold $\omega$ as the distribution of the ground truth weights is not known from the data generation. For all methods we set the number of maximum time lags at $3$, which is the maximal ground truth lag.  Any non-relevant hyperparameter that is not mentioned is set to its default value.
The hyperparameter search gave the following optimal hyperparameters for each method:


\paragraph{\mobius} We set $λ_1=0.01,\, λ_2=1,\, \omega=0.5$.
We let \mobius run for $10000$ epochs at maximum.

\paragraph{SparseRC} We set $λ_1=0.001,\, λ_2=1,\, λ_3= 0.1,\, \omega=0.3$.
We similarly let SparseRC run for $10000$ epochs at maximum.

\paragraph{\varlingam} The weight threshold is set to $\omega = 0.5$ for \varlingam and $\omega = 0.6$ for \dlingam and \clingam.

\paragraph{DYNOTEARS} The resulting values are $λ_w =0.05,\, λ_a = 0.01,\, \omega =0.3$.

\paragraph{NTS-NOTEARS} The resulting values are $λ_1 =0.001,\, λ_2 = 1,\,\omega = 0.1$. 
The $h_{tol}$ and the dimensions of the neural network were left to default.

\paragraph{tsFCI} Significance level is set to $0.001$ and $\omega = 0.1$
As previously we favor tsFCI in case of ambiguity, using the ground truth.

\paragraph{PCMCI} 
The ParCorr conditional independence test was chosen and parameters are set as $pc_a =0.1,\,a_{level} = 0.01,\, \omega = 0.1$. 
In case of ambiguity, we assume PCMCI made the correct choice.

\paragraph{TCDF} The kernel size and the dilation coefficient are set as number of lags $+1$ ($k+1 = 4$). 
The other parameters are $significance=0.8,\, epochs=1000,\,\omega=0.2$.

\subsubsection{DREAM3 dataset} 
\label{appendix:subsubsec:hyper_dream3}
Here we perform the grid search on the first available dataset of the data (out of the $5$ available) and choose the hyperparameters offering the best AUROC performance. We get the following results.

\paragraph{\mobius} $λ_1=0.001,\, λ_2=10,\, \omega=0.2$.
We let \mobius run for $10000$ epochs at maximum.

\paragraph{SparseRC} $λ_1=0.01,\, λ_2=0.1,\, λ_3= 0.1,\, \omega=0.2$.
We similarly let SparseRC run for $10000$ epochs at maximum.

\paragraph{\varlingam} The weight threshold is set to $\omega = 0.2$ for \varlingam, \dlingam and \clingam

\paragraph{DYNOTEARS} The resulting values are $λ_w =0.05,\, λ_a = 0.01,\, \omega =0.3$.

\paragraph{NTS-NOTEARS} The resulting values are $λ_1 =0.001,\, λ_2 = 0.01,\,\omega = 0.2$. 
The $h_{tol}$ and the dimensions of the neural network were left to default.

\paragraph{tsFCI} Significance level is set to $0.001$ and $\omega = 0.1$

\paragraph{PCMCI}  $pc_a =0.1,\,a_{level} = 0.01,\, \omega = 0.1$. 

\paragraph{TCDF} The kernel size and the dilation coefficient are set as number of lags $+1$ ($k+1 = 4$). 
The other parameters are $significance=0.8,\, epochs=1000,\,\omega=0.2$.

\subsection{Compute resources}
\label{appendix:exp:compute}

Our experiments were run on a single laptop machine (Dell Alienware x17 R2) with 8 core CPU with 32GB RAM and an NVIDIA GeForce RTX 3080 GPU. 
The execution of the synthetic experiments for the $5$ repetitions amounts to approximately 1 week of full run. 
Of course, initially there were some failed experiments, and after debugging the experiments were executed for only $1$ repetition to determine where each method has a time-out. 
We thus chose the time-out to $10000$ to try to make our experiments with as little cost as possible. 

\subsection{Code resources}
\label{appendix:exp:code_resources}


For the implementation of the methods in our experiments we use the following publicly available repositories or websites. All github repositories are licensed under the Apache 2.0 or MIT license, except tigramite and TCDF which are under the GPL-3.0 license.

\paragraph{SparseRC} SparseRC code \href{https://github.com/pmisiakos/SparseRC/}{https://github.com/pmisiakos/SparseRC/}. (MIT license)

\paragraph{\varlingam} We use the official \href{https://lingam.readthedocs.io/en/latest/index.html}{LiNGAM} repo which we clone from github: \href{https://github.com/cdt15/lingam}{https://github.com/cdt15/lingam}. (MIT license)

\paragraph{\clingam} \citet{akinwande2024acceleratedlingam} provide the following github repo: \href{https://github.com/aknvictor/culingam}{https://github.com/aknvictor/culingam}. (MIT license)

\paragraph{DYNOTEARS} Code is available from the CausalNex library of QuantumBlack. The code is at \href{https://github.com/mckinsey/causalnex/blob/develop/causalnex/structure/dynotears.py}{https://github.com/mckinsey/causalnex/blob/develop/causalnex/structure/dynotears.py} (Apache 2.0 license)

\paragraph{NTS-NOTEARS} We use the github code \href{https://github.com/xiangyu-sun-789/NTS-NOTEARS}{https://github.com/xiangyu-sun-789/NTS-NOTEARS} provided by~\citet{sun2023ntsnotears}. (Apache 2.0 license)

\paragraph{tsFCI} We use the R implementation from Doris Entner \href{https://sites.google.com/site/dorisentner/publications/tsfci}{website} which in turn utilizes the \href{tetrad}{https://www.cmu.edu/dietrich/philosophy/tetrad/}. Tetrad is licensed under the GNU General Public License v2.0. We also used the repository \href{https://github.com/ckassaad/causal_discovery_for_time_series}{https://github.com/ckassaad/causal\_discovery\_for\_time\_series} corresponding to the causal time series survey~\citep{assaad2022survey} (no license available). 

\paragraph{PCMCI} We use the \href{https://github.com/jakobrunge/tigramite/blob/master/tigramite/pcmci.py}{PCMCI implementation} from~\citep{runge2019PCMCI} within the 
\href{https://github.com/jakobrunge/tigramite}{tigramite} package. (GNU General Public License v3.0)

\paragraph{TCDF} We use the repository \href{https://github.com/M-Nauta/TCDF}{https://github.com/M-Nauta/TCDF} from~\citet{nauta2019TCDF}. (GNU General Public License v3.0)

\paragraph{eSRU} We use the repository \href{https://github.com/iancovert/Neural-GC}{https://github.com/iancovert/Neural-GC} from~\citet{khanna2019eSRU}. (MIT License)

\subsection{Data resources}
\label{appendix:exp:data_resources}
\paragraph{Simulated financial time series} We take the data from \href{http://www.skleinberg.org/data.html}{http://www.skleinberg.org/data.html} licensed under CC BY-NC 3.0

\paragraph{S\&P 500 stock returns} The data are downloaded using \textit{yahoofinancials} python library.


\begin{thebibliography}{79}
\providecommand{\natexlab}[1]{#1}
\providecommand{\url}[1]{\texttt{#1}}
\expandafter\ifx\csname urlstyle\endcsname\relax
  \providecommand{\doi}[1]{doi: #1}\else
  \providecommand{\doi}{doi: \begingroup \urlstyle{rm}\Url}\fi

\bibitem[Akinwande and Kolter(2024)]{akinwande2024acceleratedlingam}
Victor Akinwande and J~Zico Kolter.
\newblock {AcceleratedLiNGAM: Learning Causal DAGs at the speed of GPUs}.
\newblock \emph{arXiv preprint arXiv:2403.03772}, 2024.

\bibitem[Assaad et~al.(2022{\natexlab{a}})Assaad, Devijver, and Gaussier]{assaad2022PC-GCE}
Charles~K Assaad, Emilie Devijver, and Eric Gaussier.
\newblock {Discovery of extended summary graphs in time series}.
\newblock In \emph{Uncertainty in Artificial Intelligence}, pages 96--106. PMLR, 2022{\natexlab{a}}.

\bibitem[Assaad et~al.(2022{\natexlab{b}})Assaad, Devijver, and Gaussier]{assaad2022survey}
Charles~K Assaad, Emilie Devijver, and Eric Gaussier.
\newblock {Survey and Evaluation of Causal Discovery Methods for Time Series}.
\newblock \emph{Journal of Artificial Intelligence Research}, 73:\penalty0 767--819, 2022{\natexlab{b}}.

\bibitem[Assaad et~al.(2021)Assaad, Devijver, Gaussier, and Ait-Bachir]{assaad2021NBCB}
Karim Assaad, Emilie Devijver, Eric Gaussier, and Ali Ait-Bachir.
\newblock {A Mixed Noise and Constraint-Based Approach to Causal Inference in Time Series}.
\newblock In \emph{Machine Learning and Knowledge Discovery in Databases. Research Track: European Conference, ECML PKDD 2021, Bilbao, Spain, September 13--17, 2021, Proceedings, Part I 21}, pages 453--468. Springer, 2021.

\bibitem[Babacan et~al.(2009)Babacan, Molina, and Katsaggelos]{babacan2009bayesianComprSensing}
S~Derin Babacan, Rafael Molina, and Aggelos~K Katsaggelos.
\newblock Bayesian compressive sensing using laplace priors.
\newblock \emph{IEEE Transactions on image processing}, 19\penalty0 (1):\penalty0 53--63, 2009.

\bibitem[Bassett~Jr and Koenker(1978)]{bassett1978asymptotic_theoryLAD}
Gilbert Bassett~Jr and Roger Koenker.
\newblock Asymptotic theory of least absolute error regression.
\newblock \emph{Journal of the American Statistical Association}, 73\penalty0 (363):\penalty0 618--622, 1978.

\bibitem[Bellot et~al.(2022)Bellot, Branson, and van~der Schaar]{bellot2021neuralgraphmodelling}
Alexis Bellot, Kim Branson, and Mihaela van~der Schaar.
\newblock Neural graphical modelling in continuous-time: consistency guarantees and algorithms.
\newblock In \emph{International Conference on Learning Representations}, 2022.

\bibitem[Bussmann et~al.(2021)Bussmann, Nys, and Latr{\'e}]{bussmann2021NAVAR}
Bart Bussmann, Jannes Nys, and Steven Latr{\'e}.
\newblock {Neural Additive Vector Autoregression Models for Causal Discovery in Time Series}.
\newblock In \emph{Discovery Science: 24th International Conference, DS 2021, Halifax, NS, Canada, October 11--13, 2021, Proceedings 24}, pages 446--460. Springer, 2021.

\bibitem[Castillo et~al.(2015)Castillo, Schmidt-Hieber, and Van~der Vaart]{castillo2015bayesiansparseregression}
Isma{\"e}l Castillo, Johannes Schmidt-Hieber, and Aad Van~der Vaart.
\newblock Bayesian linear regression with sparse priors.
\newblock \emph{The Annals of Statistics}, pages 1986--2018, 2015.

\bibitem[Chai et~al.(2019)Chai, Du, Liu, and Lee]{chai2019GeneralGaussianDistRegression}
Li~Chai, Jun Du, Qing-Feng Liu, and Chin-Hui Lee.
\newblock Using generalized gaussian distributions to improve regression error modeling for deep learning-based speech enhancement.
\newblock \emph{IEEE/ACM Transactions on Audio, Speech, and Language Processing}, 27\penalty0 (12):\penalty0 1919--1931, 2019.

\bibitem[Cheng et~al.(2024)Cheng, Li, Xiao, Li, Suo, He, and Dai]{cheng2024cuts+}
Yuxiao Cheng, Lianglong Li, Tingxiong Xiao, Zongren Li, Jinli Suo, Kunlun He, and Qionghai Dai.
\newblock {CUTS+: High-dimensional Causal Discovery from Irregular Time-series}.
\newblock In \emph{Proceedings of the AAAI Conference on Artificial Intelligence}, volume~38, pages 11525--11533, 2024.

\bibitem[Chickering et~al.(2004)Chickering, Heckerman, and Meek]{chickering2004nphard}
Max Chickering, David Heckerman, and Chris Meek.
\newblock {Large-Sample Learning of Bayesian Networks is NP-Hard}.
\newblock \emph{Journal of Machine Learning Research}, 5:\penalty0 1287--1330, 2004.

\bibitem[Eltoft et~al.(2006)Eltoft, Kim, and Lee]{eltoft2006multivariatelaplacedist}
Torbj{\o}rn Eltoft, Taesu Kim, and Te-Won Lee.
\newblock On the multivariate laplace distribution.
\newblock \emph{IEEE Signal Processing Letters}, 13\penalty0 (5):\penalty0 300--303, 2006.

\bibitem[Entner and Hoyer(2010)]{entner2010tsFCI}
Doris Entner and Patrik~O Hoyer.
\newblock {On Causal Discovery from Time Series Data using FCI}.
\newblock \emph{Probabilistic graphical models}, pages 121--128, 2010.

\bibitem[Fama(1970)]{fama1970efficient}
Eugene~F Fama.
\newblock {Efficient Capital Markets: A Review of Theory and Empirical Work}.
\newblock \emph{Journal of finance}, 25\penalty0 (2):\penalty0 383--417, 1970.

\bibitem[Fiorentini and Sentana(2023)]{fiorentini2023pseudoMLE_SVAR}
Gabriele Fiorentini and Enrique Sentana.
\newblock Discrete mixtures of normals pseudo maximum likelihood estimators of structural vector autoregressions.
\newblock \emph{Journal of Econometrics}, 235\penalty0 (2):\penalty0 643--665, 2023.

\bibitem[Gao et~al.(2022)Gao, Bhattacharjya, Nelson, Liu, and Yu]{gao2022idyno}
Tian Gao, Debarun Bhattacharjya, Elliot Nelson, Miao Liu, and Yue Yu.
\newblock {IDYNO: Learning Nonparametric DAGs from Interventional Dynamic Data}.
\newblock In \emph{International Conference on Machine Learning}, pages 6988--7001. PMLR, 2022.

\bibitem[Gerhardus and Runge(2020)]{gerhardus2020lpcmci}
Andreas Gerhardus and Jakob Runge.
\newblock {High-recall causal discovery for autocorrelated time series with latent confounders}.
\newblock \emph{Advances in Neural Information Processing Systems}, 33:\penalty0 12615--12625, 2020.

\bibitem[Gong et~al.(2023)Gong, Yao, Zhang, Li, Bi, Du, and Wang]{causal2023temporaloverview}
Chang Gong, Di~Yao, Chuzhe Zhang, Wenbin Li, Jingping Bi, Lun Du, and Jin Wang.
\newblock {Causal Discovery from Temporal Data: An Overview and New Perspectives}.
\newblock KDD '23, page 5803–5804. Association for Computing Machinery, 2023.

\bibitem[Gong et~al.(2015)Gong, Zhang, Schoelkopf, Tao, and Geiger]{gong2015subsampledNGEM}
Mingming Gong, Kun Zhang, Bernhard Schoelkopf, Dacheng Tao, and Philipp Geiger.
\newblock {Discovering Temporal Causal Relations from Subsampled Data }.
\newblock In \emph{International Conference on Machine Learning}, pages 1898--1906. PMLR, 2015.

\bibitem[Gong et~al.(2022)Gong, Jennings, Zhang, and Pawlowski]{gong2022rhino}
Wenbo Gong, Joel Jennings, Cheng Zhang, and Nick Pawlowski.
\newblock Rhino: Deep causal temporal relationship learning with history-dependent noise.
\newblock \emph{arXiv preprint arXiv:2210.14706}, 2022.

\bibitem[Guan and Dy(2009)]{guan2009sparsePCA}
Yue Guan and Jennifer Dy.
\newblock Sparse probabilistic principal component analysis.
\newblock In \emph{Artificial Intelligence and Statistics}, pages 185--192. PMLR, 2009.

\bibitem[Hasan et~al.(2023)Hasan, Hossain, and Gani]{hasan2023causalsurvey}
Uzma Hasan, Emam Hossain, and Md~Osman Gani.
\newblock {A Survey on Causal Discovery Methods for I.I.D. and Time Series Data }.
\newblock \emph{Transactions on Machine Learning Research}, 2023.

\bibitem[He and Sun(2024)]{he2024LAD_sparse_dynamical_systems}
Xin He and ZhongKui Sun.
\newblock Sparse identification of dynamical systems by reweighted l1-regularized least absolute deviation regression.
\newblock \emph{Communications in Nonlinear Science and Numerical Simulation}, 131:\penalty0 107813, 2024.

\bibitem[Horn and Johnson(2012)]{horn2012matrixanalysis}
Roger~A. Horn and Charles~R. Johnson.
\newblock \emph{Matrix analysis}.
\newblock Cambridge university press, 2012.

\bibitem[Hyv{\"a}rinen et~al.(2010)Hyv{\"a}rinen, Zhang, Shimizu, and Hoyer]{hyvarinen2010varlingam}
Aapo Hyv{\"a}rinen, Kun Zhang, Shohei Shimizu, and Patrik~O Hoyer.
\newblock {Estimation of a Structural Vector Autoregression Model Using Non-Gaussianity}.
\newblock \emph{Journal of Machine Learning Research}, 11\penalty0 (5), 2010.

\bibitem[Jiang et~al.(2023)Jiang, Du, Yang, and Deng]{jiang2023RLAD_dynamical_systems}
Feng Jiang, Lin Du, Fan Yang, and Zi-Chen Deng.
\newblock Regularized least absolute deviation-based sparse identification of dynamical systems.
\newblock \emph{Chaos: An Interdisciplinary Journal of Nonlinear Science}, 33\penalty0 (1), 2023.

\bibitem[Jiang and Shimizu(2023)]{varlingam2023linkagesfinance}
Yi~Jiang and Shohei Shimizu.
\newblock {Linkages among the Foreign Exchange, Stock, and Bond Markets in Japan and the United States}.
\newblock In \emph{Causal Analysis Workshop Series}, pages 1--19. PMLR, 2023.

\bibitem[Jing et~al.(2015)Jing, Wang, and Yang]{jing2015sparsematrixfactLaplace}
Liping Jing, Peng Wang, and Liu Yang.
\newblock Sparse probabilistic matrix factorization by laplace distribution for collaborative filtering.
\newblock In \emph{Twenty-Fourth International Joint Conference on Artificial Intelligence}, 2015.

\bibitem[Kalisch and B{\"u}hlman(2007)]{kalisch2007highDimDAGs}
Markus Kalisch and Peter B{\"u}hlman.
\newblock Estimating high-dimensional directed acyclic graphs with the pc-algorithm.
\newblock \emph{Journal of Machine Learning Research}, 8\penalty0 (3), 2007.

\bibitem[Khanna and Tan(2019)]{khanna2019eSRU}
Saurabh Khanna and Vincent~YF Tan.
\newblock {Economy Statistical Recurrent Units For Inferring Nonlinear Granger Causality}.
\newblock In \emph{International Conference on Learning Representations}, 2019.

\bibitem[Kilian(2013)]{kilian2013SVAR}
Lutz Kilian.
\newblock {Structural Vector Autoregressions}.
\newblock In \emph{Handbook of research methods and applications in empirical macroeconomics}, pages 515--554. Edward Elgar Publishing, 2013.

\bibitem[Kim and Anderson(2012)]{kim2012temporal}
Hyoungshick Kim and Ross Anderson.
\newblock Temporal node centrality in complex networks.
\newblock \emph{Physical Review E}, 85\penalty0 (2):\penalty0 026107, 2012.

\bibitem[Kingma and Ba(2014)]{kingma2014adam}
Diederik~P Kingma and Jimmy Ba.
\newblock Adam: A method for stochastic optimization.
\newblock \emph{arXiv preprint arXiv:1412.6980}, 2014.

\bibitem[Kleinberg(2013)]{kleinberg2013finance}
Samantha Kleinberg.
\newblock \emph{{Causality, Probability, and Time}}.
\newblock Cambridge University Press, 2013.

\bibitem[Kumar and Singh(2015)]{kumar2015regression_model_LAD}
Pranesh Kumar and Jai~Narain Singh.
\newblock Regression model estimation using least absolute deviations, least squares deviations and minimax absolute deviations criteria.
\newblock \emph{IJCSEE}, 3\penalty0 (4):\penalty0 2320--4028, 2015.

\bibitem[Lachapelle et~al.(2019)Lachapelle, Brouillard, Deleu, and Lacoste-Julien]{lachapelle2019granDAG}
S{\'e}bastien Lachapelle, Philippe Brouillard, Tristan Deleu, and Simon Lacoste-Julien.
\newblock Gradient-based neural dag learning.
\newblock In \emph{International Conference on Learning Representations}, 2019.

\bibitem[Lanne et~al.(2017)Lanne, Meitz, and Saikkonen]{lanne2017MLE_estimation_SVAR}
Markku Lanne, Mika Meitz, and Pentti Saikkonen.
\newblock Identification and estimation of non-gaussian structural vector autoregressions.
\newblock \emph{Journal of Econometrics}, 196\penalty0 (2):\penalty0 288--304, 2017.

\bibitem[Li and Arce(2004)]{li2004fastMLE_LAD}
Yinbo Li and Gonzalo~R Arce.
\newblock A fast maximum likelihood estimation approach to lad regression.
\newblock In \emph{2004 IEEE International Conference on Acoustics, Speech, and Signal Processing}, volume~2, pages ii--889. IEEE, 2004.

\bibitem[L{\"o}we et~al.(2022)L{\"o}we, Madras, Zemel, and Welling]{lowe2022amortized}
Sindy L{\"o}we, David Madras, Richard Zemel, and Max Welling.
\newblock {Amortized Causal Discovery: Learning to Infer Causal Graphs from Time-Series Data}.
\newblock In \emph{Conference on Causal Learning and Reasoning}, pages 509--525. PMLR, 2022.

\bibitem[L{\"u}tkepohl(2005)]{lutkepohl2005new}
Helmut L{\"u}tkepohl.
\newblock \emph{{New Introduction to Multiple Time Series Analysis}}.
\newblock Springer Science \& Business Media, 2005.

\bibitem[Maekawa and Nakanishi(2023)]{maekawa2023pseudo_log_likelihood_nonGauss}
Koichi Maekawa and Tadashi Nakanishi.
\newblock Estimation of non-gaussian svar models: a pseudo-log-likelihood function approach.
\newblock \emph{Journal of Statistical Computation and Simulation}, 93\penalty0 (11):\penalty0 1830--1850, 2023.

\bibitem[Malinsky and Spirtes(2018)]{malinsky2018SVAR-FCI}
Daniel Malinsky and Peter Spirtes.
\newblock {Causal Structure Learning from Multivariate Time Series in Settings with Unmeasured Confounding}.
\newblock In \emph{Proceedings of 2018 ACM SIGKDD workshop on causal discovery}, pages 23--47. PMLR, 2018.

\bibitem[Marbach et~al.(2009)Marbach, Schaffter, Mattiussi, and Floreano]{marbach2009dream3}
Daniel Marbach, Thomas Schaffter, Claudio Mattiussi, and Dario Floreano.
\newblock Generating realistic in silico gene networks for performance assessment of reverse engineering methods.
\newblock \emph{Journal of computational biology}, 16\penalty0 (2):\penalty0 229--239, 2009.

\bibitem[Marcinkevi{\v{c}}s and Vogt(2020)]{marcinkevivcs2020GVAR}
Ri{\v{c}}ards Marcinkevi{\v{c}}s and Julia~E Vogt.
\newblock {Interpretable Models for Granger Causality Using Self-explaining Neural Networks}.
\newblock In \emph{International Conference on Learning Representations}, 2020.

\bibitem[Misiakos et~al.(2023)Misiakos, Wendler, and P{\"u}schel]{misiakos2024fewrootcauses}
Panagiotis Misiakos, Chris Wendler, and Markus P{\"u}schel.
\newblock {Learning DAGs from Data with Few Root Causes}.
\newblock \emph{Advances in Neural Information Processing Systems}, 36, 2023.

\bibitem[Misiakos et~al.(2024)Misiakos, Mihal, and P{\"u}schel]{misiakos2024icassp}
Panagiotis Misiakos, Vedran Mihal, and Markus P{\"u}schel.
\newblock {Learning Signals and Graphs from Time-Series Graph Data with Few Causes}.
\newblock In \emph{ICASSP 2024-2024 IEEE International Conference on Acoustics, Speech and Signal Processing (ICASSP)}, pages 9681--9685, 2024.

\bibitem[Narula et~al.(1999)Narula, Saldiva, Andre, Elian, Ferreira, and Capelozzi]{narula1999minimumAbsErrorRegression}
Subhash~C Narula, Paulo~HN Saldiva, Carmen~DS Andre, Silvia~N Elian, Aurea~Favero Ferreira, and Vera Capelozzi.
\newblock The minimum sum of absolute errors regression: a robust alternative to the least squares regression.
\newblock \emph{Statistics in medicine}, 18\penalty0 (11):\penalty0 1401--1417, 1999.

\bibitem[Nauta et~al.(2019)Nauta, Bucur, and Seifert]{nauta2019TCDF}
Meike Nauta, Doina Bucur, and Christin Seifert.
\newblock {Causal Discovery with Attention-Based Convolutional Neural Networks}.
\newblock \emph{Machine Learning and Knowledge Extraction}, 1\penalty0 (1):\penalty0 19, 2019.

\bibitem[Newey and McFadden(1994)]{newey1994MLEconsistency}
Whitney~K Newey and Daniel McFadden.
\newblock Large sample estimation and hypothesis testing.
\newblock \emph{Handbook of econometrics}, 4:\penalty0 2111--2245, 1994.

\bibitem[Ng et~al.(2020)Ng, Ghassami, and Zhang]{ng2020GOLEM}
Ignavier Ng, AmirEmad Ghassami, and Kun Zhang.
\newblock {On the Role of Sparsity and DAG Constraints for Learning Linear DAGs}.
\newblock \emph{Advances in Neural Information Processing Systems}, 33:\penalty0 17943--17954, 2020.

\bibitem[Pamfil et~al.(2020)Pamfil, Sriwattanaworachai, Desai, Pilgerstorfer, Georgatzis, Beaumont, and Aragam]{pamfil2020dynotears}
Roxana Pamfil, Nisara Sriwattanaworachai, Shaan Desai, Philip Pilgerstorfer, Konstantinos Georgatzis, Paul Beaumont, and Bryon Aragam.
\newblock {DYNOTEARS: Structure Learning from Time-Series Data}.
\newblock In \emph{International Conference on Artificial Intelligence and Statistics}, pages 1595--1605. PMLR, 2020.

\bibitem[Park(2020)]{park2020conditional}
Gunwoong Park.
\newblock {Identifiability of Additive Noise Models Using Conditional Variances}.
\newblock \emph{J. Mach. Learn. Res.}, 21\penalty0 (75):\penalty0 1--34, 2020.

\bibitem[Peters and B{\"u}hlmann(2014)]{peters2014identifiability}
Jonas Peters and Peter B{\"u}hlmann.
\newblock {Identifiability of Gaussian structural equation models with equal error variances}.
\newblock \emph{Biometrika}, 101\penalty0 (1):\penalty0 219--228, 2014.

\bibitem[Peters and B{\"u}hlmann(2015)]{peters201SID}
Jonas Peters and Peter B{\"u}hlmann.
\newblock Structural intervention distance for evaluating causal graphs.
\newblock \emph{Neural computation}, 27\penalty0 (3):\penalty0 771--799, 2015.

\bibitem[Peters et~al.(2013)Peters, Janzing, and Sch{\"o}lkopf]{peters2013timino}
Jonas Peters, Dominik Janzing, and Bernhard Sch{\"o}lkopf.
\newblock {Causal Inference on Time Series using Structural Equation Models}.
\newblock \emph{Advances in neural information processing systems}, 26, 2013.

\bibitem[Peters et~al.(2017)Peters, Janzing, and Sch{\"o}lkopf]{elementsCausalInference}
Jonas Peters, Dominik Janzing, and Bernhard Sch{\"o}lkopf.
\newblock \emph{{Elements of causal inference: foundations and learning algorithms}}.
\newblock The MIT Press, 2017.

\bibitem[Pollard(1991)]{pollard1991asymptoticsLAD}
David Pollard.
\newblock Asymptotics for least absolute deviation regression estimators.
\newblock \emph{Econometric Theory}, 7\penalty0 (2):\penalty0 186--199, 1991.

\bibitem[Prill et~al.(2010)Prill, Marbach, Saez-Rodriguez, Sorger, Alexopoulos, Xue, Clarke, Altan-Bonnet, and Stolovitzky]{prill2010dream3}
Robert~J Prill, Daniel Marbach, Julio Saez-Rodriguez, Peter~K Sorger, Leonidas~G Alexopoulos, Xiaowei Xue, Neil~D Clarke, Gregoire Altan-Bonnet, and Gustavo Stolovitzky.
\newblock Towards a rigorous assessment of systems biology models: the dream3 challenges.
\newblock \emph{PloS one}, 5\penalty0 (2):\penalty0 e9202, 2010.

\bibitem[Reuters(2023)]{mehta2023nvidia}
Reuters.
\newblock Nvidia shares soar nearly 30\% as sales forecast jumps and ai booms.
\newblock \url{https://www.reuters.com/technology/nvidia-forecasts-second-quarter-revenue-above-estimates-2023-05-24/}, 2023.
\newblock Accessed: 2024-05-21.

\bibitem[Reuters(2024)]{paul2024facebook}
Reuters.
\newblock Facebook parent meta declares first dividend, shares soar.
\newblock \url{https://www.reuters.com/technology/facebook-parent-meta-declares-first-ever-dividend-2024-02-01/}, 2024.
\newblock Accessed: 2024-05-21.

\bibitem[Runge(2020)]{runge2020pcmci+}
Jakob Runge.
\newblock {Discovering contemporaneous and lagged causal relations in autocorrelated nonlinear time series datasets}.
\newblock In \emph{Conference on Uncertainty in Artificial Intelligence}, pages 1388--1397. PMLR, 2020.

\bibitem[Runge et~al.(2019)Runge, Nowack, Kretschmer, Flaxman, and Sejdinovic]{runge2019PCMCI}
Jakob Runge, Peer Nowack, Marlene Kretschmer, Seth Flaxman, and Dino Sejdinovic.
\newblock {Detecting and quantifying causal associations in large nonlinear time series datasets}.
\newblock \emph{Science advances}, 5\penalty0 (11):\penalty0 eaau4996, 2019.

\bibitem[Saikkonen(2001)]{saikkonen2001stability}
Pentti Saikkonen.
\newblock \emph{{Stability results for nonlinear vector autoregressions with an application to a nonlinear error correction model}}.
\newblock Humboldt-Universit{\"a}t zu Berlin, Wirtschaftswissenschaftliche Fakult{\"a}t, 2001.

\bibitem[Seifert et~al.(2023)Seifert, Wendler, and Püschel]{bastiJournalpaper}
Bastian Seifert, Chris Wendler, and Markus Püschel.
\newblock {Causal Fourier Analysis on Directed Acyclic Graphs and Posets}.
\newblock \emph{IEEE Trans.~Signal Process.}, 71:\penalty0 3805--3820, 2023.
\newblock \doi{10.1109/TSP.2023.3324988}.

\bibitem[Shimizu et~al.(2006)Shimizu, Hoyer, Hyvärinen, and Kerminen]{shimizu2006lingam}
Shohei Shimizu, Patrik~O. Hoyer, Aapo Hyvärinen, and Antti Kerminen.
\newblock {A Linear Non-Gaussian Acyclic Model for Causal Discovery}.
\newblock \emph{Journal of Machine Learning Research}, 7\penalty0 (72):\penalty0 2003--2030, 2006.
\newblock URL \url{http://jmlr.org/papers/v7/shimizu06a.html}.

\bibitem[Shimizu et~al.(2011)Shimizu, Inazumi, Sogawa, Hyvarinen, Kawahara, Washio, Hoyer, Bollen, and Hoyer]{shimizu2011directlingam}
Shohei Shimizu, Takanori Inazumi, Yasuhiro Sogawa, Aapo Hyvarinen, Yoshinobu Kawahara, Takashi Washio, Patrik~O Hoyer, Kenneth Bollen, and Patrik Hoyer.
\newblock {DirectLiNGAM: A direct method for learning a linear non-Gaussian structural equation model}.
\newblock \emph{Journal of Machine Learning Research-JMLR}, 12\penalty0 (Apr):\penalty0 1225--1248, 2011.

\bibitem[Sims(1980)]{sims1980comparison}
Christopher~A Sims.
\newblock {Comparison of Interwar and Postwar Business Cycles: Monetarism Reconsidered}, 1980.

\bibitem[Skogestad and Postlethwaite(2005)]{skogestad2005multivariablecontrol}
Sigurd Skogestad and Ian Postlethwaite.
\newblock \emph{{Multivariable Feedback Control: Analysis and Design}}.
\newblock john Wiley \& sons, 2005.

\bibitem[Smith et~al.(2011)Smith, Miller, Salimi-Khorshidi, Webster, Beckmann, Nichols, Ramsey, and Woolrich]{smith2011FMRI}
Stephen~M Smith, Karla~L Miller, Gholamreza Salimi-Khorshidi, Matthew Webster, Christian~F Beckmann, Thomas~E Nichols, Joseph~D Ramsey, and Mark~W Woolrich.
\newblock {Network modelling methods for FMRI}.
\newblock \emph{Neuroimage}, 54\penalty0 (2):\penalty0 875--891, 2011.

\bibitem[Sugihara et~al.(2012)Sugihara, May, Ye, Hsieh, Deyle, Fogarty, and Munch]{sugihara2012CCM}
George Sugihara, Robert May, Hao Ye, Chih-hao Hsieh, Ethan Deyle, Michael Fogarty, and Stephan Munch.
\newblock {Detecting Causality in Complex Ecosystems}.
\newblock \emph{science}, 338\penalty0 (6106):\penalty0 496--500, 2012.

\bibitem[Sun et~al.(2023)Sun, Schulte, Liu, and Poupart]{sun2023ntsnotears}
Xiangyu Sun, Oliver Schulte, Guiliang Liu, and Pascal Poupart.
\newblock {NTS-NOTEARS: Learning Nonparametric DBNs With Prior Knowledge}.
\newblock In \emph{International Conference on Artificial Intelligence and Statistics}, pages 1942--1964. PMLR, 2023.

\bibitem[Sutherland(2009)]{sutherland2009topologicalspaces}
Wilson~A Sutherland.
\newblock \emph{Introduction to metric and topological spaces}.
\newblock Oxford University Press, 2009.

\bibitem[Tank et~al.(2021)Tank, Covert, Foti, Shojaie, and Fox]{tank2021neuralGranger}
Alex Tank, Ian Covert, Nicholas Foti, Ali Shojaie, and Emily~B Fox.
\newblock {Neural Granger Causality}.
\newblock \emph{IEEE Transactions on Pattern Analysis and Machine Intelligence}, 44\penalty0 (8):\penalty0 4267--4279, 2021.

\bibitem[Tibshirani(1996)]{tibshirani1996LASSOregression}
Robert Tibshirani.
\newblock Regression shrinkage and selection via the lasso.
\newblock \emph{Journal of the Royal Statistical Society Series B: Statistical Methodology}, 58\penalty0 (1):\penalty0 267--288, 1996.

\bibitem[Vowels et~al.(2021)Vowels, Camgoz, and Bowden]{dyalikedags}
Matthew~J Vowels, Necati~Cihan Camgoz, and Richard Bowden.
\newblock {D’ya like DAGs? A survey on structure learning and causal discovery}.
\newblock \emph{ACM Computing Surveys (CSUR)}, 2021.

\bibitem[Xu et~al.(2019)Xu, Huang, and Yoo]{xu2019SCGL}
Chenxiao Xu, Hao Huang, and Shinjae Yoo.
\newblock {Scalable Causal Graph Learning through a Deep Neural Network}.
\newblock In \emph{Proceedings of the 28th ACM international conference on information and knowledge management}, pages 1853--1862, 2019.

\bibitem[Yang et~al.(2022)Yang, Wang, Wang, and Lai]{yang2022heatUS}
Xueli Yang, Zhi-Hua Wang, Chenghao Wang, and Ying-Cheng Lai.
\newblock {Detecting the causal influence of thermal environments among climate regions in the United States}.
\newblock \emph{Journal of Environmental Management}, 322:\penalty0 116001, 2022.

\bibitem[Zheng et~al.(2018)Zheng, Aragam, Ravikumar, and Xing]{zheng2018notears}
Xun Zheng, Bryon Aragam, Pradeep~K Ravikumar, and Eric~P Xing.
\newblock {DAGs with NO TEARS: Continuous Optimization for Structure Learning}.
\newblock \emph{Advances in Neural Information Processing Systems}, 31, 2018.

\end{thebibliography}
\end{document}